\documentclass[twoside,11pt]{article}

\usepackage{blindtext}

\usepackage[preprint]{jmlr2e}

\usepackage{lastpage}

\usepackage[cjk]{kotex}
\usepackage{amsmath}
\usepackage{amssymb}
\usepackage{mathtools}
\usepackage{amsthm}
\usepackage{comment}
\usepackage{bm}
\usepackage{adjustbox}
\usepackage{booktabs, multirow}
\usepackage{xcolor}
\usepackage{url}

\usepackage{apptools}
\AtAppendix{\counterwithin{lemma}{section}}

\newtheorem{example}{Example} 
\newtheorem{theorem}{Theorem}
\newtheorem{lemma}{Lemma} 
 
\newtheorem{remark}{Remark}

\newtheorem{definition}{Definition}

\newtheorem{assumption}{Assumption}

\makeatletter
\renewenvironment{proof}[1][\relax]{\par
  \pushQED{\qed}%
  \normalfont \topsep6\p@\@plus6\p@\relax
  \trivlist
  \item[\hskip\labelsep\itshape
    \ifx#1\relax \proofname\else\proofname{} of #1\fi\@addpunct{.}]\ignorespaces
}{%
  \popQED\endtrivlist\@endpefalse
}
\makeatother

\newcommand{\N}{\mathbb{N}}

\ShortHeadings{Posterior concentrations of fully-connected BNNs}{Kong and Kim}
\firstpageno{1}

\begin{document}

\title{Posterior concentrations of fully-connected Bayesian neural networks 
with general priors on the weights}

\author{\name Insung Kong \email ggong369@snu.ac.kr \\
       \addr Department of Statistics\\
       Seoul National University\\
       Seoul, 08826, South Korea
       \AND
       \name Yongdai Kim \email ydkim0903@gmail.com \\
       \addr Department of Statistics\\
       Seoul National University\\
       Seoul, 08826, South Korea}

\editor{My editor}

\maketitle

\begin{abstract}
Bayesian approaches  for training deep neural networks (BNNs) have received significant interest and have been effectively utilized in a wide range of applications. 
There have been several studies on the properties of posterior concentrations of BNNs.
However, most of these studies only demonstrate results in BNN models with sparse  or heavy-tailed priors.
Surprisingly, no theoretical results currently exist for BNNs using Gaussian priors, which are the most commonly used one.
The lack of theory arises from the absence of approximation results of Deep Neural Networks (DNNs) that are non-sparse and have bounded parameters.
In this paper, we present a new approximation theory for non-sparse DNNs with bounded parameters. 
Additionally, based on the approximation theory,  we show that BNNs with non-sparse general priors can achieve
near-minimax optimal posterior concentration rates to the true model.
\end{abstract}


\section{Introduction}

Bayesian Neural Networks (BNNs) \citep{mackay1992practical, neal2012bayesian}, a framework for training Deep Neural Networks (DNNs) through Bayesian techniques, have garnered significant attention in the field of machine learning and AI.  
The distinctive feature of BNNs lies in their ability to combine the flexibility of DNNs and the probabilistic reasonings of Bayesian approaches. 
This combination results in DNNs that not only yield superior generalization capability across various tasks but also provide improved uncertainty quantification \citep{wilson2020bayesian, 16_izmailov2021bayesian}. 
This attribute is particularly vital in applications where decision-making under uncertainty is crucial.
Representative examples of such applications are 
recommender systems \citep{wang2015collaborative}, 
computer vision \citep{kendall2017uncertainties},
active learning \citep{tran2019bayesian},
medicine \citep{beker2020minimal}
and astrophysics \citep{cranmer2021bayesian}, to name just a few.

The remarkable success of BNNs can be attributed to the inherent capacity of DNNs to automatically learn features from data even if they are parametric models \citep{wang2020survey, 19_jospin2022hands}.
By leveraging this advantage, flexible BNN models can be devised easily which 
can handle complex predictive tasks data-adaptively without explicit model specification. 
This means that even in scenarios where users lack detailed knowledge about the functional relationship between inputs and outputs, 
BNNs are capable of uncovering intricate patterns and relationships existing in data.

There have been vast amounts of literature  that attempt to understanding theoretical properties of BNNs from the nonparametric regression standpoint \citep{polson2018posterior, cherief2020convergence, bai2020efficient, liu2021variable, sun2022consistent, lee2022asymptotic, jantre2023layer, pmlr-v202-kong23e, ohn2024adaptive}.
Rather than presuming the true function to be confined within a specific parametric model, these studies make the broader assumption that the true function belongs to a certain functional space, such as the Hölder function class.
Notably, \citet{polson2018posterior, cherief2020convergence, bai2020efficient, sun2022consistent, lee2022asymptotic, pmlr-v202-kong23e, ohn2024adaptive} show that posterior distribution of a BNN concentrates to the true function with near-minimax optimal rates with respect to the sample size when 
the architecture and the prior on the weights and biases are selected carefully.
These results demonstrate that even when the exact form of the data-generating process is unknown or too complex to be captured by traditional parametric models, BNNs possess the capacity for effective generalization, enabling them to learn these underlying patterns efficiently.

However, there is an important limitation in the existing results. That is,  the priors on the weights and biases 
of DNNs do not include those that are commonly used in practice. For example, \citet{polson2018posterior, cherief2020convergence, bai2020efficient, sun2022consistent, lee2022asymptotic} consider spike-and-slab priors but significant amounts of additional exploration times for searching a sparsity patterns become a significant obstacle in their practical use.
\citet{ohn2024adaptive} derives the posterior concentration rate of non-sparse BNNs but
uses the uniform
distribution on the weights whose domain diverges as the sample size increases.
It would not be easy to select the optimal size of the domain for given finite data which prevents
the Bayesian model of \citet{ohn2024adaptive} from being popularly used in practice.
\citet{pmlr-v202-kong23e} considered polynomial-tail priors, but these priors make 
the calculation of the gradient computationally demanding and thus the development
of a computationally efficient MCMC algorithm be difficult.

Surprisingly, there is no theoretical result about BNNs with
i.i.d. standard Gaussian priors on the weights and biases, which are most popular priors for BNNs in practice
\citep{18_fortuin2022bayesian, 19_jospin2022hands}. That is, there exists a significant gap between theories and applications.
A main reason why there is no result for the optimal concentration rates of BNNs using Gaussian priors is the absence of an approximation theory of fully-connected DNNs with bounded parameters.
Existing approximation theories of DNNs require the weights to be either sparse \citep{suzuki2018adaptivity, schmidt2019deep, imaizumi2019deep, bauer2019deep, ohn2019smooth, schmidt2020nonparametric, nakada2020adaptive, kohler2022estimation, chen2022nonparametric} or unbounded \citep{kohler2021rate, lu2021deep, jiao2023deep}.

The aim of this paper is to fill this gap by
 deriving near-minimax optimal concentration rates of the posterior distributions of BNNs with a class of general priors on the weights including independent Gaussian priors. 
To achieve this aim, we develop a new technique to approximate the Hölder functions
by fully connected DNNs with bounded weights, which offers several advantages:
(1) Compared to \citet{schmidt2020nonparametric}, our results allow
fully connected DNNs, enabling their application to BNNs without the need for sparse-inducing priors to control model complexity. 
(2) In contrast to \citet{kohler2021rate}, our results can employ DNNs with bounded parameters, allowing their application to BNNs without resorting to heavy-tailed priors.
(3) The ReLU activation function \citep{nair2010rectified} is extended to the Leaky-ReLU activation function \citep{maas2013rectifier}, which are known for their optimization merits \citep{xu2015empirical}.

Based on our new approximation results, we demonstrate that the posterior distributions of fully-connected BNNs
with a certain class of priors concentrate to the true function with near-minimax optimal rates.  
Notably, the assumed conditions on the priors hold for most commonly used prior distributions for BNNs, including the independent Gaussian distribution which has not been covered by existing theories.
That is, our results successively fill the important gap between
existing theories and applications in BNNs.

\section{Preliminaries}

\subsection{Notation}
Let $\mathbb{R}$ and $\mathbb{N}$ be the sets of real numbers and natural numbers, respectively. For an integer $n \in \mathbb{N}$, we denote $[n] := \{1,\dots,n\}$. A capital letter denotes a random variable or matrix interchangeably
whenever its meaning is clear.
A vector is denoted by a bold letter, and its elements are denoted by regular letters with superscript indices. e.g. $\bm{x} := (x^{(1)} , \dots, x^{(d)})^{\top}$. 
For a $d$-dimensional vector $\bm{x} \in \mathbb{R}^d$, we denote $|\bm{x}|_p := (\sum_{j=1}^{d}|x^{(j)}|^p)^{1/p}$ for $1 \leq p < \infty$, $|\bm{x}|_{0} := \sum_{j=1}^{d }\mathbb{I}(x^{(j)} \ne 0)$ and $|\bm{x}|_{\infty} := \max_{j \in [d]}|x^{(j)}|$. 
For $\bm{x}_1 = (x_1^{(1)} , \dots, x_1^{(d)})^{\top}$ and $\bm{x}_2 = (x_2^{(1)} , \dots, x_2^{(d)})^{\top}$, $\max(\cdot, \cdot)$ operator is defined as $\max(\bm{x}_1, \bm{x}_2 ) = ( \max(x_1^{(1)}, x_2^{(1)}), \dots, \max(x_1^{(d)}, x_2^{(d)})  )^{\top}$.
For a real-valued function  $f : \mathcal{X} \to \mathbb{R}$  and $1 \leq p < \infty$, we denote $||f||_{p,n} := (\sum_{i=1}^{n} f(\bm{x}_i)^p/n)^{1/p}$ and $||f||_{p,\mathrm{P}_{\bm{X}}} := \left(\int_{\bm{X} \in \mathcal{X}} f(\bm{X})^p d\mathrm{P}_{\bm{X}} \right)^{1/p}$, where
$\mathrm{P}_{\bm{X}}$ is a probability measure defined on the input space $\mathcal{X}$. We assume $\mathcal{X} \subseteq [-a,a]^d$ for some $a \geq 1$. 
Also, we define $||f||_{\infty} := \operatorname{sup}_{\bm{x} \in \mathcal{X}}|f(\bm{x})|$.
For $\bm{\alpha} = (\alpha_1 , \dots , \alpha_r) \in \mathbb{N}^r$, we denote $\partial^{\bm{\alpha}} := \partial^{\alpha_{1}} \ldots \partial^{\alpha_{r}}.$
We denote $\circ$ as the composition of functions.
For two positive sequences $\{a_n\}$ and $\{b_n\}$, we denote $a_n \lesssim b_n$ if there exists a positive sequence $C>0$ such that $a_n \leq Cb_n$ for all $n \in \mathbb{N}$. 
We denote $a_n \asymp b_n$ if $a_n \lesssim b_n$ and $a_n \gtrsim b_n$ hold.
We use the little $o$ notation, that is, we write $a_n = o(b_n)$ if $\lim_{n \to \infty} a_n / b_n = 0$.

We consider the $\beta$-Hölder class $\mathcal{H}_d^\beta (K)$ for the class where the true function belongs, 
which is defined as
\begin{equation*}
\mathcal{H}_d^\beta (K) := \{ f : [-a,a]^d \to \mathbb{R} ; ||f||_{\mathcal{H}^\beta} \leq K \}, 
\end{equation*}
where $||f||_{\mathcal{H}^\beta}$ denotes the Hölder norm defined by
\begin{align*}
||f||_{\mathcal{H}^\beta} :=  \sum_{\bm{\alpha} : |\bm{\alpha}|_1 < \beta} \left\|\partial^{\bm{\alpha}} f\right\|_{\infty} 
+\sum_{\bm{\alpha}:|\bm{\alpha}|_1 =\lfloor\beta\rfloor} \sup _{\underset{\bm{x}_1 \ne \bm{x}_2}{\bm{x}_1, \bm{x}_2 \in [-a,a]^d}} 
\frac{\left|\partial^{\bm{\alpha}} f(\bm{x}_1)-\partial^{\bm{\alpha}} f(\bm{x}_2)\right|}{|\bm{x}_1-\bm{x}_2|_{\infty}^{\beta-\lfloor\beta\rfloor}}.
\end{align*}

\subsection{Deep Neural Networks}
For a depth $L \in \mathbb{N}$ and width $\bm{r} = (r^{(0)}, r^{(1)}, ... , r^{(L)},  r^{(L+1)})^{\top} \in \mathbb{N}^{L+2}$ where $r^{(0)}=d$ and $r^{(L+1)}=1$, Deep Neural Network (DNN) with the $(L, \bm{r})$ architecture is defined as a DNN model which has $L$ hidden layers and 
$r^{(l)}$ many neurons at the $l$-th hidden layer for $l\in [L].$
The output of the DNN model can be written as
\begin{equation}\label{DNN}
f_{\bm{\theta}, \bm{\rho}}^{\operatorname{DNN}}(\cdot) := A_{L+1} \circ \bm{\rho} \circ A_{L} \dots \circ \bm{\rho} \circ A_{1} (\cdot),
\end{equation}
where $A_l : \mathbb{R}^{r^{(l-1)}} \mapsto \mathbb{R}^{r^{(l)}}$ for $l \in [L+1]$ is an affine map defined as $A_l (\bm{x}) := W_l \bm{x} + \bm{b}_l$ with $W_l \in \mathbb{R}^{r^{(l)} \times r^{(l-1)}}$ and $\bm{b}_l \in \mathbb{R}^{r^{(l)}}$
and $\bm{\rho}$ is an activation function.
Here, $\bm{\theta} := (\bm{\theta}_w^{\top} , \bm{\theta}_b^{\top})^{\top}$ is the concatenation of 
the parameters of the DNN model, where
\begin{align*}
\bm{\theta}_w :=& (\operatorname{vec}(W_1)^{\top}, \dots, \operatorname{vec}(W_{L+1})^{\top})^{\top}, \\
\bm{\theta}_b :=& (\bm{b}_1^{\top}, \dots, \bm{b}_{L+1}^{\top})^{\top}
\end{align*}
are the concatenation of the weight matrices and bias vectors.
We denote $T$ as the dimension of $\bm{\theta}$, i.e.,
$$T := T(L, \bm{r}) = \sum_{l=1}^{L+1} (r^{(l-1)}+1)r^{(l)}.$$

The standard choice for the activation function $\bm{\rho}$ is the Rectified linear unit (ReLU) activation function \citep{nair2010rectified}, which is defined as   
$$\bm{\rho}_0(\bm{x}) = \max\{\bm{x},\bm{0}\}.$$
The ReLU activation function is known to alleviate the vanishing gradient problem compared to the sigmoid or tanh activation functions, enabling efficient gradient propagation and enhancing DNN performance \citep{goodfellow2016deep}.

As assumed in other papers for simplicity \citep{kohler2021rate}, we consider DNN architectures whose numbers of neurons in each hidden layer are the same.
For a given activation function $\bm{\rho}$, the number of hidden layers $L \in \mathbb{N}$ and the number of neurons in each hidden layer $r \in \mathbb{N}$, we define the set of DNN functions which are parameterized by $\bm{\theta} \in \mathbb{R}^T$ as below.
\begin{definition}
    For an activation function $\bm{\rho}$, depth $L \in \mathbb{R}$ and width $r \in \mathbb{N}$, we define $\mathcal{F}^{\operatorname{DNN}}_{\bm{\rho}}(L, r)$ as the function class of DNNs with the $(L, (d,r,\dots,r,1)^{\top})$ architecture and the activation function $\bm{\rho}$. 
    That is,
    \begin{align*}
    \mathcal{F}^{\operatorname{DNN}}_{\bm{\rho}}(L,r) := \Bigg\{ f : f = f_{\bm{\theta}, \bm{\rho}}^{\operatorname{DNN}} \text{ is a DNN with the }(L, (d,r,\dots,r,1)^{\top})
    \text{ architecture} \Bigg\}.
    \end{align*}
    Also, for $B \geq 1$, we define $\mathcal{F}^{\operatorname{DNN}}_{\bm{\rho}}(L, r, B)$ as the subset of $\mathcal{F}^{\operatorname{DNN}}_{\bm{\rho}}(L, r)$ 
    consisting of DNNs whose parameter values lie within the absolute bound $B$.
    That is,
    \begin{align*}
    \mathcal{F}^{\operatorname{DNN}}_{\bm{\rho}}(L,r,B) := \Bigg\{ f : f = f_{\bm{\theta}, \bm{\rho}}^{\operatorname{DNN}} \text{ is a DNN with the }(L, (d,r,\dots,r,1)^{\top})&\\
    \text{ architecture, } |\bm{\theta}|_{\infty} \leq B & \Bigg\}.
    \end{align*}
    In addition, for a sparsity $S \in [T]$, we define $\mathcal{F}^{\operatorname{SDNN}}_{\bm{\rho}}(L, r, S, B)$ as the subset of $\mathcal{F}^{\operatorname{DNN}}_{\bm{\rho}}(L, r, B)$ 
    consisting of DNNs whose number of non-zero parameters is bounded by $S$.
    That is,
    \begin{align*}
    \mathcal{F}^{\operatorname{SDNN}}_{\bm{\rho}}(L,r,S,B) := \Bigg\{ f : f = f_{\bm{\theta}, \bm{\rho}}^{\operatorname{DNN}} \text{ is a DNN with the }(L, (d,r,\dots,r,1)^{\top})&\\
    \text{ architecture, } |\bm{\theta}|_{0} \leq S, |\bm{\theta}|_{\infty} \leq B & \Bigg\}.
    \end{align*}
\end{definition}

For an any activation function $\bm{\rho}$, depths $L_1 \leq L_2$, widths $r_1 \leq r_2$,
sparsity $S_1 \leq S_2$ and $B_1 \leq B_2$, we have
$\mathcal{F}^{\operatorname{DNN}}_{\bm{\rho}}(L_1, r_1) \subseteq \mathcal{F}^{\operatorname{DNN}}_{\bm{\rho}}(L_2, r_2)$,
$\mathcal{F}^{\operatorname{DNN}}_{\bm{\rho}}(L_1, r_1, B_1) \subseteq \mathcal{F}^{\operatorname{DNN}}_{\bm{\rho}}(L_2, r_2, B_2)$ and
$\mathcal{F}^{\operatorname{SDNN}}_{\bm{\rho}}(L_1, r_1, S_1, B_1) \subseteq \mathcal{F}^{\operatorname{SDNN}}_{\bm{\rho}}(L_2, r_2, S_2, B_2)$ due to the enlarging property of DNNs.
Also, by definition, we have $\mathcal{F}^{\operatorname{SDNN}}_{\bm{\rho}}(L_1, r_1, S_1, B_1) 
\subseteq \mathcal{F}^{\operatorname{DNN}}_{\bm{\rho}}(L_1, r_1, B_1) 
\subseteq \mathcal{F}^{\operatorname{DNN}}_{\bm{\rho}}(L_1, r_1)$.
In addition, if we denote $T_1$ as the total number of parameters in $\mathcal{F}^{\operatorname{DNN}}_{\bm{\rho}}(L_1, r_1, B_1)$, we have 
$\mathcal{F}^{\operatorname{DNN}}_{\bm{\rho}}(L_1, r_1, B_1) \subseteq
        \mathcal{F}^{\operatorname{SDNN}}_{\bm{\rho}}(L_2, r_2, T_1, B_1)$.

\subsection{Approximation results for DNNs}

Analysis of nonparametric regression using neural networks has been developed over the years. 
\cite{leshno1993multilayer} and \cite{barron1993universal} develop universal approximation properties of shallow neural networks. 
While shallow neural networks can approximate functions well, \cite{montufar2014number} and \cite{eldan2016power} claim that the expressive power of DNNs grows exponentially with the number of layers.
Approximation errors of sparse DNNs with the ReLU activation function have been derived for 
$\beta$-times differentiable functions
\citep{yarotsky2017error} and piecewise smooth functions \citep{petersen2018optimal}. 
\citet{schmidt2020nonparametric} demonstrate that least square estimators based on sparsely connected DNNs with the ReLU activation function and properly chosen architectures achieve near-minimax optimal convergence rates.
Similar results for sparse DNNs can be found in \citet{suzuki2018adaptivity, imaizumi2019deep, bauer2019deep, ohn2019smooth, schmidt2019deep, nakada2020adaptive, kohler2022estimation, chen2022nonparametric}.

These optimal results rely heavily on sparsity constraints on DNNs.
This is because the class of fully-connected DNNs is too large to yield optimal results.
For instance, \citet{schmidt2020nonparametric} uses the approximation theorem that 
there exist
positive constants $C_L^{(s)}$, $C_r^{(s)}$, $C_s^{(s)}$ and $c^{(s)}$
such that for every $f_0 \in \mathcal{H}_d^\beta(K)$ and any sufficiently large $M \in \mathbb{N}$,
there exists 
$$f_{\hat{\bm{\theta}}, \bm{\rho}_0}^{\operatorname{DNN}} \in \mathcal{F}_{\bm{\rho}_0}^{\operatorname{SDNN}}\left(\lceil C_L^{(s)} \log_2 M \rceil, \left\lceil C_r^{(s)} M^{2d} \right\rceil, C_s^{(s)} M^{2d} \log_2 M, 1\right)$$
with
$\|f_{\hat{\bm{\theta}}, \bm{\rho}_{0}}^{\operatorname{DNN}}-f_{0}\|_{\infty, [-a,a]^d} \leq c^{(s)} M^{-2\beta}$.
Indeed, their approximation theorem implies that 
a fully connected DNN in
$\mathcal{F}_{\bm{\rho}_0}^{\operatorname{DNN}}(\lceil C_L^{(s)} \log_2 M \rceil, \lceil C_r^{(s)} M^{2d} \rceil, 1)$ 
approximates $f_0$ well, but the complexity of  $\mathcal{F}_{\bm{\rho}_0}^{\operatorname{DNN}}(\lceil C_L^{(s)} \log_2 M \rceil, \lceil C_r^{(s)} M^{2d} \rceil, 1)$
is too large to use, and so they consider sparsity constraints to reduce complexity.

\citet{kohler2021rate} show that least squares estimators based on fully connected DNNs with the ReLU activation function also achieves near-minimax optimal convergence rates.
They use a new approximation theorem that there exist positive constants $C_L^{(k)}$, $C_r^{(k)}$ and $c^{(k)}$ such that 
for every $f_0 \in \mathcal{H}_d^\beta(K)$ and any sufficiently large $M \in \mathbb{N}$,
there exists 
$$f_{\hat{\bm{\theta}}, \bm{\rho}_0}^{\operatorname{DNN}} \in \mathcal{F}_{\bm{\rho}_0}^{\operatorname{DNN}}\left(\lceil C_L^{(k)} \log_2 M \rceil, \left\lceil C_r^{(k)} M^{d} \right\rceil\right)$$
with
$\|f_{\hat{\bm{\theta}}, \bm{\rho}_{0}}^{\operatorname{DNN}}-f_{0}\|_{\infty, [-a,a]^d} \leq c^{(k)} M^{-2\beta}$.
A significant advantage of the DNN approximation theorem proposed by \citet{kohler2021rate} is that it does not require a sparsity condition.
However, their DNN approximation has the drawback of unbounded parameter sizes.
By following their proof, it can be confirmed that parameters are bounded by $M^2$, and this bound should diverge to reduce the approximation error. 
The DNN approximations of \citet{lu2021deep} and \citet{jiao2023deep} also face the issue of unbounded parameters.
This point limits their applications in some areas, such as Bayesian analysis.

\subsection{Posterior concentration results for BNNs}

Posterior concentration, which is an asymptotic behavior of posterior distributions as the sample size increases, is crucial for the frequentist justification of Bayesian methods.
Concentration rate is defined by the rate at which there exists a neighborhoods of the true model shrinking meanwhile still capturing most of the posterior mass \citep{ghosal2000convergence}.
One of the key components for deriving the posterior concentration rate is the prior concentration condition, which is the requirement of a sufficient amount of prior mass assigned to a shrinking neighborhood of the true model \citep{ghosal2017fundamentals}.

For BNNs, the prior concentration condition is usually established by the two steps: firstly choosing a DNN that approximates the true function well and then secondly devising a prior which puts sufficient
masses around this DNN.
Thus, when the approximation results of DNNs require sparsity assumptions
\citep{suzuki2018adaptivity, imaizumi2019deep, bauer2019deep, ohn2019smooth, schmidt2020nonparametric},
sparse-inducing priors \citep{polson2018posterior, cherief2020convergence, bai2020efficient, sun2022consistent, lee2022asymptotic} are inevitably required to make the posterior concentration rate be optimal.
Specifically, \cite{polson2018posterior} proves that posterior distributions of BNNs using spike-and-slab priors concentrate to the true function at near-minimax optimal rates on the Hölder spaces.
This result has been extended to variational posterior distributions \citep{cherief2020convergence, bai2020efficient}, continuous relaxation of spike-and-slab priors \citep{sun2022consistent} and the case where the true functions belong to the Besov spaces \citep{lee2022asymptotic}.

Following the recent results of approximation using non-sparse DNNs \citep{kohler2021rate}, efforts have been made to show that BNNs with non-sparse priors can also achieve optimal posterior concentration rates \citep{pmlr-v202-kong23e, ohn2024adaptive}.
However, to achieve the optimal posterior concentration
rates using the approximation results of \citet{kohler2021rate}, 
heavy-tailed prior distributions should be used.
For example, \citet{pmlr-v202-kong23e} employs polynomial tail distributions such as the Cauchy distribution as the prior distribution, while \citet{ohn2024adaptive} employs uniform distributions defined on the diverging range. 

So far, no approximation results for non-sparse short-tailed priors to achieve the optimal posterior concentration rates
are available. Due to this theoretical shortcoming, there is no result available 
about the optimal posterior concentration rates of non-sparse BNNs with standard not-heavy-tailed priors such as
independent Gaussian distributions.

\section{Approximation using fully-connected DNN with bounded parameters} \label{sec3}

In this section, we develop a new approximation result using fully-connected DNNs with bounded parameters. 
For activation functions in DNNs, we consider the Leaky-ReLU activation function \citep{maas2013rectifier}, which is 
defined as
$$\bm{\rho}_{\nu}(\bm{x}) := \max\{\bm{x},\nu \bm{x}\}$$
for $\nu \in [0,1)$.
Note that the Leaky-ReLU activation function includes the ReLU activation function as a special case with $\nu=0$.
The Leaky-ReLU activation function addresses one of the main limitations of the ReLU function: the dying ReLU problem \citep{douglas2018relu, lu2020dying}. 
In the original ReLU, negative inputs are zeroed out, potentially leading to dead neurons during training. 
The Leaky-ReLU addresses this issue by allowing a small, non-zero gradient for negative values.
This leads to more consistent trainings of DNNs, and often results in improved performance on various tasks \citep{xu2015empirical}.

We aim to approximate $f_0 : \mathbb{R}^d \to \mathbb{R}$ that belongs to the $\beta$-Hölder class for a pre-specified value $\beta \in (0,\infty)$ by non-sparse DNNs.
In the following theorem, we show that any function in the $\beta$-Hölder class can be approximated by a fully-connected DNN with bounded parameters.


\begin{theorem}  \label{thm_approx}
    For $\beta \in (0,\infty)$, $K \geq 1$ and $\nu \in [0,1)$, 
	there exist positive constants $C_L, C_r, C_B$ and $c_1$
 such that for every $f_0 \in \mathcal{H}_d^\beta(K)$ and every sufficiently large $M \in \mathbb{N}$,
	there exists $f_{\hat{\bm{\theta}}, \bm{\rho}_{\nu}}^{\operatorname{DNN}} \in \mathcal{F}_{\bm{\rho}_{\nu}}^{\operatorname{DNN}}(\lceil C_L \log_2 M \rceil, \left\lceil C_r M^d \right\rceil, C_B)$
	with
	\begin{align*}
	\left\|f_{\hat{\bm{\theta}}, \bm{\rho}_{\nu}}^{\operatorname{DNN}}-f_{0}\right\|_{\infty, [-a,a]^d} \leq c_1 \frac{1}{M^{2\beta}}.
	\end{align*}
\end{theorem}

The proof of Theorem \ref{thm_approx} is provided in Appendix \ref{app_sec_1}.
Theorem \ref{thm_approx} employs DNNs with the number of nodes similar to that in \citet{kohler2021rate} to achieve a similar error rate.
However, while the infinite norm of the parameters in  the approximation of \citet{kohler2021rate} is unbounded,
the parameters in our approximating DNNs are bounded by a constant $C_B>0$, which is independent of $M$
and plays a crucial role to study posterior concentration rates with non-sparse DNNs and not-heavy-tailed priors.
In addition, our theorem demonstrates the result for the Leaky-ReLU activation function using arbitrary $\nu \in [0,1)$, extending the scope beyond the ReLU activation function used in \citet{kohler2021rate}.


\section{Posterior Concentration} \label{sec4}

In this section, we demonstrate the optimal posterior concentration rate of fully-connected BNNs with general priors, including Gaussian priors,
based on the result in Theorem \ref{thm_approx}.
BNNs are defined by a DNN model and priors over its parameters.
We consider a DNN model $f_{\bm{\theta}, \bm{\rho}}^{\operatorname{DNN}}(\cdot)$ with the $(L, \bm{r})$ architecture, where $L \in \mathbb{N}$ and $\bm{r} = (r^{(0)}, r^{(1)}, ... , r^{(L)}, r^{(L+1)})^{\top} \in \mathbb{N}^{L+2}$ denote the depth and width of the DNN respectively.
We assign a prior distribution $\Pi_{\bm{\theta}}$ over the parameters $\bm{\theta} = (\theta^{(1)}, \dots, \theta^{(T)})^{\top} \in \mathbb{R}^{T}$.
Then, the corresponding posterior distribution is given as
\begin{align*}
    \Pi_n (\bm{\theta} \mid \mathcal{D}^{(n)}) =
    \frac{\Pi_{\bm{\theta}} (\bm{\theta}) \mathcal{L}(\bm{\theta} | \mathcal{D}^{(n)})}{\int_{\bm{\theta}} \Pi_{\bm{\theta}} (\bm{\theta}) \mathcal{L}(\bm{\theta} | \mathcal{D}^{(n)})},
\end{align*}
where $\mathcal{D}^{(n)}$ and $\mathcal{L}(\cdot | \mathcal{D}^{(n)})$ are train dataset and corresponding likelihood function, respectively.

For a new test example $\bm{z} \in \mathcal{X}$, the prediction of the BNN is made based on the predictive distribution:
\begin{equation*}
p(y \mid \bm{z}, \mathcal{D}^{(n)})=\int_{\bm{\theta}} p(y \mid \bm{z}, \bm{\theta}) \Pi_n (\bm{\theta} \mid \mathcal{D}^{(n)}) d\bm{\theta}.
\end{equation*}
Since evaluating this integral directly is challenging, the Monte Carlo method is often employed to approximate it:
\begin{equation*}
p(y \mid \bm{z}, \mathcal{D}^{(n)}) \approx \frac{1}{B} \sum_{b=1}^B p(y \mid \bm{z}, \bm{\theta}_b),
\end{equation*}
where $\bm{\theta}_1 \dots, \bm{\theta}_B  \sim \Pi_n (\bm{\theta} \mid \mathcal{D}^{(n)})$ are samples drawn from the posterior.
These samples are usually generated by stochastic gradient Markov chain Monte Carlo \citep{welling2011bayesian, chen2014stochastic, li2016preconditioned, zhang2019cyclical, heek2019bayesian, wilson2020bayesian} or variational inference \citep{graves2011practical, blundell2015weight, louizos2017multiplicative, swiatkowski2020k}.

\subsection{Sufficient condition for priors} \label{sec4_1}

Since the choice of prior directly influences the resulting posterior distribution much, 
it has been the subject of considerable discussions \citep{16_izmailov2021bayesian, fortuin2022priors}.
While the i.i.d. standard Gaussian prior is the most prevalent choice \citep{16_izmailov2021bayesian, 18_fortuin2022bayesian, 19_jospin2022hands}, 
several alternative priors have been proposed, such as Laplace priors \citep{10_williams1995bayesian, 18_fortuin2022bayesian}, radial-directional priors \citep{14_oh2020radial, 15_farquhar2020radial} and hierarchical priors \citep{3_hernandez2015probabilistic, 13_louizos2017bayesian, 8_wu2018deterministic, 12_ghosh2019model, 5_dusenberry2020efficient, 9_ober2021global, 1_seto2021halo}, among others.

To demonstrate general theoretical results that encompass these priors, 
we make only the following very mild assumption about the prior distributions.
\begin{assumption} \label{assumption_theta_prior}    
$\Pi_{\bm{\theta}}$ admits a probability density function $\pi(\cdot)$ on $\mathbb{R}^{T}$ with respect to Lebesgue measure. 
Also, for every $\kappa>0$, there exists $\delta_\kappa>0$ (not depending on $T$) such that 
$\pi(\bm{\theta})$ is lower bounded by $\delta_\kappa^{T}$ on $\bm{\theta} \in [-\kappa,\kappa]^{T}$.
\end{assumption}

Since a lower bound of a density function usually decreases at an exponential rate as $T$ increases,
most prior distributions commonly considered for BNNs satisfy Assumption \ref{assumption_theta_prior}.
Independent priors with specific conditions are simple examples, as follows.

\begin{example}[Independent prior] \label{example_indep}
 Assume that $\theta^{(1)}, \dots, \theta^{(T)}$ are independent with the probability density functions $\pi^{(1)}(\cdot), \dots, \pi^{(T)}(\cdot)$ on $\mathbb{R}$ with respect to Lebesgue measure respectively.
 Also, for every $\kappa>0$, there exists $\delta_\kappa>0$ (not depending on $T$) such that 
for every $t \in [T]$, $\pi^{(t)}(\theta)$ is lower bounded by $\delta_\kappa$ on $\theta \in [-\kappa,\kappa]$. 
\end{example}

Specifically, Example \ref{example_indep}  includes independent Gaussian and Laplace priors, which have not yet considered for the study of optimal posterior concentration rates in other papers.

Although independent priors are predominantly used in most BNNs for algorithmic convenience,
several studies have explored using hierarchical priors for greater flexibility in the prior structure.
Representative examples include zero-mean Gaussian prior with inverse-gamma prior on the prior variance \citep{3_hernandez2015probabilistic, 8_wu2018deterministic}
and group horseshoe prior \citep{13_louizos2017bayesian, 12_ghosh2019model}.
Most of hierarchical priors fulfill Assumption \ref{assumption_theta_prior}, as illustrated in the following example.

\begin{example}[Hierarchical prior] \label{example_hie}
    Assume that the prior distribution of $\bm{\theta}$ is defined by a hierarchical structure:
    \begin{align*}
        \bm{\psi} \sim& \Pi_{\bm{\psi}}, \\
        \bm{\theta}|\bm{\psi} \sim& \Pi_{\bm{\theta}|\bm{\psi}},
    \end{align*}
    where $\bm{\psi} \in \mathbb{R}^S$ 
    is an auxiliary parameter, $\Pi_{\bm{\psi}}$ is a distribution of $\bm{\psi}$ and 
    $\Pi_{\bm{\theta}|\bm{\psi}}$ is a conditional distribution of $\bm{\theta}$ for given $\bm{\psi} \in \mathbb{R}^S$.    
    Further assume that there exist a subset $\Psi \subseteq \mathbb{R}^S$ and a positive constant $\delta_1$ 
    such that
    (1)
    $\Pi_{\bm{\psi}}\left( \bm{\psi} \in \Psi \right) \geq \delta_1^T $
    and
    (2) for every $\bm{\psi} \in \Psi$,
    $\Pi_{\bm{\theta}|\bm{\psi}}$ satisfies Assumption \ref{assumption_theta_prior} with $\delta_\kappa$ not depending on $\bm{\psi}$. 
    Then, Assumption \ref{assumption_theta_prior} holds.
\end{example}

Another example satisfying Assumption \ref{assumption_theta_prior} is the multivariate Gaussian distribution.
Examples of the use of multivariate Gaussian priors include continual learning \citep{nguyen2018variational} and transfer learning \citep{17_spendl2023easy}, 
where they serve as informative priors for transferring information from one domain to another.

\begin{example}[Multivariate Gaussian prior] \label{example_mul}
Assume that there exist positive constants $B$, $\lambda_{\min}$ and $\lambda_{\max}$ such that $\Pi_{\bm{\theta}}$ is a multivariate Gaussian prior with a mean vector in $[-B,B]^T$ 
and a covariance matrix whose eigenvalues are bounded between $\lambda_{\min}$ and $\lambda_{\max}$. 
Then, Assumption \ref{assumption_theta_prior} holds.
\end{example}

The proofs of the three examples are provided in Appendix \ref{app_proof_prior_ex}.
Beyond these examples, most prior distributions commonly considered for BNNs also satisfy Assumption \ref{assumption_theta_prior}.
An example of a prior, however, that does not satisfy Assumption \ref{assumption_sigma_prior} is a uniform distribution defined on $[-1,1]^T$, as its density function has a value $0$ for any $\bm{\theta} \in \mathbb{R}^T \setminus [-1,1]^T$.

\subsection{Result on nonparametric Gaussian regression} 

In nonparametric Gaussian regression problems, we assume that the input vector $\bm{X} \in \mathcal{X} \subseteq [-a,a]^d$ and the response variable $Y \in \mathbb{R}$ are generated from the model
\begin{align}
\begin{split} \label{reg}
\bm{X} \sim & \mathrm{P}_{\bm{X}}, \\
Y|\bm{X} \sim & N(f_0 (\bm{X}), \sigma_0^2), 
\end{split}
\end{align}
where $\mathrm{P}_{\bm{X}}$ is the probability measure defined on $\mathcal{X}.$ Here, $f_0 : \mathcal{X} \to \mathbb{R}$ and $\sigma_0^2 > 0$ are the unknown true regression function and unknown variance of the noise, respectively.
We assume that the true regression function $f_0$ satisfies $||f_0||_{\infty} \leq F$ and $f_0 \in \mathcal{H}_d^\beta (K)$ for some $F \geq 1$, $\beta>0$ and $K \geq 1$.
We assume that $\mathcal{D}^{(n)} := \{ (\boldsymbol{X}_i , Y_i) \}_{i \in [n]}$ are independent copies.

 For Bayesian inference, we consider the probabilistic model
\begin{equation*}
Y_i \stackrel{ind.}\sim N\left( T_F \circ f_{\bm{\theta}, \bm{\rho}_{\nu}}^{\operatorname{DNN}}(\bm{X}_i), \sigma^2 \right) 
\end{equation*}
for a pre-specified $\nu \in [0,1)$,
where $T_F$ is the truncation operator defined as $T_F (x) = \min(\max(x,-F),F)$ and $f_{\bm{\theta}, \bm{\rho}_{\nu}}^{\operatorname{DNN}}$ is the $(L_n, \bm{r}_n)$ architecture DNN, where $L_n$ and $\bm{r}_n$ are given by 
\begin{align}
    \begin{split} \label{network_size}
    L_n :=& \left\lceil C_L \log n \right\rceil,  \\
	r_n :=& \left\lceil C_r n^\frac{d}{2(2\beta + d)} \right\rceil,\\
	\bm{r}_n :=& (d, r_n , \dots, r_n, 1)^{\top} \in \mathbb{N}^{L_n + 2} 
    \end{split}
\end{align}
for constants $C_L$ and $C_r$ defined in Theorem \ref{thm_approx}.
Then, the likelihood is expressed as
\begin{align*}
    \mathcal{L} ( \bm{\theta}, \sigma^2 | \mathcal{D}^{(n)}) = (2 \pi \sigma^2 )^{-\frac{n}{2}}
    \exp \left(-\frac{\sum_{i=1}^n (Y_i -  T_F \circ f_{\bm{\theta}, \bm{\rho}_{\nu}}^{\operatorname{DNN}}(\bm{X}_i))^2}{2 \sigma^2}\right), 
\end{align*}
and the corresponding posterior distribution is given by
\begin{align*}
    \Pi_n (\bm{\theta}, \sigma^2 \mid \mathcal{D}^{(n)}) =
    \frac{
    \Pi_{\bm{\theta}} (\bm{\theta}) \Pi_{\sigma^2} (\sigma^2) \mathcal{L}(\bm{\theta}, \sigma^2 | \mathcal{D}^{(n)})
    }
    {
    \int_{\sigma^2} \int_{\bm{\theta}} \Pi_{\bm{\theta}} (\bm{\theta}) \Pi_{\sigma^2} (\sigma^2) \mathcal{L}(\bm{\theta}, \sigma^2 | \mathcal{D}^{(n)}) d\bm{\theta} d\sigma^2.
    } 
\end{align*}
Here, $\Pi_{\sigma^2}$ is a prior distribution over $\sigma^2 \in \mathbb{R}^{+}$, which is independent of $\Pi_{\bm{\theta}}$ and satisfies the following mild condition.
\begin{assumption} \label{assumption_sigma_prior}
    $\Pi_{\sigma^2}$ admits a density with respect to Lebesgue measure, which is continuous and positive on $(0,2\sigma^2_0).$
    Additionally, $\Pi_{\sigma^2}(\sigma^2>K) \lesssim \frac{1}{K}$ holds for sufficiently large $K$.
\end{assumption}
Assumption \ref{assumption_sigma_prior} holds for most distributions whose support includes $\sigma_0^2$.
The most commonly used prior for the $\sigma^2$ is the inverse-gamma distribution; however, other priors, such as the uniform distribution, can also be employed.

In the following theorem, we demonstrate that
BNNs with general priors (i.e., priors satisfying Assumption \ref{assumption_theta_prior} and \ref{assumption_sigma_prior}) achieve optimal (up to a logarithmic factor) posterior concentration rates to the true regression function.

\begin{theorem} \label{thm_regression}
Assume $f_0 \in \mathcal{H}_d^\beta (K)$ and $\|f_0\|_{\infty} \leq F$ for some $K \geq 1$, $\beta>0$ and $F \geq 1$.
	For $\nu \in [0,1)$, consider the DNN model $f_{\bm{\theta}, \bm{\rho}_{\nu}}^{\operatorname{DNN}}$ with the ($L_n, \bm{r}_n$) architecture, where $L_n$ and $\bm{r}_n$ are given in (\ref{network_size}).
    For any priors $\Pi_{\bm{\theta}}$ and $\Pi_{\sigma^2}$ satisfying Assumption \ref{assumption_theta_prior} and Assumption \ref{assumption_sigma_prior}, respectively,
	the posterior distribution of $T_F \circ f_{\bm{\theta}\bm{\rho}_{\nu}}^{\operatorname{DNN}}$ and $ \sigma^2$ concentrates to $f_0$ and $\sigma_0^2$ at the rate $\varepsilon_{n}=n^{-\beta /(2 \beta+d)} \log ^{\gamma}(n)$ for $\gamma>2$, in the sense that
	\begin{align*}
	\Pi_n \Big( \left(\bm{\theta}, \sigma^2 \right) : \  
	|| T_F \circ f_{\bm{\theta}, \bm{\rho}_{\nu}}^{\operatorname{DNN}} - f_0 ||_{2, \mathrm{P}_{X}} + |\sigma^2 - \sigma_0^2 | > M_n \varepsilon_{n}  \Bigm\vert \mathcal{D}^{(n)}\Big) \overset{\mathbb{P}_{0}^{n}}
 {\to} 0
	\end{align*}
	as $n \to \infty$ for any $M_n \to \infty$, where $\mathbb{P}_{0}^{n}$ is the probability measure of the training data $\mathcal{D}^{(n)}$. 
\end{theorem}

The proof of Theorem \ref{thm_regression} is provided in Appendix \ref{app_B_con_reg}.
The concentration rate $n^{-\beta /(2 \beta+d)}$ is known to be the minimax lower bound when estimating the $\beta$-Hölder smooth functions \cite{tsybakov2009introduction}. Our concentration rate is near-optimal up to a logarithmic factor.

Similar concentration rates have been derived in previous works
\citep{polson2018posterior, cherief2020convergence, bai2020efficient, sun2022consistent, pmlr-v202-kong23e, ohn2024adaptive}.
However, as we mentioned earlier, the prior distributions considered in these studies are not commonly used in practice.
\citet{polson2018posterior, cherief2020convergence, bai2020efficient, sun2022consistent} 
require sparse-inducing priors, which is computationally demanding due to additional exploration time for searching sparsity patterns.
Near-optimal posterior concentration rates for non-sparse BNNs have been obtained by
\citet{pmlr-v202-kong23e} and \citet{ohn2024adaptive}, but extremely heavy-tailed priors are required which do not even include Gaussian distributions.
Theorem \ref{thm_regression} stands as the first result to establish the theoretical optimality of BNNs with Gaussian priors.

\begin{remark}
To prove Theorem \ref{thm_regression}, we 
check the conditions in \citet{ghosal2007convergence},
which is the standard methodology for demonstrating posterior concentrations in nonparametric regression problems.
This technique necessitates showing the prior concentration condition
\begin{align}
\Pi_{\bm{\theta}, \sigma^2}\left( B_{n}^{*}\left((f_0 , \sigma_0^2), \varepsilon_n\right) \right) 
\gtrsim& e^{-n \varepsilon_{n}^2}, \label{KL-ball}
\end{align}
where $B_{n}^{*}\left((f_0 , \sigma_0^2), \varepsilon_n\right)$ denotes the $\varepsilon_n$-Kullback–Leibler neighbourhood around $(f_0 , \sigma_0^2)$.
To establish (\ref{KL-ball}), we must first find a DNN that approximates $f_0$ and then demonstrate that sufficient prior probability exists around it.
While, existing approximation results 
fall short of enabling the demonstration of (\ref{KL-ball}) for the Gaussian prior, in Section \ref{sec3} makes it possible.
See details in the proof of Lemma \ref{reg_cond2} of the Appendix.
\end{remark}

\subsection{Result on nonparametric logistic regression}

In nonparametric logistic regression problems, we assume that the input vector $\bm{X} \in \mathcal{X} \subseteq [-a,a]^d$ and the response 
variable $Y \in \{0,1\}$ are generated from the model
\begin{align}
\begin{split} \label{cla}
\bm{X} \sim & \mathrm{P}_{\bm{X}}, \\
Y|\bm{X} \sim & \operatorname{Bernoulli}\left( \phi \circ f_0 (\bm{X})\right),
\end{split}
\end{align}
where $\mathrm{P}_{\bm{X}}$ is the probability measure defined on $\mathcal{X}$ and $\phi(z) := (1+\exp(-z))^{-1}$ is the sigmoid function. 
Here, $f_0 : \mathcal{X} \to \mathbb{R}$ is the logit of the unknown probability function.
We assume that the true function $f_0$ satisfies $||f_0||_{\infty} \leq F$ and $f_0 \in \mathcal{H}_d^\beta (K)$ for some $F \geq 1$, $\beta>0$ and $K \geq 1$.
We assume that $\mathcal{D}^{(n)} := \{ (\boldsymbol{X}_i , Y_i) \}_{i \in [n]}$ are independent copies.

For Bayesian inference, we consider the probabilistic model
\begin{align*}
Y_i \stackrel{ind.}\sim \operatorname{Bernoulli}\left(\phi \circ T_F \circ f_{\bm{\theta}, \bm{\rho}_{\nu}}^{\operatorname{DNN}}(\bm{X}_i)\right) 
\end{align*} 
for a pre-specified $\nu \in [0,1)$, where $f_{\bm{\theta}, \bm{\rho}_{\nu}}^{\operatorname{DNN}}(\bm{X}_i)$ is the $(L_n, \bm{r}_n)$ architecture DNN, where $L_n$ and $\bm{r}_n$ are given by (\ref{network_size}). 
Then, the likelihood is expressed as
\begin{align*}
    \mathcal{L} ( \bm{\theta} | \mathcal{D}^{(n)}) =      \prod_{i=1}^n (\phi \circ T_F \circ f_{\bm{\theta}, \bm{\rho}_{\nu}}^{\operatorname{DNN}}(\bm{X}_i))^{Y_i}
     (1-\phi \circ T_F \circ f_{\bm{\theta}, \bm{\rho}_{\nu}}^{\operatorname{DNN}}(\bm{X}_i))^{1-Y_i},
\end{align*}
and the corresponding posterior distribution is given bt
\begin{align*}
    \Pi_n (\bm{\theta} \mid \mathcal{D}^{(n)}) =
    \frac{
    \Pi_{\bm{\theta}} (\bm{\theta}) \mathcal{L}(\bm{\theta} | \mathcal{D}^{(n)})
    }
    {
    \int_{\bm{\theta}} \Pi_{\bm{\theta}} (\bm{\theta}) \mathcal{L}(\bm{\theta} | \mathcal{D}^{(n)}) d\bm{\theta}
    }
\end{align*}

In the following theorem, we demonstrate that the BNNs with general priors
(i.e., priors satisfying Assumption \ref{assumption_theta_prior})
achieve optimal (up to a logarithmic factor) posterior concentration rates to the true conditional class probability.
\begin{theorem} \label{thm_classification}
	Assume $f_0 \in \mathcal{H}_d^\beta (K)$ and $\|f_0\|_{\infty} \leq F$
    for some $K \geq 1$, $\beta>0$ and $F>0$.
	For $\nu \in [0,1)$, consider the DNN model $f_{\bm{\theta}, \bm{\rho}_{\nu}}^{\operatorname{DNN}}$ with the ($L_n, \bm{r}_n$) architecture, where $L_n$ and $\bm{r}_n$ are given in (\ref{network_size}).
    For any prior $\Pi_{\bm{\theta}}$ over $\bm{\theta}$ satisfying Assumption \ref{assumption_theta_prior},
	the posterior distribution of $\phi \circ T_F \circ f_{\bm{\theta}\bm{\rho}_{\nu}}^{\operatorname{DNN}}$ concentrates to the true conditional class probability at the rate $\varepsilon_{n}=n^{-\beta /(2 \beta+d)} \log ^{\gamma}(n)$ for $\gamma>2$, in the sense that
	\begin{align*}
	\Pi_n \Big( \bm{\theta} : \  
	|| \phi \circ T_F \circ f_{\bm{\theta}, \bm{\rho}_{\nu}}^{\operatorname{DNN}} - \phi \circ f_0 ||_{2, \mathrm{P}_{X}} > M_n \varepsilon_{n}  \Bigm\vert \mathcal{D}^{(n)}\Big) \overset{\mathbb{P}_{0}^{n}}
 {\to} 0
	\end{align*}
	as $n \to \infty$ for any $M_n \to \infty$, where $\mathbb{P}_{0}^{n}$ is the probability measure of the training data $\mathcal{D}^{(n)}$. 
\end{theorem}

\subsection{Avoiding the curse of dimensionality by assuming
hierarchical composition structure 
} \label{sec_hi}

Due to the inherent hierarchical structure of DNNs, they hold particular advantages when modeling data that also exhibits a hierarchical structure.
For example, image data encompasses multiple levels of abstraction, ranging from low-level features such as edges and textures to high-level concepts like objects and scenes, and is hence considered to exhibit a hierarchical structure.
Building upon this intuition, \citet{schmidt2020nonparametric} and \citet{kohler2021rate} prove that by assuming a hierarchical structure for the true function, DNN models can avoid the curse of dimensionality and achieve faster convergence rates.
However, faster concentration rate of hierarchical composition structure functions in BNNs has not yet been investigated.
In this subsection, we demonstrate that faster BNN concentration results can be achieved by assuming a similar hierarchical structure.

For a minimum smoothness $\beta_{min}>0$, a maximum smoothness $\beta_{max} > 0$ and a maximum dimension $d_{max} \in \mathbb{N}$, 
let 
$$\mathcal{P} \subset [\beta_{min}, \beta_{max}] \times \{ 1, \dots, d_{max} \}$$
be a constraint set consisting of pairs of smoothness and dimension.
We assume that the true function $f_0$ follows a hierarchical composition structure with $\bm{N} \in \mathbb{N}^q$
for some $q \in \mathbb{N}$ and the constraint set $\mathcal{P}$, 
which is defined as follows.

\begin{definition}[hierarchical composition structure] \label{assum_composite}
We say that a function $f$ follows the hierarchical composition structure $\mathcal{H}(\bm{N},\mathcal{P})$ 
if for $(N_0, \dots, N_{q-1})^{\top} := \bm{N}$ and $N_q := 1$,
\begin{itemize}

    \item[a)] For $i \in [N_{0}]$, there exists $d' \in [d]$ such that
    $$f_{0,i}(\bm{x})=x^{(d')} \quad \text { for all } \bm{x} \in \mathbb{R}^d.$$

    \item[b)] For $l \in [q]$ and $i \in [N_{l}]$, there exists $(\beta_{l,i}, d_{l,i}) \in \mathcal{P}$ such that $\sum_{i=1}^{N_{l}} d_{l,i} = N_{l-1}$ holds. 
    Also, there exists a $C_{Lip}$-Lipschitz function $g_{l,i} : \mathbb{R}^{d_{l,i}} \to \mathbb{R}$ with $C_{Lip} \geq 1$ such that $g_{l,i} \in \mathcal{H}_{d_{l,i}}^{\beta_{l,i}} (K)$, $||g_{l,i}||_{\infty} \leq F$ and
    $$f_{l,i}(\bm{x})=g_{l,i}\left(f_{l-1, \sum_{i'=1}^{i-1} d_{l,i'} + 1}(\bm{x}), \ldots, f_{l-1, \sum_{i'=1}^{i-1} d_{l,i'} + d_{l,i}}(\bm{x})\right) \quad \text { for all } \bm{x} \in \mathbb{R}^d.$$

    \item[c)] The function $f$ satisfies 
    $$f(\bm{x})= f_{q,1}(\bm{x}) \quad \text { for all } \bm{x} \in \mathbb{R}^d.$$
\end{itemize}
\end{definition}

\begin{figure}[t]
\centering
\includegraphics[width=.8\linewidth]{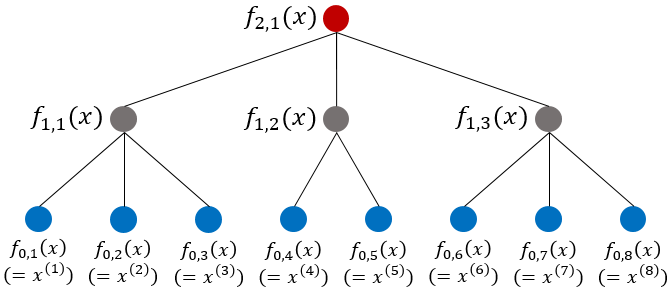}
\caption{\textbf{Example of hierarchical composition structure.} 
} \label{fig_comp_example}
\end{figure}

Figure \ref{fig_comp_example} illustrates a simple example of a hierarchical composition structure $\mathcal{H}(\bm{N}, \mathcal{P})$, where $\bm{N}=(8,3)^{\top}$ and 
$\mathcal{P} = \{(2,4), (3,4), (3,5) \}$.
In this case, we have $d_{1,1}=3, d_{1,2}=2, d_{1,3}=3$, $d_{2,1}=3$, $f_{0,i} = x^{(i)}$ for $i \in [8]$ and
\begin{align*}
    f_{1,1}(\bm{x}) &= g_{1,1} (f_{0,1}(\bm{x}), f_{0,2}(\bm{x}),f_{0,3}(\bm{x}) ), \\
    f_{1,2}(\bm{x}) &= g_{1,2} (f_{0,4}(\bm{x}), f_{0,5}(\bm{x})), \\
    f_{1,3}(\bm{x}) &= g_{1,3} (f_{0,6}(\bm{x}), f_{0,7}(\bm{x}),f_{0,8}(\bm{x}) ), \\
    f_{2,1}(\bm{x}) &= g_{2,1} (f_{1,1}(\bm{x}), f_{1,2}(\bm{x}),f_{1,3}(\bm{x}) ),
\end{align*}
where
$g_{1,1} \in \mathcal{H}_{3}^{4} (K)$, 
$g_{1,2} \in \mathcal{H}_{2}^{4} (K)$,
$g_{1,3} \in \mathcal{H}_{3}^{5} (K)$
and $g_{2,1} \in \mathcal{H}_{3}^{5} (K)$.
In this example, the actual input dimension of $f_{2,1}$ is $8$, but the maximum input dimension of $g_{l,i}$ is $3$.
The primary advantage of assuming such a hierarchical composition structure is that the dimensions of each $g$ are significantly smaller compared to the overall dimensions,
which allows the function to be approximated with a smaller DNN model.

We assume that $\mathcal{D}^{(n)} := \{ (\boldsymbol{X}_i , Y_i) \}_{i \in [n]}$ are independent copies generated from (\ref{reg}) for the nonparametric Gaussian regression problem and from (\ref{cla}) for the nonparametric logistic regression problem,
where the true function $f_0$ follows the hierarchical composition structure $\mathcal{H}(\bm{N},\mathcal{P})$.
For Bayesian inference, we consider the probabilistic model
\begin{equation*}
Y_i \stackrel{ind.}\sim N\left( T_F \circ f_{\bm{\theta}, \bm{\rho}_{\nu}}^{\operatorname{DNN}}(\bm{X}_i), \sigma^2 \right) 
\end{equation*}
for nonparametric Gaussian regression problem and
\begin{align*}
Y_i \stackrel{ind.}\sim \operatorname{Bernoulli}\left(\phi \circ T_F \circ f_{\bm{\theta}, \bm{\rho}_{\nu}}^{\operatorname{DNN}}(\bm{X}_i)\right) 
\end{align*} 
for nonparametric logistic regression problem,
where $\nu \in [0,1)$ is a pre-specified value and $f_{\bm{\theta}, \bm{\rho}_{\nu}}^{\operatorname{DNN}}$ is the $(L_n, \bm{r}_n)$ architecture DNN, where $L_n$ and $\bm{r}_n$ are given by 
\begin{align}
    \begin{split} \label{network_size_comp}
    L_n :=& \lceil \tilde{C}_L \log_2 n \rceil, \\
    r_n :=& \lceil \tilde{C}_r \max_{(\beta', d') \in \mathcal{P}} n^{\frac{d'}{2(2\beta' + d')}} \rceil, \\
	\bm{r}_n :=& (d, r_n , \dots, r_n, 1)^{\top} \in \mathbb{N}^{L_n + 2}, 
    \end{split}
\end{align}
where $\tilde{C}_L$ and $\tilde{C}_r$ are constants (depending on $\beta_{min}$, $\beta_{max}$ and $d_{max}$) defined in Lemma \ref{thm_approx_comp} in Appendix \ref{sec_proof_con_comp}.

The following theorem shows that by assuming hierarchical compositional structure, BNNs with general priors can avoid the curse of dimensionality.
\begin{theorem} \label{thm_con_comp}
    Assume that data are generated from (\ref{reg}) for the nonparametric Gaussian regression problem and from (\ref{cla}) for the nonparametric logistic regression problem. 
	Assume that $f_0$ follows the hierarchical composition structure $\mathcal{H}(\bm{N},\mathcal{P})$, defined in Definition \ref{assum_composite}.
	For $\nu \in [0,1)$, consider the DNN model $f_{\bm{\theta}, \bm{\rho}_{\nu}}^{\operatorname{DNN}}$ with the ($L_n, \bm{r}_n$) architecture, where $L_n$ and $\bm{r}_n$ are given in (\ref{network_size_comp}).
    For any priors $\Pi_{\bm{\theta}}$ (and $\Pi_{\sigma^2}$) which satisfy Assumption \ref{assumption_theta_prior} (and Assumption \ref{assumption_sigma_prior}, respectively),
	the posterior distribution concentrates to the true function at the rate $\varepsilon_{n}=\max_{(\beta',d') \in \mathcal{P}} 
    n^{-\frac{\beta'}{(2\beta'+d')}} \log ^{\gamma}(n)$ for $\gamma>2$, in the sense that for any $M_n \to \infty$,
	\begin{align*}
	\Pi_n \Big( \left(\bm{\theta}, \sigma^2 \right) : \  
	|| T_F \circ f_{\bm{\theta}, \bm{\rho}_{\nu}}^{\operatorname{DNN}} - f_0 ||_{2, \mathrm{P}_{X}} + |\sigma^2 - \sigma_0^2 | > M_n \varepsilon_{n}  \Bigm\vert \mathcal{D}^{(n)}\Big) \overset{\mathbb{P}_{0}^{n}}
 {\to} 0
	\end{align*}
    holds for the nonparametric Gaussian regression problem and
    \begin{align*}
	\Pi_n \Big( \bm{\theta} : \  
	|| \phi \circ T_F \circ f_{\bm{\theta}, \bm{\rho}_{\nu}}^{\operatorname{DNN}} - \phi \circ f_0 ||_{2, \mathrm{P}_{X}} > M_n \varepsilon_{n}  \Bigm\vert \mathcal{D}^{(n)}\Big) \overset{\mathbb{P}_{0}^{n}}
 {\to} 0
	\end{align*}
	holds for the nonparametric logistic regression problem, where $\mathbb{P}_{0}^{n}$ is the probability measure of the training data $\mathcal{D}^{(n)}$. 
\end{theorem}

The proof of Theorem \ref{thm_con_comp} is provided in Appendix \ref{sec_proof_con_comp}.
Note that the concentration rate $\max_{(\beta',d') \in \mathcal{P}} 
n^{-\frac{\beta'}{(2\beta'+d')}}$ is much faster than $n^{-\frac{\beta}{(2\beta+d)}}$ in the case $d_{max} << d$ and hence avoids the curse of dimensionality.
This rate is known to be the minimax lower bound when estimating hierarchical composition structure functions \citep{schmidt2020nonparametric}. 
Concentration rate in Theorem \ref{thm_con_comp} is near-optimal up to a logarithmic factor.

\section{Bayesian Neural Networks adaptive to Smoothness}
\label{sec_adapt}
Similar to minimax optimal results of least square estimators with DNNs \citep{schmidt2020nonparametric, kohler2021rate}, 
theories in the previous section also face the limitation of requiring knowledge of the smoothness of the true function (or constraint set $\mathcal{P}$ for hierarchical composition structure in Section \ref{sec4}) to choose a network of appropriate size.
Since the true smoothness is rarely known,
the optimal width is usually determined through the use of a validation dataset in practice.

In Bayesian analysis, this issue is often addressed by assigning a prior to the parameter
related to smoothness or model complexity.
Instead of choosing the width $r$ depending on $\beta$ as well as $n$ in our previous results in Section \ref{sec4},
we can assign a prior to $r.$ 
Specifically, we give prior 
\begin{align}
    \Pi_r (r) &\propto e^{- (\log n)^5 r^2}, \label{width_prior}
\end{align}
which is similar to the prior considered in \citet{pmlr-v202-kong23e} and \citet{ohn2024adaptive}.
For given $r \in \mathbb{N}$, $L_n$ and $\bm{r}_n$ are given by 
\begin{align}
    \begin{split} \label{network_size_adapt}
    L_n &:= \left\lceil \tilde{C}_L \log n \right\rceil,  \\
	\bm{r}_n &:= (d, r , \dots, r, 1)^{\top} \in \mathbb{N}^{L_n + 2},
    \end{split}
\end{align}	
for the constant $\tilde{C}_L$ used in (\ref{network_size_comp}).
For a given DNN structure, we assign a prior $\Pi_{\bm{\theta}}$ over $\bm{\theta} \in \mathbb{R}^{T_n}$ which satisfies Assumption \ref{assumption_theta_prior}, where $T_n$ is defined by
$$T_n := (d+1)r + (L_n -1)(r+1)r + (r+1).$$
In addition, for nonparametric Gaussian regression, we assign a prior $\Pi_{\sigma^2}$ over $\sigma^2$ which satisfies Assumption \ref{assumption_sigma_prior}.

For Bayesian inference, we consider the probabilistic model
\begin{equation*}
Y_i \stackrel{ind.}\sim N\left( T_F \circ f_{\bm{\theta}, \bm{\rho}_{\nu}}^{\operatorname{DNN}}(\bm{X}_i), \sigma^2 \right) 
\end{equation*}
for nonparametric Gaussian regression problem and
\begin{align*}
Y_i \stackrel{ind.}\sim \operatorname{Bernoulli}\left(\phi \circ T_F \circ f_{\bm{\theta}, \bm{\rho}_{\nu}}^{\operatorname{DNN}}(\bm{X}_i)\right) 
\end{align*} 
for nonparametric logistic regression problem,
where $\nu \in [0,1)$ is a pre-specified value and $f_{\bm{\theta}, \bm{\rho}_{\nu}}^{\operatorname{DNN}}$ is the $(L_n, \bm{r}_n)$ architecture DNN.

The following theorem shows that by giving suitable prior on the width, 
BNNs with general priors achieve optimal (up to a logarithmic factor) posterior concentration rates to the true function, adaptively to the true smoothness.
\begin{theorem} \label{thm_adaptive}
    Assume that data are generated from (\ref{reg}) for the nonparametric Gaussian regression problem and from (\ref{cla}) for the nonparametric logistic regression problem.
    Assume that $f_0$ follows the hierarchical composition structure $\mathcal{H}(\bm{N},\mathcal{P})$, defined in Definition \ref{assum_composite}.
    For $\nu \in [0,1)$,  consider the prior (\ref{width_prior}) over the width $r$, and consider the DNN model $f_{\bm{\theta}, \bm{\rho}_{\nu}}^{\operatorname{DNN}}$ with the ($L_n, \bm{r}_n$) architecture, where $L_n$ and $\bm{r}_n$ are given in (\ref{network_size_adapt}).
    For any priors $\Pi_{\bm{\theta}}$ (and $\Pi_{\sigma^2}$) which satisfy Assumption \ref{assumption_theta_prior} (and Assumption \ref{assumption_sigma_prior}, respectively),
	the posterior distribution concentrates to the true function at the rate $\varepsilon_{n}=\max_{(\beta',d') \in \mathcal{P}} 
    n^{-\frac{\beta'}{(2\beta'+d')}} \log ^{\gamma}(n)$ for $\gamma>\frac{5}{2}$, in the sense that for any $M_n \to \infty$,
	\begin{align*}
	\Pi_n \Big( \left(\bm{\theta}, \sigma^2 \right) : \  
	|| T_F \circ f_{\bm{\theta}, \bm{\rho}_{\nu}}^{\operatorname{DNN}} - f_0 ||_{2, \mathrm{P}_{X}} + |\sigma^2 - \sigma_0^2 | > M_n \varepsilon_{n}  \Bigm\vert \mathcal{D}^{(n)}\Big) \overset{\mathbb{P}_{0}^{n}}
 {\to} 0
	\end{align*}
    holds for the nonparametric Gaussian regression problem and
    \begin{align*}
	\Pi_n \Big( \bm{\theta} : \  
	|| \phi \circ T_F \circ f_{\bm{\theta}, \bm{\rho}_{\nu}}^{\operatorname{DNN}} - \phi \circ f_0 ||_{2, \mathrm{P}_{X}} > M_n \varepsilon_{n}  \Bigm\vert \mathcal{D}^{(n)}\Big) \overset{\mathbb{P}_{0}^{n}}
 {\to} 0
	\end{align*}
	holds for the nonparametric logistic regression problem, where $\mathbb{P}_{0}^{n}$ is the probability measure of the training data $\mathcal{D}^{(n)}$. 
\end{theorem}

The proof of Theorem \ref{thm_adaptive} is provided in Appendix \ref{sec_proof_adapt}.
Theorem \ref{thm_adaptive} implies that BNNs with a random width achieve near-optimal concentration rates up to a logarithmic factor as long as the prior of the random width is carefully selected.
Note that the proposed prior for the width
does not utilize $\mathcal{P}$, which means that
the BNNs achieve the optimal posterior concentration rates adaptively to the smoothness of the true model.

Posterior inference in such models poses significant challenges, since the dimension of parameters changes as the network structure changes.
In such cases, a commonly used method is reversible jump MCMC \citep{green1995reversible}, which employs dimension-matching techniques and proposes changes of the dimensionality along with appropriate adjustments in the parameters.
Alternatively, one could consider applying masking variables to a sufficiently large DNN, and by using a well-designed proposal distribution for the Metropolis-Hastings algorithm to update the masking variables, faster posterior mixing can be induced \citep{pmlr-v202-kong23e}.

\section{Discussions}

The main contributions of this paper are (1) to provide the new 
approximation result of DNNs in Theorem 1 and (2) to derive
the posterior concentration rates of BNNs with general priors based on Theorem 1.
We believe that there are other problems where the new approximation result of DNNs 
plays a crucial role. A possible example would be asymptotic properties of a certain penalized
least square (or maximum likelihood) estimator of DNN. There are some studies of sparse penalties with DNNs \citep{ohn2022nonconvex}
but no results are available for non-sparse penalties. 

We have only considered the posterior concentration rates of BNNs. A more interesting property would be uncertainty quantification. For Bayesian nonparametric regression, \citet{szabo2015frequentist} and \citet{rousseau2020asymptotic} have studied asymptotic
properties in view of uncertainty quantification. It would be expected that similar results hold
for BNNs but not yet proved. Our new approximation result could be a good starting point.



\newpage

\appendix
\numberwithin{equation}{section}
\section{Proofs for Section \ref{sec3}}
\label{app_sec_1}
\renewcommand{\theequation}{A.\arabic{equation}}

In this section, we prove Theorem \ref{thm_approx}.
In Section \ref{app_A_not}, we describe additional notations for the proofs.
In Section \ref{app_A_lemma}, we state and prove a re-scaling lemma for Leaky-ReLU DNNs.
In Section \ref{app_A_aux}, we construct auxiliary networks with Leaky-ReLU activation function.
Based on these results, we prove Theorem \ref{thm_approx} in Section \ref{app_A_approx}.

Our proof for Theorem \ref{thm_approx} closely follows the proof of \citet{kohler2021rate} but
we make the following four modifications: (1) alteration of the activation function, (2) adjustments to the network size accordingly, (3) imposition of the upper bound of the absolut values of parameters 
in each layer and (4) alteration of the upper bounds 
to make them similar.
For simplicity, we refer to the results in the proof of \citet{kohler2021rate}
unless there is any confusion, and  focus on the four modifications.

\subsection{Additional notations}
\label{app_A_not}

For a DNN model $f^{\operatorname{DNN}}$, $L(f^{\operatorname{DNN}})$ denotes the number of hidden layers in $f^{\operatorname{DNN}}$.
For $l \in [L(f^{\operatorname{DNN}})+1]$, $W_l(f^{\operatorname{DNN}})$ and $\bm{b}_l(f^{\operatorname{DNN}})$ denote the weight matrix and bias vector of the $l$-th layer of $f^{\operatorname{DNN}}$, respectively. 
We futher denote
\begin{align*}
\bm{\theta}_w (f^{\operatorname{DNN}}) :=& (\operatorname{vec}(W_1(f^{\operatorname{DNN}}))^{\top}, \dots, \operatorname{vec}(W_{L+1} (f^{\operatorname{DNN}}))^{\top})^{\top}, \\
\bm{\theta}_b (f^{\operatorname{DNN}}) :=& (\bm{b}_1 (f^{\operatorname{DNN}})^{\top}, \dots, \bm{b}_{L+1} (f^{\operatorname{DNN}})^{\top})^{\top}.
\end{align*}
We define the set of Leaky-ReLU DNN functions where the absolute values of the weights and biases are bounded by $B_w$ and $B_b$, respectively, by $\tilde{\mathcal{F}}_{\bm{\rho}_{\nu}}^{\operatorname{DNN}}(L,r,B_w, B_b)$ as follows.

\begin{definition}
    For $L \in \mathbb{N}$, $r \in \mathbb{N}$, $B_w \geq 1$ and $B_b \geq 1$, we define $\tilde{\mathcal{F}}^{\operatorname{DNN}}_{\bm{\rho}}(L,r,B_w, B_b)$ as the function class of DNNs with the $(L, (d,r,\dots,r,1)^{\top})$ architecture and the activation function $\bm{\rho}$
    such that the absolute values of the weights are bounded by $B_w$ and the absolute values of the biases are bounded by $B_b$. That is,
    \begin{align*}
    \tilde{\mathcal{F}}_{\bm{\rho}_{\nu}}^{\operatorname{DNN}}(L,r,B_w, B_b) := \Bigg\{ f : f = f_{\bm{\theta}, \bm{\rho}}^{\operatorname{DNN}} \text{ is a DNN with the }\left(L, (d,r,\dots,r,1)^{\top}\right) \text{architecture},& \\
     |\bm{\theta}_w(f)|_{\infty} \leq B_w, |\bm{\theta}_b(f)|_{\infty} \leq B_b & \Bigg\}.
    \end{align*}
\end{definition}
We denote $\mathbb{N}_0 := \{0\} \cup \mathbb{N}$ and $\mathbb{R}^{+} := \{ x \in \mathbb{R} : x>0 \}$.
A vector function is denoted by a bold letter. e.g. $\bm{f}(\cdot) := (f^{(1)}(\cdot),\dots,f^{(k)}(\cdot))^{\top}$.


\subsection{Re-scaling Lemma for Leaky-ReLU DNNs}
\label{app_A_lemma}

\begin{lemma} \label{lemma_equal_network}
For $\nu \in [0,1)$, $L \in \mathbb{N}$ and $\bm{r} \in \mathbb{N}^{L+2}$,
consider a DNN with the $(L, \bm{r})$ architecture $f_{\bm{\theta}, \bm{\rho}_{\nu}}^{\operatorname{DNN}} : \mathbb{R}^d \to \mathbb{R}$
which is parameterized by $\bm{\theta} = (\bm{\theta}_w^{\top} , \bm{\theta}_b^{\top})^{\top}$, where 
$\bm{\theta}_w = (\operatorname{vec}(W_1)^{\top}, \dots, \operatorname{vec}(W_{L+1})^{\top})^{\top}$ and 
$\bm{\theta}_b = (\bm{b}_1^{\top}, \dots, \bm{b}_{L+1}^{\top})^{\top}$. 
For positive constants $\zeta_1,\dots, \zeta_{L+1}$, we define
\begin{align*}
    \tilde{W}_l :=& \zeta_l \cdot W_l, \\
    \tilde{\bm{b}}_l :=& \left(\prod_{l'=1}^{l} \zeta_{l'} \right) \cdot \bm{b}_l
\end{align*}
for $l \in [L+1]$ and define
\begin{align*}
    \tilde{\bm{\theta}}_w :=& (\operatorname{vec}(\tilde{W}_1)^{\top}, \dots, \operatorname{vec}(\tilde{W}_{L+1})^{\top})^{\top}, \\
    \tilde{\bm{\theta}}_b :=& (\tilde{\bm{b}}_1^{\top}, \dots, \tilde{\bm{b}}_{L+1}^{\top})^{\top},\\
    \tilde{\bm{\theta}} :=& (\tilde{\bm{\theta}}_w^{\top} , \tilde{\bm{\theta}}_b^{\top})^{\top}.
\end{align*}
If $\prod_{l=1}^{L+1} \zeta_l = 1$,
$$f_{\bm{\theta}, \bm{\rho}_{\nu}}^{\operatorname{DNN}}(\bm{x}) = f_{\tilde{\bm{\theta}}, \bm{\rho}_{\nu}}^{\operatorname{DNN}} (\bm{x})$$
holds for every $\bm{x} \in \mathbb{R}^d$
\end{lemma}

\begin{proof}
    Note that the Leaky-ReLU activation function has the property
    $$\bm{\rho}_{\nu}(\zeta \bm{x}) = \max\{\zeta \bm{x},\zeta \nu \bm{x}\} =
    \zeta \max\{\bm{x},\nu \bm{x}\} = \zeta \bm{\rho}_{\nu}(\bm{x})$$
    for every $\zeta>0$ and vector $\bm{x}$.
    For $l \in [L+1]$, we denote $\tilde{A}_l (\bm{x}) := \tilde{W}_l \bm{x} + \tilde{\bm{b}}_l$ as the affine map defined by $\tilde{W}_l$ and $\tilde{\bm{b}}_l$.   
    
    First, we have 
    $\tilde{A}_1 (\bm{x}) = \tilde{W}_1 \bm{x} + \tilde{\bm{b}}_1 = \zeta_1 A_1(\bm{x}).$
    Assume that for some $l \in [L]$, 
    $$\tilde{A}_{l} \circ \bm{\rho}_{\nu} \dots \circ \bm{\rho}_{\nu} \circ \tilde{A}_{1} (\bm{x}) =  \left(\prod_{l'=1}^{l} \zeta_{l'} \right) \cdot \left(A_{l} \circ \bm{\rho}_{\nu} \dots \circ \bm{\rho}_{\nu} \circ A_{1} (\bm{x})\right)$$
    holds.
    Then, we have
    \begin{align*}
        \tilde{A}_{l+1} \circ \bm{\rho}_{\nu} \circ \tilde{A}_{l} \circ \bm{\rho}_{\nu} \dots \circ \bm{\rho}_{\nu} \circ \tilde{A}_{1} (\bm{x}) 
        = & \tilde{A}_{l+1} \circ \bm{\rho}_{\nu} \circ \left( \left(\prod_{l'=1}^{l} \zeta_{l'} \right) \cdot \left(A_{l} \circ \bm{\rho}_{\nu} \dots \circ \bm{\rho}_{\nu} \circ A_{1} (\bm{x})\right) \right) \\
        = & \tilde{A}_{l+1} \circ \left( \left(\prod_{l'=1}^{l} \zeta_{l'} \right) \cdot \left(\bm{\rho}_{\nu} \circ A_{l} \circ \bm{\rho}_{\nu} \dots \circ \bm{\rho}_{\nu} \circ A_{1} (\bm{x})\right) \right) \\
        = & \tilde{W}_{l+1} \left( \left(\prod_{l'=1}^{l} \zeta_{l'} \right)  \cdot \left(\bm{\rho}_{\nu} \circ A_{l} \circ \bm{\rho}_{\nu} \dots \circ \bm{\rho}_{\nu} \circ A_{1} (\bm{x})\right) \right) + \tilde{\bm{b}}_{l+1} \\
        = & \left(\prod_{l'=1}^{l+1} \zeta_{l'} \right) \cdot \left(W_{l+1} \left( \bm{\rho}_{\nu} \circ A_{l} \circ \bm{\rho}_{\nu} \dots \circ \bm{\rho}_{\nu} \circ A_{1} (\bm{x}) \right) + \bm{b}_{l+1}\right) \\
        = & \left(\prod_{l'=1}^{l+1} \zeta_{l'} \right) \cdot \left(  A_{l+1} \circ \bm{\rho}_{\nu} \circ A_{l} \circ \bm{\rho}_{\nu} \dots \circ \bm{\rho}_{\nu} \circ A_{1} (\bm{x}) \right).
    \end{align*}
    Hence, by mathematical induction, we have the assertion.
\end{proof}

Lemma \ref{lemma_equal_network} implies that scale of some layers can be transferred to other layers in Leaky-ReLU DNNs.
This is due to the fact that the Leaky-ReLU activation function satisfies the property $\bm{\rho}_{\nu}(c\bm{x}) = c \bm{\rho}_{\nu}(\bm{x})$ for any $\bm{x}$ and $c>0$.
This lemma is particularly useful when the absolute values of some parameters in the lower layers are large, while those in other layers are not. 
Figure \ref{fig_equal_example} provides a simple illustration of Lemma \ref{lemma_equal_network} with $\zeta_1 = 2^{-L}$ and $\zeta_2 = \dots = \zeta_{L+1} = 2$.

\begin{figure}[t]
\centering
\includegraphics[width=0.8\linewidth]{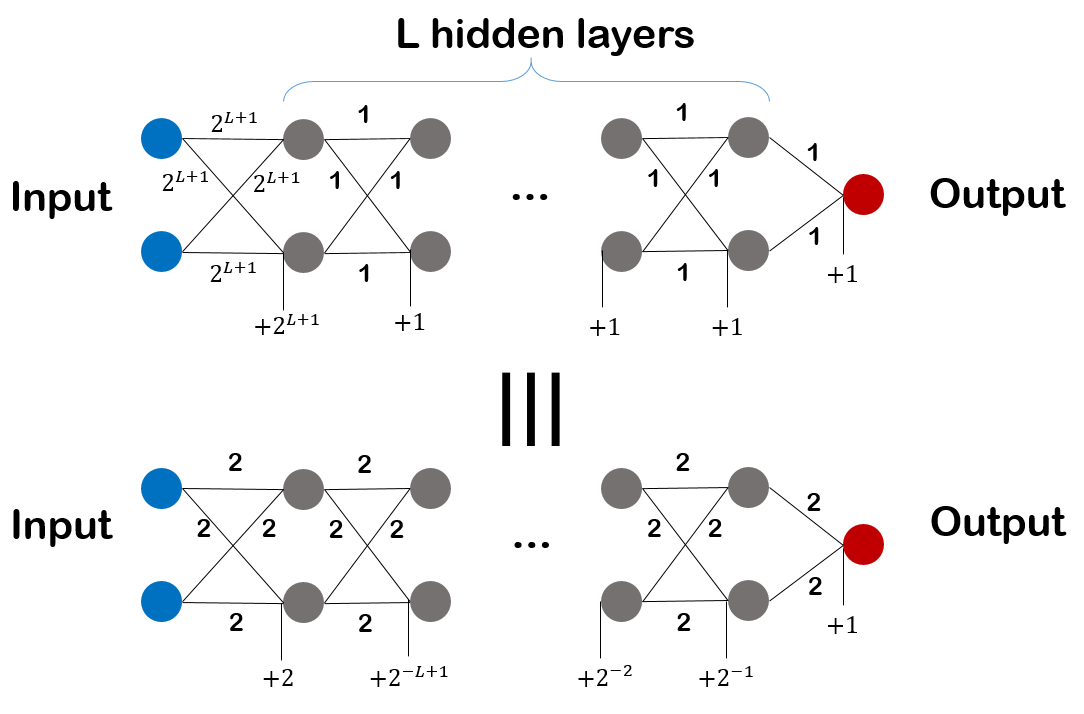}
\caption{\textbf{Example of Lemma \ref{lemma_equal_network}} \label{fig_equal_example}
A DNN with depth $L$ and width $\bm{r} = (2,\dots,2,1)^{\top}$ (above) and its re-scaled DNN (below) using Lemma \ref{lemma_equal_network}. 
We use $\zeta_1 = 2^{-L}$ and $\zeta_2 = \dots = \zeta_{L+1} = 2$ for re-scaling.
The two networks produce a same output for a same input.
} \label{proposal_compare2}
\end{figure}

\subsection{Auxiliary networks with the Leaky-ReLU activation function}
\label{app_A_aux}

\begin{lemma} \label{aux_network_basic}
For $\nu \in [0,1)$ and $k \in \mathbb{N}$,
\begin{enumerate}
    \item[a)] There exists a neural network $\bm{f}_{id}(\cdot) : \mathbb{R}^k \to \mathbb{R}^k$ such that for every $\bm{x} \in \mathbb{R}^k$
$$\bm{f}_{id}(\bm{x}) = \bm{x},$$
and for every neural network $\bm{g}(\cdot) : \mathbb{R}^d \to \mathbb{R}^k$ in $\tilde{\mathcal{F}}_{\bm{\rho}_{\nu}}^{\operatorname{DNN}}(L, r, B_w, B_b)$,
$$\bm{f}_{id}(\bm{g}(\cdot)) \in \tilde{\mathcal{F}}_{\bm{\rho}_{\nu}}^{\operatorname{DNN}}\left(L+1, \max(r, 2k), B_w, B_b\right).$$
    \item[b)] There exists a neural network $\bm{f}_{\rho_0}(\cdot) : \mathbb{R}^k \to \mathbb{R}^k$ 
such that for every $\bm{x} \in \mathbb{R}^k$
$$\bm{f}_{\rho_0}({\bm{x}}) = \bm{\rho}_0 (\bm{x}),$$
and for every neural network $\bm{g}(\cdot) : \mathbb{R}^d \to \mathbb{R}^k$ in $\tilde{\mathcal{F}}_{\bm{\rho}_{\nu}}^{\operatorname{DNN}}(L, r, B_w, B_b)$,
$$\bm{f}_{\rho_0}(\bm{g}(\cdot)) \in \tilde{\mathcal{F}}_{\bm{\rho}_{\nu}}^{\operatorname{DNN}}\left(L+1, \max(r, 2k), \max\left(B_w, \frac{1}{1-\nu^2}\right), B_b\right).$$
\end{enumerate}
\end{lemma}
\begin{proof}
\begin{enumerate}
    \item[a)] For $\bm{x} = (x^{(1)}, \dots, x^{(k)})^{\top}$, 
    $$\bm{f}_{id}(\bm{x}) := \frac{1}{1+\nu} \bm{\rho}_{\nu}(\bm{x}) - \frac{1}{1+\nu} \bm{\rho}_{\nu}(-\bm{x})$$
     satisfies $f_{id}(\bm{x})^{(i)} = \frac{x^{(i)}}{1+\nu} + \frac{\nu x^{(i)}}{1+\nu} = x^{(i)}$ for $x^{(i)} \geq 0$ and $f_{id}(\bm{x})^{(i)} = \frac{\nu x^{(i)}}{1+\nu} + \frac{x^{(i)}}{1+\nu} = x^{(i)}$ for $x^{(i)} < 0$.         
    Also, $\bm{f}_{id}(\bm{g}(\cdot))$ requires $r$ neurons in each of the $1$st through $L$th hidden layers and $2k$ neurons in the $(L+1)$th hidden layer. 
    Hence, we get the assertion.

    \item[b)] For $\bm{x} = (x^{(1)}, \dots, x^{(k)})^{\top}$, 
    $$\bm{f}_{\rho_0}(\bm{x}) := \frac{1}{1-\nu^2} \bm{\rho}_{\nu}(\bm{x}) + \frac{\nu}{1-\nu^2} \bm{\rho}_{\nu}(-\bm{x})$$
    satisfies $f_{\rho_0}(\bm{x})^{(i)} = \frac{x^{(i)}}{1-\nu^2} - \frac{\nu^2 x^{(i)}}{1-\nu^2} = x^{(i)}$ for $x^{(i)} \geq 0$ and $f_{\rho_0}(\bm{x})^{(i)} = \frac{\nu x^{(i)}}{1-\nu^2} - \frac{\nu x^{(i)}}{1-\nu^2} = 0$ for $x^{(i)} < 0$.
    Also, $\bm{f}_{\rho_0}(\bm{g}(\cdot))$ requires $r$ neurons in each of the $1$st through $L$th hidden layers and $2k$ neurons in the $(L+1)$th hidden layer. 
    Hence, we get the assertion.
\end{enumerate}
\end{proof}
For $L \in \N$, we denote $\bm{f}_{id}^{L}(\cdot) := \bm{f}_{id} \circ \bm{f}_{id} \circ \cdots \circ \bm{f}_{id}(\cdot) \in \tilde{\mathcal{F}}_{\bm{\rho}_{\nu}}^{\operatorname{DNN}}(L, 2k, 1, 1)$, where $k$ is the input dimension of $\bm{f}_{id}^{L}$.

\begin{lemma} \label{aux_network_indtest}
Let $\nu \in [0,1)$, $R \in \N$, $\bm{b}_1, \bm{b}_2 \in [-a, a]^d$ with
$b_2^{(i)} - b_1^{(i)} \geq \frac{2}{R}$ for all $i \in [d]$
and
\begin{align*}
K_{1/R} = \left\{\bm{x} \in \mathbb{R}^d : \forall i \in [d],  x^{(i)} \notin [b_1^{(i)}, b_1^{(i)}+1/R) \cup (b_2^{(i)} - 1/R, b_2^{(i)}) \right\}.
\end{align*}
\begin{enumerate}
    \item[a)] There exists a neural network 
$f_{ind, [\bm{b}_1, \bm{b}_2)}(\cdot) : \mathbb{R}^d \to \mathbb{R}$ such that
for $\bm{x} \in K_{1/R}$
\begin{align} 
f_{ind, [\bm{b}_1, \bm{b}_2)}(\bm{x}) = \mathbb{I}_{ [\bm{b}_1, \bm{b}_2)}(\bm{x}) ,\label{ind_property_1}
\end{align}
and for $\bm{x} \in \mathbb{R}^d$
\begin{align}
\left|f_{ind, [\bm{b}_1, \bm{b}_2)}(\bm{x}) - \mathbb{I}_{[\bm{b}_1, \bm{b}_2)}(\bm{x})\right| \leq 1, \label{ind_property_2}
\end{align}
and for every neural network $\bm{g}(\cdot) : \mathbb{R}^d \to \mathbb{R}^d$ in $\tilde{\mathcal{F}}_{\bm{\rho}_{\nu}}^{\operatorname{DNN}}(L, r, B_w, B_b)$
$$f_{ind, [\bm{b}_1, \bm{b}_2)}(\bm{g}(\cdot)) \in \tilde{\mathcal{F}}_{\bm{\rho}_{\nu}}^{\operatorname{DNN}} \left(L+2, \max(r, 4d), \max\left(B_w, \frac{R}{1-\nu^2}\right) , \max(B_b, 1+a) \right).$$ 
    \item[b)] Let $|s| \leq R$. Then there exists the network $f_{test}(\cdot) : \mathbb{R}^{3d+1} \to \mathbb{R}$ such that for $\bm{x} \in K_{1/R}$
\begin{align}
  f_{test}(\bm{x}, \bm{b}_1, \bm{b}_2, s)
  =
  s \cdot \mathbb{I}_{[\bm{b}_1, \bm{b}_2)}(\bm{x}), \label{test_property_1}
\end{align}
and for $\bm{x} \in \mathbb{R}^d$
\begin{align}
  \left|f_{test}(\bm{x}, \bm{b}_1, \bm{b}_2, s) -
  s \cdot \mathbb{I}_{[\bm{b}_1, \bm{b}_2)}(\bm{x})\right| \leq |s|, \label{test_property_2}
\end{align}
and for $\bm{g}_1(\cdot) : \mathbb{R}^d \to \mathbb{R}^d$ in $\tilde{\mathcal{F}}_{\bm{\rho}_{\nu}}^{\operatorname{DNN}}(L, r_1, B_w, B_b)$, 
$\bm{g}_2(\cdot) : \mathbb{R}^d \to \mathbb{R}^d$ in $\tilde{\mathcal{F}}_{\bm{\rho}_{\nu}}^{\operatorname{DNN}}(L, r_2, B_w, B_b)$ and $\bm{g}_3(\cdot) : \mathbb{R}^d \to \mathbb{R}^d$ in $\tilde{\mathcal{F}}_{\bm{\rho}_{\nu}}^{\operatorname{DNN}}(L, r_3, B_w, B_b)$,
we have
\begin{align}
    \begin{split} \label{test_property_3}
    f_{test}(\bm{g}_1(\cdot), \bm{g}_2(\cdot), \bm{g}_3(\cdot),s)
    \in \tilde{\mathcal{F}}_{\bm{\rho}_{\nu}}^{\operatorname{DNN}}
    \Bigg(L+2, \max(r_1 + r_2 + r_3 ,8d + 4),& \\ 
    \max\left(B_w, \frac{R^2}{1-\nu^2}\right), B_b & \Bigg).
    \end{split}
\end{align}
\end{enumerate}
\end{lemma}

\begin{proof}
For the network $f_{\rho_0}$ defined on Lemma \ref{aux_network_basic} with $k=1$, we define
\begin{align*}
f_{ind, [\bm{b}_1, \bm{b}_2)}(\bm{x}) &:= f_{\rho_0} \bigg(1-R \cdot \sum_{i=1}^d \left(f_{\rho_0}\left(b_1^{(i)} + 1/R - x^{(i)}\right)
 + f_{\rho_0}\left(x^{(i)} - b_2^{(i)} + 1/R \right) \right)\bigg) \\
 &= \rho_0 \bigg(1-R \cdot \sum_{i=1}^d \left(\rho_0 \left(b_1^{(i)} + 1/R - x^{(i)}\right)
 + \rho_0 \left(x^{(i)} - b_2^{(i)} + 1/R \right) \right)\bigg),
\end{align*}
and
\begin{align*}
f_{test}(\bm{x}, \bm{b}_1, \bm{b}_2, s) :=  f_{\rho_0} \bigg(f_{id}(s)-R^2 \cdot \sum_{i=1}^d \left(f_{\rho_0} \left(b_1^{(i)} + 1/R - x^{(i)}\right) + f_{\rho_0} \left(x^{(i)} - b_2^{(i)} + 1/R \right)\right)\bigg)\\
 - f_{\rho_0} \bigg(-f_{id}(s)-R^2 \cdot \sum_{i=1}^d \left(f_{\rho_0}\left(b_1^{(i)} + 1/R - x^{(i)}\right) + f_{\rho_0} \left(x^{(i)} - b_2^{(i)} + 1/R\right)\right)\bigg)\\
 =  \rho_0 \bigg(f_{id}(s)-R^2 \cdot \sum_{i=1}^d \left(\rho_0 \left(b_1^{(i)} + 1/R - x^{(i)}\right)\right.
\left. + \rho_0 \left(x^{(i)} - b_2^{(i)} + 1/R \right)\right)\bigg)\\
 - \rho_0 \bigg(-f_{id}(s)-R^2 \cdot \sum_{i=1}^d \left(\rho_0 \left(b_1^{(i)} + 1/R - x^{(i)}\right)\right.
\left.  + \rho_0 \left(x^{(i)} - b_2^{(i)} + 1/R \right)\right)\bigg).
\end{align*}
Then,
(\ref{ind_property_1}), (\ref{ind_property_2}), (\ref{test_property_1}) and (\ref{test_property_2}) hold by Lemma 6 of \citet{kohler2021supplementA}, and
(\ref{test_property_3}) holds by Lemma \ref{aux_network_basic}.
\end{proof}

\begin{lemma} \label{aux_network_prod}
Let $\nu \in [0,1)$ and sufficiently large $R \in \mathbb{N}$ be given.
\begin{enumerate}
    \item[a)] There exists a neural network 
$f_{mult}(\cdot, \cdot) : \mathbb{R}^2 \to \mathbb{R}$ 
such that 
$$|f_{mult}(\bm{x}) - x^{(1)} x^{(2)}| \leq c_3 \cdot 4^{-R}$$
for $\bm{x} \in [-a,a]^2$, and for every neural network $\bm{g}(\cdot) : \mathbb{R}^d \to [-a,a]^2$ in the class $\tilde{\mathcal{F}}_{\bm{\rho}_{\nu}}^{\operatorname{DNN}}(L, r, B_w, B_b)$,
$$f_{mult}(\bm{g}(\cdot)) \in \tilde{\mathcal{F}}_{\bm{\rho}_{\nu}}^{\operatorname{DNN}}(L+R, \max(r, 24), \max(B_w, c_4), \max(B_b, c_5)),$$
where $c_3$, $c_4$ and $c_5$ are constants not depending on $R$. 
    \item[b)] There exists a neural network $f_{mult,d}(\cdot) :\mathbb{R}^d \to \mathbb{R}$ such that 
$$\left|f_{mult, d}(\bm{x}) - \prod_{i=1}^d x^{(i)} \right| \leq c_6 \cdot 4^{-R}$$
for $\bm{x} \in [-a,a]^d$, and for every neural network $\bm{g}(\cdot) : \mathbb{R}^d \to [-a,a]^d$ in the class $\tilde{\mathcal{F}}_{\bm{\rho}_{\nu}}^{\operatorname{DNN}}(L, r, B_w, B_b)$,
$$f_{mult,d}(\bm{g}(\cdot)) \in \tilde{\mathcal{F}}_{\bm{\rho}_{\nu}}^{\operatorname{DNN}}(L+R \lceil \log_2 (d) \rceil, \max(r, 24d), \max(B_w, c_4), \max(B_b, c_5)),$$
where $c_6$ is a constant not depending on $R$.
    \item[c)] \label{testtest} For $N \in \N$, let $m_1, \dots, m_{\binom{d+N}{d}}$ denote all monomials of the form 
\begin{align*}
\prod_{k=1}^d \left(x^{(k)}\right)^{\alpha_k}
\end{align*}
for some $\alpha_1, \dots, \alpha_d \in \N_0$ such that $\alpha_1+\dots+\alpha_d \leq N$.
For $u_1, \dots, u_{\binom{d+N}{d}} \in [-1,1]$, define 
\begin{align*}
p\left(\bm{x}, y_1, \dots, y_{\binom{d+N}{d}}\right) := \sum_{i=1}^{\binom{d+N}{d}} u_i \cdot y_i \cdot m_i(\bm{x}).
\end{align*}
Then, there exists a neural network 
$f_{p}\left(\cdot, \dots, \cdot \right) : \mathbb{R}^{\left[d+\binom{d+N}{d}\right]} \to \mathbb{R}$
such that 
\begin{align*}
\left|f_{p}\left(\bm{x}, y_{1}, \dots, y_{\binom{d+N}{d}}\right) - p\left(\bm{x}, y_{1}, \dots, y_{\binom{d+N}{d}}\right) \right| \leq c_{7} \cdot 4^{-R}
\end{align*}
for $\bm{x} \in [-a,a]^d$ and $y_i \in [-a,a]$, and for every neural network $\bm{g}_0(\cdot) : \mathbb{R}^d \to [-a,a]^d$ in the class $\tilde{\mathcal{F}}_{\bm{\rho}_{\nu}}^{\operatorname{DNN}}(L, r_0, B_w, B_b)$ and $g_i(\cdot) : \mathbb{R}^d \to [-a,a]$ in the class $\tilde{\mathcal{F}}_{\bm{\rho}_{\nu}}^{\operatorname{DNN}}(L, r_i, B_w, B_b)$ for $i \in [\binom{d+N}{d}]$, we have
\begin{align*}
    & f_{p}\left(\bm{g}_0(\cdot), g_{1}(\cdot), \dots, g_{\binom{d+N}{d}}(\cdot)\right) 
    \\
    & \in  \ \tilde{\mathcal{F}}_{\bm{\rho}_{\nu}}^{\operatorname{DNN}} \Bigg(L+ R \lceil \log_2 (N+1) \rceil, \max \Big(\sum_{i=0}^{\binom{d+N}{d}} r_i ,\
    24 (N+1) \binom{d+N}{d}   \Big),\\
    & \qquad \qquad \qquad \qquad \qquad \qquad \qquad \qquad \qquad \qquad \max(B_w, c_{4}), \max(B_b, c_{5}) \Bigg),
\end{align*}
where $c_7$ is a constant not depending on $R$.
\end{enumerate}
\end{lemma}

\begin{proof}
\begin{enumerate}
    \item[a)] We define $f_{ \land}(\cdot) : \mathbb{R} \to \mathbb{R}$ as
\begin{align*}
f_{\land}(x) :=  2 f_{\rho_0}(x) - 4 f_{\rho_0}(x - 0.5) + 2 f_{\rho_0}(x -1),
\end{align*}
where $f_{\rho_0}$ is defined on Lemma \ref{aux_network_basic}.
Then, we can check that $f_{\land}(\cdot)$ satisfies $f_{\land}(x) = 2x \cdot \mathbb{I}(0\leq x <0.5) + 2(1-x) \cdot \mathbb{I}(0.5 \leq x <1)$ for every $x$.
Also, for any $g(\cdot) : \mathbb{R} \to \mathbb{R}$ in $\tilde{\mathcal{F}}_{\bm{\rho}_{\nu}}^{\operatorname{DNN}}(L, r, B_w, B_b)$,
we have
$f_{\land}(g(\cdot)) \in \tilde{\mathcal{F}}_{\bm{\rho}_{\nu}}^{\operatorname{DNN}}(L+1, \max(r, 6), \max(B_w, \frac{4}{1-\nu^2}), B_b)$.

Now, for $\hat{f}_{1, 0}(x) = \hat{f}_{2, 0}(x) = x$ and $\hat{f}_{3, 0}(x) = 0$, we define
\begin{align*}
    \hat{f}_{1, l}(x) &= f_{id}(\hat{f}_{1,l-1}(x)) \\
    \hat{f}_{2, l}(x) &= f_{\land}(\hat{f}_{2, l-1}(x)) \\
    \hat{f}_{3,l}(x) &= f_{id}(\hat{f}_{3,l-1}(x)) - f_{\land}(\hat{f}_{2,l-1}(x))/2^{2l}
\end{align*}
for $l \in [R-1]$ and
$$f_{sq_{[0,1]}}(x)
=
f_{id}(\hat{f}_{1, R-1}(x)) -  f_{\land}(\hat{f}_{2, R-1}(x))/2^{2R} + f_{id}(\hat{f}_{3, R-1}(x)) .$$

By the proof of Lemma 20 of \citet{kohler2021supplementB}, we obtain
$$|f_{sq_{[0,1]}}(x) - x^2| \leq  2^{-2R-2}$$
for $x \in [0,1]$.
Note that $f_{sq_{[0,1]}}(g(\cdot))$ requires $r$ neurons in each of the $1$st through $L$th hidden layers and $2+6+2$ neurons in each of the $(L+1)$th through $(L+R)$th hidden layers. 
Hence, we have $f_{sq_{[0,1]}}(g(\cdot)) \in \tilde{\mathcal{F}}_{\bm{\rho}_{\nu}}^{\operatorname{DNN}}(L+R, \max(r, 10), \max(B_w, \frac{4}{1-\nu^2}), B_b)$.

Now, with the additional network $f_{tran}: [-2a, 2a] \to [0,1]$ such that
$$
f_{tran}(x) := \frac{x}{4 a}+\frac{1}{2},
$$
we define
\begin{align*}
    f_{sq}(x)
:=& 16 a^2 f_{sq_{[0,1]}} (f_{tran}(x)) - 4 a f_{id}^R (x)- 4a^2.
\end{align*}
Then, for $x \in [-2a,2a]$, we have
\begin{align*}
|f_{sq}(x) - x^2|
= & |(16 a^2 f_{sq_{[0,1]}} (f_{tran}(x)) - 4 a x- 4a^2) - (4 a f_{tran}(x) - 2a)^2|\\
= & |(16 a^2 f_{sq_{[0,1]}} (f_{tran}(x)) - 16 a^2 f_{tran}(x) + 4a^2 ) - (4 a f_{tran}(x) - 2a)^2|\\
\leq & 16 a^2 \cdot |f_{sq_{[0,1]}}(f_{tran}(x)) - (f_{tran}(x))^2| \\
\leq & 16 a^2 \cdot 4^{-R-1} \\
= & 4 a^2 \cdot 4^{-R}.
\end{align*} 
and
$$f_{sq}(g(\cdot)) \in \tilde{\mathcal{F}}_{\bm{\rho}_{\nu}}^{\operatorname{DNN}}(L+R, \max(r, 12), \max(B_w, c_4), \max(B_b, c_5)),$$
where $c_4 = \max(16a^2, \frac{4}{1-\nu^2})$ and $c_5 = 4a^2$. 

Now, we define $f_{mult}(\cdot,\cdot) : [-a,a]^2 \to \mathbb{R}$ as
$$f_{mult}(x,y) := \frac{1}{4} f_{sq}(x+y) - \frac{1}{4} f_{sq}(x-y).$$
Since $|x+y| \leq 2a$ and $|x-y| \leq 2a$, we have
\begin{align*}
    |f_{mult}(x,y) - xy| 
    = & |f_{mult}(x,y) - \frac{1}{4}((x+y)^2 - (x-y)^2)|\\
    \leq & \frac{1}{4} | f_{sq}(x+y) - (x+y)^2| + \frac{1}{4} | f_{sq}(x-y) - (x-y)^2|\\
    \leq & 2a^2 \cdot 4^{-R}
\end{align*}
for any $x,y \in [-a,a]$.
Also, we have
$$f_{mult}(g(\cdot)^{(1)}, g(\cdot)^{(2)}) \in \tilde{\mathcal{F}}_{\bm{\rho}_{\nu}}^{\operatorname{DNN}}(L+R, \max(r, 24), \max(B_w, c_4), \max(B_b, c_5)).$$
\item[b)] For $q=\lceil \log_2(d)\rceil$, we consider $\tilde{\bm{x}} := (x^{(1)}, \dots, x^{(d)}, 1, \dots, 1)^{\top} \in [-a,a]^{2^q}$.  
In the first $R$ layers, we compute
\[
f_{mult}(\tilde{x}^{(1)},\tilde{x}^{(2)}), 
\dots,
f_{mult}(\tilde{x}^{(2^q - 1)},\tilde{x}^{(2^q)}), 
\]
which can be formulated by $R$ hidden layers and $24 \cdot 2^{q-1} \leq 24 d$ neurons in each of the hidden layers.
We define $f_{mult,d} : \mathbb{R}^d \to \mathbb{R}$ as the deep neural network which iteratively pairing those outputs and applying $f_{mult}$.
Then, we have $$f_{mult,d}(\bm{g}(\cdot)) \in \tilde{\mathcal{F}}_{\bm{\rho}_{\nu}}^{\operatorname{DNN}}(L+R \lceil \log_2 (d) \rceil, \max(r, 24d), \max(B_w, c_4), \max(B_b, c_5)).$$
By Lemma 8 of \citet{kohler2021supplementA}, we obtain the other assertion.

\item[c)] Using $f_{mult, d}$, we can construct a neural network $f_{m_i }\left(\cdot, \dots, \cdot \right) : [-a,a]^{\left[d+\binom{d+N}{d}\right]} \to \mathbb{R}$ for $i \in [\binom{d+N}{d}]$  
such that 
\begin{align*}
\left|f_{m_i}\left(\bm{x}, y_{1}, \dots, y_{\binom{d+N}{d}}\right) - y_i \cdot m_i (\bm{x}) \right| \leq c_{8} 4^{-R}
\end{align*}
and
\begin{align}
\begin{split} \label{aux_mi_1}
    &f_{m_i}\left(\bm{g}_0(\cdot), g_{1}(\cdot), \dots, g_{\binom{d+N}{d}}(\cdot)\right) \\
    & \in  \tilde{\mathcal{F}}_{\bm{\rho}_{\nu}}^{\operatorname{DNN}}\Bigg( L+R  \lceil \log_2 (N+1) \rceil, \max\Big(\sum_{i=0}^{\binom{d+N}{d}} r_i, 24(N+1)\Big),
    \max(B_w, c_{4}), \max(B_b, c_{5})\Bigg),
\end{split}
\end{align}
where $c_{8}$ is a constant not depending on $R$.

Now, we define a network $f_{p}\left(\cdot, \dots, \cdot \right) $ as
$$f_{p}\left(\cdot, \dots, \cdot \right) := \sum_{i=1}^{\binom{d+N}{d}} u_i \cdot f_{m_i}\left(\cdot, \dots, \cdot \right),$$
which satisfies
\begin{align*}
& \left|f_{p}\left(\bm{x}, y_{1}, \dots, y_{\binom{d+N}{d}}\right) - p\left(\bm{x}, y_{1}, \dots, y_{\binom{d+N}{d}}\right) \right| \\
\leq & \sum_{i=1}^{\binom{d+N}{d}} |u_i| \cdot \left|  f_{m_i}( \bm{x}, y_{1}, \dots, y_{\binom{d+N}{d}} )
- y_i \cdot m_i (\bm{x}) \right|  \\
\leq & c_{8} \cdot \binom{d+N}{d}  \cdot \max_i |u_i| \cdot 4^{-R}.
\end{align*}
Also, we have the other assertion by (\ref{aux_mi_1}).
\end{enumerate}
\end{proof}

\subsection{Proof for Theorem \ref{thm_approx}}
\label{app_A_approx}

We follow the notations and partitions of \citet{kohler2021supplementA}. 
For a half-open cube $C \subset [-a,a]^d$ with length $s>0$, which is defined by 
\begin{align*}
C = \Big\{ \bm{x} : C_{left}^{(j)} \leq x^{(j)}  < C_{left}^{(j)} + s , \ j \in [d] \Big\},
\end{align*}
we denote the ``bottom left" corner of $C$
by $\bm{C}_{left}$. 
Also, for a half-open cube $C \subset [-a,a]^d$ with length $s$ and $0<\delta<\frac{s}{2}$, we denote $C_{\delta}^0 \subset C$ as the half-open cube which contains all $\bm{x} \in C$ that lie with a distance of at least $\delta$ to the boundaries of $C$. i.e.,
\begin{align*}
C_{\delta}^0 = \Big\{ \bm{x} : C_{left}^{(j)} + \delta \leq x^{(j)} <  C_{left}^{(j)} + s -\delta, \ j \in [d] \Big\}.
\end{align*}

We partition $[-a,a)^d$ into $M^d$ and $M^{2d}$ equal-sizes half-open cubes (i.e., length of $2a/M$ and $2a/ M^2 $). We denote these partitions as 
$$
\mathcal{P}_1:=\{C_{i,1}\}_{i \in [M^d]}
$$
and
$$\mathcal{P}_2:=\{C_{j,2}\}_{j \in [M^{2d}]}.$$
Furthermore, let 
$$\bm{v}_1 , \dots, \bm{v}_{M^d} \in \left\{0, \frac{2a}{M^2}, \dots,  \frac{2a(M-1)}{M^2}\right\}^d$$
be the $d$-dimensional $M^d$ different vectors.
We denote the half-open cubes 
$\tilde{C}_{1, i}, \dots, \tilde{C}_{M^d, i}$ as the half-open cubes of $\mathcal{P}_2$ that are contained in $C_{i,1}$
and ordered in such a way that
$$
(\bm{\tilde{C}}_{k, i})_{left} = (\bm{C}_{i,1})_{left} + \bm{v}_k
$$
holds for all $k \in [M^d]$ and $i \in [M^d]$. 
Then, we have
\begin{align*}
\mathcal{P}_2=\{C_{j,2}\}_{j \in [M^{2d}]} = \{\tilde{C}_{k,i}\}_{k \in [M^d]\}, i \in [M^d]}.
\end{align*}
For a partition $\mathcal{P}$ (=$\mathcal{P}_1$ or $\mathcal{P}_2$) and $\bm{x} \in [-a,a)^d$, 
we denote $C_\mathcal{P} (\bm{x})$ as the half-open cube of $\mathcal{P}$
which includes $\bm{x}$.

\begin{lemma} \label{lemma_approx_inner}    
For $\beta>0$ and $K \geq 1$, let $f \in \mathcal{H}_d^\beta (K)$ be the $\beta$-Hölder class function defined on $[-a,a]^d$. 
For any $\nu \in [0,1)$ and sufficiently large $M \in \N$, there exists a neural network $f_{net, \mathcal{P}_2}(\cdot) : [-a, a]^d \to \mathbb{R}$ such that
$$f_{net, \mathcal{P}_2}(\cdot) \in \mathcal{F}_{\bm{\rho}_{\nu}}^{\operatorname{DNN}} \left(\lceil c_9 \log_2 M \rceil , c_{10} M^d, c_{11} \right)$$
and
$$|f_{net, P_2}(\bm{x}) - f(\bm{x})| \leq c_{12} \frac{1}{M^{2\beta}}$$
holds for all $\bm{x} \in \bigcup_{j \in [M^{2d}]} \left(C_{j,2}\right)_{1/M^{2\beta+2}}^0$
and
$$|f_{net, P_2}(\bm{x})| \leq c_{13}$$
holds for all $\bm{x} \in [-a,a]^d$,
where $c_9, c_{10}, c_{11}, c_{12}$ and $c_{13}$ are constants not depending on $M$.
\end{lemma}

\begin{proof}
We construct the networks
\begin{align}
\begin{split} \label{phi1_phi2}
    \bm{\hat{\phi}}_{1,1} =& \left(\hat{\phi}_{1,1}^{(1)}, \dots, \hat{\phi}_{1,1}^{(d)}\right) = f_{id}^2 (\bm{x}),\\
    \bm{\hat{\phi}}_{2,1} =& \left(\hat{\phi}_{2,1}^{(1)}, \dots, \hat{\phi}_{2,1}^{(d)}\right)=  \sum_{i \in [M^{d}]} (\bm{C}_{i,1})_{left} \cdot f_{ind, {C_{i,1}}}(\bm{x})
\end{split}
\end{align}
using Lemma \ref{aux_network_basic} and Lemma \ref{aux_network_indtest},
where $R$ in Lemma \ref{aux_network_indtest} is chosen as $M$.
Also, for $j \in [M^d]$ and $\bm{l} \in \{\bm{l}_1 , \dots, \bm{l}_{\binom{d+\lfloor \beta \rfloor}{d}} \} := \{ \bm{l} \in \N_0^d$,  $\|\bm{l}\|_1 \leq \lfloor \beta \rfloor \}$, we construct the network
$$\hat{\phi}_{3,1}^{(\bm{l},j)} = \sum_{i \in [M^d]} (\partial^{\bm{l}} f)\left((\bm{\tilde{C}}_{j,i})_{left}\right) \cdot f_{ind,{C_{i,1}}}(\bm{x}).$$
using Lemma \ref{aux_network_indtest}, where $R$ in Lemma \ref{aux_network_indtest} is chosen as $R=M$. Then, 
\begin{align}
\left(\bm{\hat{\phi}}_{1,1}, \bm{\hat{\phi}}_{2,1}, \hat{\phi}_{3, 1}^{(\bm{l}_1,1)}, \dots, \hat{\phi}_{3, 1}^{\left(\bm{l}_{\binom{d+ \lfloor \beta \rfloor }{d}}, M^d \right)}\right) \label{first_layer}
\end{align}
requires $2$ hidden layers, $2d+ M^d \cdot 4d $
neurons in each of the hidden layers, the absolute values of weights are bounded by $\frac{M}{1-\nu^2}$ and the absolute values of biases are bounded by $a+1$. 
In other words, 
$$(\ref{first_layer}) \in \tilde{\mathcal{F}}_{\bm{\rho}_{\nu}}^{\operatorname{DNN}} \left(2 , 2d+ 4dM^d , \frac{M}{1-\nu^2}, a+1\right).$$

Next, on the top of the (\ref{first_layer}), we construct the networks
\begin{align}
\begin{split} \label{second_phi1_phi2}    
    \bm{\hat{\phi}}_{1,2}=& \left(\hat{\phi}_{1,2}^{(1)}, \dots, \hat{\phi}_{1,2}^{(d)}\right) = f_{id}^{2} (\bm{\hat{\phi}}_{1,1}),\\
    \bm{\hat{\phi}}_{2,2}=&\left(\hat{\phi}_{2,2}^{(1)}, \dots, \hat{\phi}_{2,2}^{(d)}\right),
\end{split}
\end{align}
where 
$$\hat{\phi}_{2,2}^{(i)} = \sum_{j=1}^{M^d} f_{test}\left(\bm{\hat{\phi}}_{1,1}, \bm{\hat{\phi}}_{2,1} + \bm{v}_j, \bm{\hat{\phi}}_{2,1} + \bm{v}_j+\frac{2a}{M^2} \cdot \mathbf{1}, \hat{\phi}_{2,1}^{(i)} + v_j^{(i)}\right)$$
is constructed using Lemma \ref{aux_network_indtest}, where $R$ in Lemma \ref{aux_network_indtest} is chosen as $M$.
Also, for $\bm{l} \in \{\bm{l}_1 , \dots, \bm{l}_{\binom{d+\lfloor \beta \rfloor}{d}} \}$, we construct the networks
$$\hat{\phi}_{3, 2}^{(\bm{l})} = \sum_{j=1}^{M^d} f_{test}\left(\bm{\hat{\phi}}_{1,1}, \bm{\hat{\phi}}_{2,1} + \bm{v}_j, \bm{\hat{\phi}}_{2,1} + \bm{v}_j+\frac{2a}{M^2} \cdot \mathbf{1}, \hat{\phi}_{3, 1}^{(\bm{l}, j)}\right)$$
using Lemma \ref{aux_network_indtest}, where $R$ in Lemma \ref{aux_network_indtest} is chosen as $R=M$.
Then, 
\begin{align}
\left(\bm{\hat{\phi}}_{1,2}, \bm{\hat{\phi}}_{2,2}, \hat{\phi}_{3, 2}^{(\bm{l}_1)}, \dots, \hat{\phi}_{3, 2}^{\left(\bm{l}_{\binom{d+\lfloor \beta \rfloor}{d}} \right)}\right) \label{second_layer}
\end{align}
requires $2+2$ hidden layers, $\max\left(2d+ 4dM^d, 2d+ d \cdot M^d \cdot  (8d+4)+ \binom{d+\lfloor \beta \rfloor}{d} \cdot M^{d} \cdot (8d+4)\right)$ neurons in each of the hidden layers, the absolute values of weights are bounded by $\frac{M^2}{1-\nu^2}$ and the absolute values of biases are bounded by $a+1$. 
In other words, 
$$(\ref{second_layer}) \in \tilde{\mathcal{F}}_{\bm{\rho}_{\nu}}^{\operatorname{DNN}} \left(4 ,  2d + (8d+4)M^{d} \cdot \left(\binom{d+\lfloor \beta \rfloor}{d} + d\right)  , \frac{M^2}{1-\nu^2}, a+1 \right).$$ 
Note that by Lemma \ref{aux_network_basic} and Lemma \ref{aux_network_indtest}, we have $\bm{\hat{\phi}}_{1,2}(\bm{x}) = \bm{x}$ for $\bm{x} \in [-a,a]^d$ and $\bm{\hat{\phi}}_{2,2}(\bm{x}) = (C_{\mathcal{P}_2} (\bm{x}))_{left}$ for $\bm{x} \in \bigcup_{i \in [M^{2d}]} (C_{i,2})_{1/M^{2p+2}}^0$.

Then, on the top of the (\ref{second_layer}), we construct the network
\begin{align*}
\tilde{f}_{net, \mathcal{P}_2}(x) = f_p\left(\bm{\hat{\phi}}_{1,2} - \bm{\hat{\phi}}_{2,2}, \hat{\phi}_{3, 2}^{(\bm{l}_1)}, \dots, \hat{\phi}_{3, 2}^{\left(\bm{l}_{\binom{d+\lfloor \beta \rfloor}{d}} \right)} \right)
\end{align*}
using Lemma \ref{aux_network_prod}, 
where the coefficients $u_1, \dots, u_{\binom{d+N}{d}}$ in Lemma \ref{aux_network_prod} are chosen as $u_i = \frac{1}{\bm{l}_i!}$ and $N$ and $R$ in Lemma \ref{aux_network_prod} are chosen as $N=\max(1,\lfloor \beta \rfloor)$ and $R = \lceil \log_2 \left(M^{\beta}\right)\rceil$, respectively.
By Lemma 3 of \citet{kohler2021supplementA}, for all $\bm{x} \in \bigcup_{j \in [M^{2d}]} \left(C_{j,2}\right)_{1/M^{2\beta+2}}^0$
\begin{align}
    |\tilde{f}_{net, P_2}(\bm{x}) - f(\bm{x})| \leq c_{14} \cdot (\max(2a,K))^{4(\lfloor\beta\rfloor+1)} \cdot \frac{1}{M^{2 \beta}}, \label{tilde_net_p2_1}
\end{align}
where $c_{14}$ is a constant not depending on $M$ and for all $\bm{x} \in [-a,a]^d$
\begin{align}
    |\tilde{f}_{net, P_2}(\bm{x})| \leq 1 + e^{2ad} K. \label{tilde_net_p2_2}
\end{align}

Note that $\tilde{f}_{net, \mathcal{P}_2}(\cdot)$ requires $4+ \max(1, \lceil \log_2\left( \lfloor \beta \rfloor +1 \right)\rceil ) \cdot  \lceil \log_2 \left(M^{\beta}\right)\rceil $ hidden layers and $\max\left(2d + (8d+4)M^{d} \cdot \left(\binom{d+\lfloor \beta \rfloor}{d} + d\right),  24(\lfloor \beta \rfloor+1 ) \cdot \binom{d+\lfloor \beta \rfloor}{d} \right)$ neurons in each of the hidden layers.

Also, we define $U_l(\tilde{f}_{net, \mathcal{P}_2}) \in \mathbb{R}^{+}$ for $l \in [L(\tilde{f}_{net, \mathcal{P}_2})+1]$ as
\begin{align*}
    U_l(\tilde{f}_{net, \mathcal{P}_2}) &:= \frac{M^2}{1-\nu^2} && l \in \{ 1,2,3,4,5\}, \\
    U_l(\tilde{f}_{net, \mathcal{P}_2}) &:= c_{4} && l \in \{ 6,\dots,L(\tilde{f}_{net, \mathcal{P}_2}) + 1 \},    
\end{align*}
where $c_4$ is a constant defined on Lemma \ref{aux_network_prod}.
Then, $U_l(\tilde{f}_{net, \mathcal{P}_2})$ satisfies 
$$\max(1, |\operatorname{vec}(W_l(\tilde{f}_{net, \mathcal{P}_2}))|_{\infty})  \leq U_l(\tilde{f}_{net, \mathcal{P}_2}).$$
Now we choose $\zeta_l$ as
$$ \zeta_l :=
\frac{1}{U_l(\tilde{f}_{net, \mathcal{P}_2})} \left( c_{4}^{L(\tilde{f}_{net, \mathcal{P}_2}) +1 -5} \left(\frac{M^2}{1-\nu^2}\right)^5 \right)^{\frac{1}{L(\tilde{f}_{net, \mathcal{P}_2}) + 1}}$$  
and re-scale the parameters of $\tilde{f}_{net, \mathcal{P}_2}$ using Lemma \ref{lemma_equal_network}.
We denote this DNN model as $f_{net, \mathcal{P}_2}$.
Since $\prod_{l=1}^{L(\tilde{f}_{net, \mathcal{P}_2}) + 1} \zeta_l = 1 $, (\ref{tilde_net_p2_1}) and (\ref{tilde_net_p2_2}) also hold for $f_{net, \mathcal{P}_2}$ by Lemma \ref{lemma_equal_network}.   
Also note that
\begin{align*}
    \left( c_{4}^{L(\tilde{f}_{net, \mathcal{P}_2}) -4} \left(\frac{M^2}{1-\nu^2}\right)^5 \right)^{\frac{1}{L(\tilde{f}_{net, \mathcal{P}_2}) + 1}}
    \leq& \frac{c_{4}}{1-\nu^2} (M^{10})^{\frac{1}{\beta \log_{2}(M)}}\\
    =& \frac{c_{4}}{1-\nu^2} 2^{10/\beta} =: c_{15},
\end{align*}
which means $|\bm{\theta}_w (f_{net, \mathcal{P}_2})|_{\infty} \leq c_{15}$ holds.
Finally, since $\prod_{l'=1}^{l} \zeta_{l'} \leq 1$ for every $l \in [L(f_{net, \mathcal{P}_2}) + 1]$, we have $|\bm{\theta}_b (f_{net, \mathcal{P}_2})|_{\infty} \leq a+1$.
By choosing $c_{11} = \max(c_{15}, a+1)$, we get the last assertion.
\end{proof}

However, Lemma \ref{lemma_approx_inner} only holds for 
$\bm{x} \in \bigcup_{j \in [M^{2d}]} \left(C_{j,2}\right)_{1/M^{2\beta+2}}^0$. 
To approximate $f(\bm{x})$ on every $\bm{x} \in [-a,a)^d$, we need additional networks.
The strategy is to approximate $w_{\mathcal{P}_2}(\bm{x}) f(\bm{x})$ rather than $f(\bm{x})$, where $w_{\mathcal{P}_2}(\bm{x})$ is defined by 
\begin{align*}
    w_{\mathcal{P}_2}(\bm{x}) = \prod_{j=1}^d \max\left(0, 1- \frac{M^2}{a} \cdot \left|(C_{\mathcal{P}_{2}}(\bm{x}))_{left}^{(j)} + \frac{a}{M^2} - x^{(j)}\right|\right). 
\end{align*}
Note that $w_{\mathcal{P}_2}(\bm{x})$ takes maximum value 1 at the center of $C_{\mathcal{P}_{2}}(\bm{x})$ and gradually decreases linearly to 0 towards the edge of $C_{\mathcal{P}_{2}}(\bm{x})$.
Also, 
\begin{align}
    w_{\mathcal{P}_2}(\bm{x}) \leq \frac{2}{M^\beta}. \label{wp2_decrease}
\end{align}
holds for $\bm{x} \in \bigcup_{i \in [M^{2d}]} C_{i,2} \setminus (C_{i,2})_{2/M^{2\beta+2}}^0$.

We aim to approximate $w_{\mathcal{P}_2}(\bm{x}) f(\bm{x})$ for every $\bm{x} \in [-a,a)^d$ by following three steps.  
The first step is to construct a network $f_{check, \mathcal{P}_2}$, which ascertains whether $\bm{x}$ falls within the boundaries of $\mathcal{P}_2$ or not.
\begin{lemma} \label{lemma_check}
For any $\nu \in [0,1)$ and sufficiently large $M \in \N$, 
there exists a neural network $f_{check, \mathcal{P}_2}(\cdot) : \mathbb{R}^d \to \mathbb{R}$ such that
$$f_{check, \mathcal{P}_2}(\cdot) \in \mathcal{F}_{\bm{\rho}_{\nu}}^{\operatorname{DNN}} \left(L(f_{net, \mathcal{P}_2}) , c_{16} M^d, c_{17} \right)$$
and
\begin{align} \label{net_check_1}
    f_{check, \mathcal{P}_2}(\bm{x}) = 1
\end{align}
for $\bm{x} \in \bigcup_{i \in [M^{2d}]} C_{i,2} \setminus (C_{i,2})_{1/M^{2\beta+2}}^0$,
\begin{align} \label{net_check_2}
    f_{check, \mathcal{P}_2}(\bm{x}) = 0
\end{align}
for $\bm{x} \in \bigcup_{i \in [M^{2d}]} (C_{i,2})_{2/M^{2\beta+2}}^0$ and
\begin{align} \label{net_check_3}
    |f_{check, \mathcal{P}_2}(\bm{x})| \leq 1
\end{align}
for $\bm{x} \in [-a,a)^d$, where $c_{16}$ and $c_{17}$ are constants not depending on $M$.
\end{lemma}

\begin{proof}
We construct the network
\begin{align*}
\bar{f}_{check, \mathcal{P}_2}(\bm{x})& = 1-f_{\rho_0} \Bigg( - f_{id}^2 \Big(1-\sum_{k=1}^{M^d} f_{ind, (C_{k,1})_{1/M^{2\beta+2}}^0}(\bm{x}) \Big)  \\
& + \sum_{k=1}^{M^d}f_{test}\Big(f_{id}^2(\bm{x}), \bm{\hat{\phi}}_{2,1} +\bm{v}_k +\frac{1}{M^{2 \beta +2}}\cdot \mathbf{1},
 \bm{\hat{\phi}}_{2,1} + \bm{v}_k+\frac{2a}{M^2} \cdot \mathbf{1}-\frac{1}{M^{2 \beta+2}}\cdot \mathbf{1}, 1\Big) \Bigg)
\end{align*}
using Lemma \ref{aux_network_basic}, Lemma \ref{aux_network_indtest} and (\ref{phi1_phi2}), where $R$ in Lemma \ref{aux_network_indtest} is chosen as $M$.
Further, we construct the network
$$\tilde{f}_{check, \mathcal{P}_2}(\bm{x}) = f_{id}^{L(f_{net, \mathcal{P}_2})-5}\left(\bar{f}_{check, \mathcal{P}_2}(\bm{x})\right).$$
Then, by Lemma 10 of \citet{kohler2021supplementA}, (\ref{net_check_1}), (\ref{net_check_2}) and (\ref{net_check_3}) hold for $\tilde{f}_{check, \mathcal{P}_2}$.
Note that $\tilde{f}_{check, \mathcal{P}_2}(\cdot)$ requires $L(f_{net, \mathcal{P}_2})$ hidden layers and
$\max(4dM^d + 2d+4dM^d, 2 + (8d+4)M^d )$ neurons in each of the hidden layers.

Also, we define $U_l(\tilde{f}_{check, \mathcal{P}_2}) \in \mathbb{R}^{+}$ for $l \in [L(\tilde{f}_{check, \mathcal{P}_2})+1]$ as
\begin{align*}
    U_l(\tilde{f}_{check, \mathcal{P}_2}) &:= \frac{M^2}{1-\nu^2} 
    && l \in \{ 1, \dots, 6 \},\\
    U_l(\tilde{f}_{check, \mathcal{P}_2}) &:= 1 
    && l \in \{ 7, \dots, L(f_{net, \mathcal{P}_2})+1  \},   
\end{align*}
which satisfies $\max(1, |\operatorname{vec}(W_l(\tilde{f}_{check, \mathcal{P}_2}))|_{\infty})  \leq U_l(\tilde{f}_{check, \mathcal{P}_2}).$
Now we choose $\zeta_l$ as
$$ \zeta_l :=
\frac{1}{U_l(\tilde{f}_{check, \mathcal{P}_2})} \left( \left(\frac{M^2}{1-\nu^2}\right)^6 \right)^{\frac{1}{L(\tilde{f}_{check, \mathcal{P}_2}) + 1}}$$  
and re-scale the parameters of $\tilde{f}_{check, \mathcal{P}_2}$ using Lemma \ref{lemma_equal_network}.
We denote this DNN model as $f_{check, \mathcal{P}_2}$.
Since $\prod_{l=1}^{L(\tilde{f}_{check, \mathcal{P}_2}) + 1} \zeta_l = 1 $, (\ref{net_check_1}), (\ref{net_check_2}) and (\ref{net_check_3}) also holds for $f_{check, \mathcal{P}_2}$ by Lemma \ref{lemma_equal_network}.   
Also note that
\begin{align*}
     \left( \left(\frac{M^2}{1-\nu^2}\right)^6 \right)^{\frac{1}{L(\tilde{f}_{check, \mathcal{P}_2}) + 1}}
    \leq  &  \frac{1}{1-\nu^2} (M^{12})^{\frac{1}{\beta \log_{2}(M)}} \\
    \leq& \frac{1}{1-\nu^2} 2^{12/\beta} =: c_{18},\\
\end{align*}
which means $|\bm{\theta}_w (f_{check, \mathcal{P}_2})|_{\infty} \leq c_{18}$ holds.
Finally, since $\prod_{l'=1}^{l} \zeta_{l'} \leq 1$ for every $l \in [L(f_{check, \mathcal{P}_2}) + 1]$, we have $|\bm{\theta}_b (f_{check, \mathcal{P}_2})|_{\infty} \leq a+1$.
By choosing $c_{17} = \max(c_{18}, a+1)$, we get the last assertion.
\end{proof}

The second step is to construct a network $f_{w_{\mathcal{P}_2}}(\cdot)$, which approximates $w_{\mathcal{P}_2}(\cdot)$ on $\bm{x} \in \bigcup_{j \in [M^{2d}]} \left(C_{j,2}\right)_{1/M^{2\beta+2}}^0$.
\begin{lemma} \label{lemma_wp2}
For any $\nu \in [0,1)$ and sufficiently large $M \in \N$, there exists a neural network $f_{w_{\mathcal{P}_2}}(\cdot) : \mathbb{R}^d \to \mathbb{R}$ 
such that 
$$f_{w_{\mathcal{P}_2}}(\cdot) \in \mathcal{F}_{\rho_{\nu}}^{\operatorname{DNN}} \left(\lceil c_{19} \log_2 M \rceil , c_{20} M^d , c_{21} \right)$$
and
\begin{align*}
    |f_{w_{\mathcal{P}_2}}(\bm{x}) - w_{\mathcal{P}_2}(\bm{x})| \leq c_{22} \frac{1}{M^{2\beta}} 
\end{align*}
for $\bm{x} \in \bigcup_{i \in [M^{2d}]} (C_{i,2})_{1/M^{2p+2}}^0$ 
and 
$$
|f_{w_{\mathcal{P}_2}}(\bm{x})| \leq 2
$$
for $\bm{x} \in [-a,a)^d$, where $c_{19}, c_{20}, c_{21}$ and $c_{22}$ are constants not depending on $M$.
\end{lemma}

\begin{proof}
For $\bm{\hat{\phi}}_{1,2} = \left(\hat{\phi}_{1,2}^{(1)}, \dots, \hat{\phi}_{1,2}^{(d)}\right)$ and $\bm{\hat{\phi}}_{2,2} = \left(\hat{\phi}_{2,2}^{(1)}, \dots, \hat{\phi}_{2,2}^{(d)}\right)$ considered on (\ref{second_phi1_phi2}), 
we define 
\begin{align*}
f_{w_{{\mathcal{P}_2},j}}(\bm{x}) := f_{\rho_0} \left(\frac{M^2}{a} \cdot \left(
\hat{\phi}_{1,2}^{(j)} - \hat{\phi}_{2,2}^{(j)} \right) \right)
& -2 f_{\rho_0} \left( \frac{M^2}{a} \cdot \left( \hat{\phi}_{1,2}^{(j)} - \hat{\phi}_{2,2}^{(j)} - \frac{a}{M^2} \right) \right) \\
&+ f_{\rho_0} \left( \frac{M^2}{a} \cdot \left( \hat{\phi}_{1,2}^{(j)} - \hat{\phi}_{2,2}^{(j)} - \frac{2a}{M^2} \right) \right)
\end{align*}
for $j \in [d]$, where $f_{\rho_0}(\cdot)$ is defined on Lemma \ref{aux_network_basic}.
Since
$$\max(0,x)-2\max(0,x-1)+\max(0,x-2) = \max(0,1-|1-x|)$$ 
holds for every $x \in \mathbb{R}$, we have  
$$f_{w_{{\mathcal{P}_2},j}}(\bm{x}) = 
\max \left(0, 1- \frac{M^2}{a} \cdot \left|\hat{\phi}_{2,2}^{(j)} + \frac{a}{M^2} - \hat{\phi}_{1,2}^{(j)}\right| \right).$$

Then, we define the network
$$\tilde{f}_{w_{\mathcal{P}_2}}(\bm{x}) := f_{mult,d} \left(f_{w_{{\mathcal{P}_2},1}}(\bm{x}), \dots, f_{w_{{\mathcal{P}_2},d}}(\bm{x})\right)$$ using Lemma \ref{aux_network_prod}, where R in Lemma \ref{aux_network_prod} is chosen as $R= \lceil \log_2(M^{\beta})\rceil$.
By Lemma \ref{aux_network_prod}, we obtain 
$$\left|\tilde{f}_{w_{\mathcal{P}_2}}(\bm{x}) -\prod_{j=1}^d \max \left(0, 1- \frac{M^2}{a} \cdot \left| \hat{\phi}_{2,2}^{(j)} + \frac{a}{M^2} - \hat{\phi}_{1,2}^{(j)} \right| \right) \right| \leq c_{22} \frac{1}{M^{2\beta}},$$
where $c_{22}$ is a constant not depending on $M$.
Using this fact, we have
\begin{align}
    |\tilde{f}_{w_{\mathcal{P}_2}}(\bm{x})| \leq & \left|\tilde{f}_{w_{\mathcal{P}_2}}(\bm{x}) - 
\prod_{j=1}^d f_{w_{{\mathcal{P}_2},j}}(\bm{x})\right| + \left|\prod_{j=1}^d f_{w_{{\mathcal{P}_2},j}}(\bm{x})\right| \nonumber \\
\leq & 1 + \prod_{j=1}^d \max \left(0, 1- \frac{M^2}{a} \cdot \left|\hat{\phi}_{2,2}^{(j)} + \frac{a}{M^2} - \hat{\phi}_{1,2}^{(j)}\right| \right) \nonumber \\
\leq & 2 \label{fwp2_all}
\end{align}
for $\bm{x} \in [-a,a)^d$.
Also, for
$\bm{x} \in \bigcup_{i \in [M^{2d}]} (C_{i,2})_{1/M^{2p+2}}^0$, since
$\hat{\phi}_{1,2}^{(j)} = x^{(j)}$ and $\hat{\phi}_{2,2}^{(j)} = (C_{\mathcal{P}_2} (\bm{x}))^{(j)}_{left}$ holds, we have
\begin{align}
    |\tilde{f}_{w_{\mathcal{P}_2}}(\bm{x}) - w_{\mathcal{P}_2}(\bm{x})| \leq c_{22} \frac{1}{M^{2\beta}} \label{fwp2_inner}
\end{align}
for
$\bm{x} \in \bigcup_{i \in [M^{2d}]} (C_{i,2})_{1/M^{2p+2}}^0$.
Note that $\tilde{f}_{w_{\mathcal{P}_2}}(\cdot)$ requires 
$4+1+ \lceil \log_2 (d) \rceil \lceil \log_2(M^{\beta})\rceil$ hidden layers and $\max(2d+ M^d \cdot d \cdot (8d+4), 6d, 24d)$ neurons in each of the hidden layers.

Also, we define $U_l(\tilde{f}_{w_{\mathcal{P}_2}}) \in \mathbb{R}^{+}$ for  $l \in [L(\tilde{f}_{w_{\mathcal{P}_2}})+1]$ as
\begin{align*}
    U_l(\tilde{f}_{w_{\mathcal{P}_2}}) &:= \frac{M^2}{1-\nu^2} && l \in \{ 1, \dots, 6\}, \\
    U_l(\tilde{f}_{w_{\mathcal{P}_2}}) &:= c_{4} && l \in \{ 7,\dots,L(\tilde{f}_{w_{\mathcal{P}_2}}) + 1 \},    
\end{align*}
where $c_4$ is a constant defined on Lemma  \ref{aux_network_prod}.
Then, $U_l(\tilde{f}_{w_{\mathcal{P}_2}})$ satisfies
$$\max(1, |\operatorname{vec}(W_l(\tilde{f}_{w_{\mathcal{P}_2}}))|_{\infty}) \leq U_l(\tilde{f}_{w_{\mathcal{P}_2}}).$$  
Now we choose $\zeta_l$ as
$$ \zeta_l :=
\frac{1}{U_l(\tilde{f}_{w_{\mathcal{P}_2}})} \left( c_{4}^{L(\tilde{f}_{w_{\mathcal{P}_2}}) + 1 - 6} \left(\frac{M^2}{1-\nu^2}\right)^6 \right)^{\frac{1}{L(\tilde{f}_{w_{\mathcal{P}_2}}) + 1}}$$  
and re-scale the parameters of $\tilde{f}_{w_{\mathcal{P}_2}}$ using Lemma \ref{lemma_equal_network}.
We denote this DNN model as $f_{w_{\mathcal{P}_2}}$.
Since $\prod_{l=1}^{L(\tilde{f}_{w_{\mathcal{P}_2}}) + 1} \zeta_l = 1 $, (\ref{fwp2_all}) and (\ref{fwp2_inner}) also hold for $f_{w_{\mathcal{P}_2}}$ by Lemma \ref{lemma_equal_network}.   
Also note that
\begin{align*}
    \left( c_{4}^{L(\tilde{f}_{w_{\mathcal{P}_2}})-5} \left(\frac{M^2}{1-\nu^2}\right)^6 \right)^{\frac{1}{L(\tilde{f}_{w_{\mathcal{P}_2}}) + 1}}
    \leq&  \frac{c_{4}}{1-\nu^2} (M^{12})^{\frac{1}{\beta \log_{2}(M)}}\\
    =& \frac{c_{4}}{1-\nu^2} 2^{12/\beta} =: c_{23},
\end{align*}
which means $|\bm{\theta}_w (f_{w_{\mathcal{P}_2}}|_{\infty} \leq c_{23}$ holds.
Finally, since $\prod_{l'=1}^{l} \zeta_{l'} \leq 1$ for every $l \in [L(f_{w_{\mathcal{P}_2}}) + 1]$, we have $|\bm{\theta}_b (f_{w_{\mathcal{P}_2}})|_{\infty} \leq a+1$.
By choosing $c_{21} = \max(c_{23}, a+1)$, we get the last assertion.
\end{proof}

The last step is to construct a network $f_{net}$ which approximates $w_{\mathcal{P}_2} \cdot f$ on $\bm{x} \in [-a,a)^d$.
In this step, we use the networks $f_{net, \mathcal{P}_2}$, $f_{check, \mathcal{P}_2}$ and $f_{w_{\mathcal{P}_2}}$, which are defined on Lemma \ref{lemma_approx_inner}, Lemma \ref{lemma_check} and Lemma \ref{lemma_wp2}, respectively.   

\begin{lemma} \label{lemma_f_net}
For $\beta>0$ and $K \geq 1$, let $f \in \mathcal{H}_d^\beta (K)$ be the $\beta$-Hölder class function defined on $[-a,a]^d$. 
For any $\nu \in [0,1)$ and sufficiently large $M \in \N$, there exists a neural network $f_{net}(\cdot) : \mathbb{R}^d \to \mathbb{R}$ such that
$$f_{net}(\cdot) \in \mathcal{F}_{\rho_{\nu}}^{\operatorname{DNN}} \left(\lceil c_{24} \log_2 M \rceil, c_{25} M^d, c_{26} \right)$$
and
\begin{align*}
\left|f_{net}(\bm{x}) - w_{\mathcal{P}_2}(\bm{x}) \cdot f(\bm{x})\right| \leq c_{27} \cdot \frac{1}{M^{2\beta}}  
\end{align*}
hold for $\bm{x} \in [-a,a)^d$, where $c_{24}, c_{25}, c_{26}$ and $c_{27}$ are constants not depending on $M$.
\end{lemma}
  
\begin{proof}
We construct
\begin{align*}
  f_{net, \mathcal{P}_{2}, true}(\bm{x}) :=& f_{\rho_0}\Big(f_{net, \mathcal{P}_{2}}(\bm{x}) - c_{13} \cdot f_{check, \mathcal{P}_{2}}(\bm{x})\Big) \\
  & - f_{\rho_0}\Big(-f_{net, \mathcal{P}_{2}}(\bm{x}) - c_{13} \cdot f_{check, \mathcal{P}_{2}}(\bm{x})\Big),
\end{align*}
where $f_{\rho_0}(\cdot)$ is defined on Lemma \ref{aux_network_basic}, 
$f_{net, \mathcal{P}_{2}}(\cdot)$ and $c_{13}$ are defined on Lemma \ref{lemma_approx_inner} and
$f_{check, \mathcal{P}_{2}}(\cdot)$ is defined on Lemma \ref{lemma_check}.
Note that Since $|f_{net, \mathcal{P}_{2}}(\bm{x})| \leq c_{13}$, we have 
$$f_{net, \mathcal{P}_{2}, true}(\bm{x}) = 0$$ for $\bm{x} \in \bigcup_{i \in [M^{2d}]} C_{i,2} \setminus (C_{i,2})_{1/M^{2\beta+2}}^0$,
$$f_{net, \mathcal{P}_{2}, true}(\bm{x}) = f_{net, \mathcal{P}_{2}}(\bm{x})$$ for $\bm{x} \in \bigcup_{i \in [M^{2d}]} (C_{i,2})_{2/M^{2\beta+2}}^0$ and
$$ \left|f_{net, \mathcal{P}_{2}, true}(\bm{x})\right| \leq \left| f_{net, \mathcal{P}_{2}}(\bm{x}) \right| \leq c_{13}$$
for $\bm{x} \in \bigcup_{i \in [M^{2d}]} (C_{i,2})_{1/M^{2\beta+2}}^0 \setminus (C_{i,2})_{2/M^{2\beta+2}}^0$.
For $L_{\operatorname{diff}} := L(f_{net, \mathcal{P}_{2}, true}) - L(f_{w_{\mathcal{P}_2}})$, we define the network $f_{net}(\cdot)$ as
$$f_{net}(\bm{x}) := f_{mult}\left(
f_{id}^{\max(0,L_{\operatorname{diff}})}(f_{w_{\mathcal{P}_2}}(\bm{x})), f_{id}^{\max(0,-L_{\operatorname{diff}})}(f_{net, \mathcal{P}_{2}, true}(\bm{x}))\right),$$
where $f_{mult}$ is defined on Lemma \ref{aux_network_prod} with $R=\lceil\log_2(M^{\beta})\rceil$, $f_{id}$ is defined on Lemma \ref{aux_network_basic} and $f_{w_{\mathcal{P}_2}}$ is defined on Lemma \ref{lemma_wp2}.
Then, by Lemma \ref{aux_network_prod}, we have
$$\left|f_{net}(\bm{x}) - f_{w_{\mathcal{P}_2}}(\bm{x}) f_{net, \mathcal{P}_{2}, true}(\bm{x})  \right| \leq \frac{c_{28}}{M^{2 \beta}}$$
for some constant $c_{28}$ not depending on $M$. From    
\begin{align}
\left|f_{net}(\bm{x}) - w_{\mathcal{P}_2}(\bm{x}) \cdot f(\bm{x})\right|
\leq & \left|f_{net}(\bm{x}) - f_{w_{\mathcal{P}_2}}(\bm{x}) f_{net, \mathcal{P}_{2}, true}(\bm{x})  \right| \nonumber
\\&+ \left|f_{w_{\mathcal{P}_2}}(\bm{x}) f_{net, \mathcal{P}_{2}, true}(\bm{x}) - w_{\mathcal{P}_2}(\bm{x}) \cdot f(\bm{x})\right|, \label{f_net_common} 
\end{align}    
\begin{enumerate}
    \item For $\bm{x} \in \bigcup_{i \in [M^{2d}]} (C_{i,2})_{2/M^{2\beta+2}}^0$,
    \begin{align*}
        (\ref{f_net_common}) \leq&
        \frac{c_{28}}{M^{2\beta}} + \left|f_{w_{\mathcal{P}_2}}(\bm{x}) \right| \cdot \left| f_{net, \mathcal{P}_{2}, true}(\bm{x}) -  f(\bm{x})\right| + \left|f(\bm{x}) \right| \cdot \left| f_{w_{\mathcal{P}_2}}(\bm{x}) - w_{\mathcal{P}_2}(\bm{x}) \right| \\
        \leq &
        \frac{c_{28}}{M^{2\beta}} + 2 \left| f_{net, \mathcal{P}_{2}}(\bm{x}) -  f(\bm{x})\right| + F \cdot \left| f_{w_{\mathcal{P}_2}}(\bm{x}) - w_{\mathcal{P}_2}(\bm{x}) \right|\\
        \leq& \frac{c_{28}}{M^{2\beta}} + \frac{2 c_{12}}{M^{2\beta}} + \frac{F c_{22}}{M^{2\beta}}  
    \end{align*}
    holds by Lemma \ref{lemma_approx_inner} and Lemma \ref{lemma_wp2}.
    \item For $\bm{x} \in \bigcup_{i \in [M^{2d}]} (C_{i,2})_{1/M^{2\beta+2}}^0 \setminus (C_{i,2})_{2/M^{2\beta+2}}^0$,
    \begin{align*}
        (\ref{f_net_common}) \leq&
        \frac{c_{28}}{M^{2\beta}} + \left|f_{net, \mathcal{P}_{2}, true}(\bm{x}) \right| \cdot \left| f_{w_{\mathcal{P}_2}}(\bm{x}) -  w_{\mathcal{P}_2}(\bm{x}) \right| 
        + \left|w_{\mathcal{P}_2}(\bm{x})  \right| \cdot \left| f_{net, \mathcal{P}_{2}, true}(\bm{x}) - f(\bm{x}) \right| \\
        \leq &
        \frac{c_{28}}{M^{2\beta}} + c_{13} \cdot \left| f_{w_{\mathcal{P}_2}}(\bm{x}) - w_{\mathcal{P}_2}(\bm{x}) \right| + |w_{\mathcal{P}_2}(\bm{x})| \cdot (c_{13} + F) \\
        \leq&  \frac{c_{28}}{M^{2\beta}} + \frac{c_{13} c_{22}}{M^{2\beta}} + \frac{2 (c_{true} + F)}{M^{2\beta}}  
    \end{align*}
    holds by Lemma \ref{lemma_approx_inner}, Lemma \ref{lemma_wp2} and (\ref{wp2_decrease}).
    \item For $\bm{x} \in \bigcup_{i \in [M^{2d}]} C_{i,2} \setminus (C_{i,2})_{1/M^{2\beta+2}}^0$,
    \begin{align*}
        (\ref{f_net_common}) =&  \left| f_{net}(\bm{x}) \right| + \left| w_{\mathcal{P}_2}(\bm{x}) \cdot f(\bm{x}) \right| \\
        & \leq  \frac{c_{28}}{M^{2\beta}} + F \cdot |w_{\mathcal{P}_2}(\bm{x})|\\
        & \leq  \frac{c_{28}}{M^{2\beta}} + \frac{2F}{M^{2\beta}},
    \end{align*}
    holds by Lemma \ref{aux_network_prod} and (\ref{wp2_decrease}).
\end{enumerate}

In conclusion, there exists $c_{27}$ not depending on $M$ such that
\begin{align*}
\left|f_{net}(\bm{x}) - w_{\mathcal{P}_2}(\bm{x}) \cdot f(\bm{x})\right| \leq c_{27} \cdot \frac{1}{M^{2\beta}}.
\end{align*}
Also, since we have
$$f_{net, \mathcal{P}_{2}, true} \in \mathcal{F}_{\rho_{\nu}}^{\operatorname{DNN}} \left(\lceil c_{9} \log_2 M \rceil + 1, (c_{10}+c_{16}) M^d , \max(c_{11}, c_{17}, c_{13}) \right),$$
we obtain 
$$f_{net} \in \mathcal{F}_{\bm{\rho}_{\nu}}^{\operatorname{DNN}} \left( \lceil c_{24} \log_2 M \rceil, c_{25} M^d , c_{26} \right),$$
where $c_{24}, c_{25}$ and $c_{26}$ are constants not depending on $M$.
\end{proof}

\begin{proof}[Theorem \ref{thm_approx}]
We partition $[-2a,2a)^d$ into $M^{2d}$ equal-sizes half-open cubes (i.e., length of $4a/ M^2$).
We denote this partition as 
$$
\mathcal{P}_3:=\{C_{j,3}\}_{j \in [M^{2d}]}.
$$ 
Furthermore, let 
$$\bm{u}_1 , \dots, \bm{u}_{2^d} \in \left\{0, \frac{2a}{M^2}\right\}^d$$
be the $d$-dimensional $2^d$ different vectors.
For $k \in [2^d]$ and $j \in [M^{2d}]$, we define 
$$C_{j,3,k} := \{ \bm{x} \in \mathbb{R}^d : \bm{x}-\bm{u}_k \in C_{j,3} \}$$
as the slightly shifted block, and define
$$\mathcal{P}_{3,k} := \{C_{j,3,k}\}_{j \in [M^{2d}]}$$
as the slightly shifted partition. 
In other words, $\mathcal{P}_{3,k}$ is the partition of 
$$\mathcal{X}_k := \{\bm{x} \in \mathbb{R}^d : \bm{x} - \bm{u}_k \in  [-2a,2a)^d \}$$ 
with equal-sizes half-open cubes $C_{1,3,k}, \dots, C_{M^{2d}, 3, k}$. 

We can apply Lemma \ref{lemma_f_net} to partitions $\mathcal{P}_{3,k}$ for $k \in 2^d$, instead of $\mathcal{P}_{2}$. 
In other words, there exist neural networks $f_{net, k}(\cdot) : \mathbb{R}^d \to \mathbb{R}$ for $k \in [2^d]$ such that
$$f_{net, k}(\cdot) \in \mathcal{F}_{\rho_{\nu}}^{\operatorname{DNN}} \left(c_{29} \log_2 M, c_{30} M^d, c_{31} \right)$$
and
\begin{align*}
\left|f_{net,k}(\bm{x}) - w_{\mathcal{P}_{2,k}}(\bm{x}) \cdot f(\bm{x})\right| \leq c_{32} \cdot \frac{1}{M^{2\beta}}  
\end{align*}
holds for $\bm{x} \in \mathcal{X}_k$ and hence for $\bm{x} \in [-a, a)^d$, where $c_{29}, c_{30}, c_{31}$ and $c_{32}$ are constants not depending on $M$ and $k$ and
\begin{align*}
    w_{\mathcal{P}_{2,k}}(\bm{x}) = \prod_{j=1}^d \max\left(0, 1- \frac{M^2}{2a} \cdot \left|(C_{\mathcal{P}_{2,k}}(\bm{x}))_{left}^{(j)} + \frac{2a}{M^2} - x^{(j)}\right|\right). 
\end{align*}
For any $j \in [d]$, note that
\begin{enumerate}
    \item If $(C_{\tilde{\mathcal{P}}_2}(\bm{x}))_{left}^{(j)} \leq x^{(j)} < (C_{\tilde{\mathcal{P}}_2}(\bm{x}))_{left}^{(j)} + \frac{2a}{M^2}$, 
    then the half of $\{(C_{\mathcal{P}_{2,k}}(\bm{x}))_{left}^{(j)}\}_{k=1}^{2^d}$ 
    have value $(C_{\tilde{\mathcal{P}}_2}(\bm{x}))_{left}^{(j)}$ 
    and the other half have value $(C_{\tilde{\mathcal{P}}_2}(\bm{x}))_{left}^{(j)} -\frac{2a}{M^2}$. Also,  
    $$\frac{M^2}{2a} \left|    (C_{\tilde{\mathcal{P}}_2}(\bm{x}))_{left}^{(j)} + \frac{2a}{M^2} - x^{(j)}\right| + \frac{M^2}{2a}\left|\left((C_{\tilde{\mathcal{P}}_2}(\bm{x}))_{left}^{(j)} - \frac{2a}{M^2} \right) + \frac{2a}{M^2} - x^{(j)}\right| = 1.$$
    \item If $(C_{\tilde{\mathcal{P}}_2}(\bm{x}))_{left}^{(j)} + \frac{2a}{M^2} \leq x^{(j)} < (C_{\tilde{\mathcal{P}}_2}(\bm{x}))_{left}^{(j)} + \frac{4a}{M^2}$, 
    the half of $\{(C_{\mathcal{P}_{2,k}}(\bm{x}))_{left}^{(j)}\}_{k=1}^{2^d}$ 
    have value $(C_{\tilde{\mathcal{P}}_2}(\bm{x}))_{left}^{(j)}$ 
    and the other half have value $(C_{\tilde{\mathcal{P}}_2}(\bm{x}))_{left}^{(j)} +\frac{2a}{M^2}$. Also,
    $$\frac{M^2}{2a} \left|    (C_{\tilde{\mathcal{P}}_2}(\bm{x}))_{left}^{(j)} + \frac{2a}{M^2} - x^{(j)}\right| + \frac{M^2}{2a}\left|\left((C_{\tilde{\mathcal{P}}_2}(\bm{x}))_{left}^{(j)} + \frac{2a}{M^2} \right) + \frac{2a}{M^2} - x^{(j)}\right| = 1$$
\end{enumerate}
Hence, by factorization we have
\begin{align*}
    \sum_{k=1}^{2^d} w_{\mathcal{P}_{2,k}}(\bm{x}) =& \sum_{k=1}^{2^d} \prod_{j=1}^d \max\left(0, 1- \frac{M^2}{2a} \cdot \left|(C_{\mathcal{P}_{2,k}}(\bm{x}))_{left}^{(j)} + \frac{2a}{M^2} - x^{(j)}\right|\right) \\
    =& \sum_{k=1}^{2^d} \prod_{j=1}^d \left(1- \frac{M^2}{2a} \cdot \left|(C_{\mathcal{P}_{2,k}}(\bm{x}))_{left}^{(j)} + \frac{2a}{M^2} - x^{(j)}\right|\right) \\
    =& \prod_{j=1}^d \left(1+1-1\right) = 1.
\end{align*}
Hence, by defining
\begin{align*}
    f_{\hat{\bm{\theta}}, \bm{\rho}_{\nu}}^{\operatorname{DNN}} (\bm{x}) := \sum_{k=1}^{2^d} f_{net,k}(\bm{x}),
\end{align*}
we obtain 
$$\left\|f_{\hat{\bm{\theta}}, \bm{\rho}_{\nu}}^{\operatorname{DNN}} - f \right\|_{\infty, [-a,a]^d} \leq  c_{32} \cdot 2^d \cdot \frac{1}{M^{2 \beta}}$$
and 
$$f_{\hat{\bm{\theta}}, \bm{\rho}_{\nu}}^{\operatorname{DNN}} \in \mathcal{F}_{\rho_{\nu}}^{\operatorname{DNN}} \left(\lceil c_{29} \log_2 M \rceil, c_{30} 2^d M^d, c_{31} \right).$$
\end{proof}

\newpage
\section{Proofs for Section \ref{sec4}}
\renewcommand{\theequation}{B.\arabic{equation}}

In this section, we prove the examples and theorems presented in Section \ref{sec4}.
In Section \ref{app_proof_prior_ex},
we demonstrate that prior distributions in Example \ref{example_indep}, \ref{example_hie} and \ref{example_mul} satisfy Assumption \ref{assumption_theta_prior}. 
In Section \ref{app_B_aux_lemma}, we describe auxiliary lemmas for demonstrating the concentration results.
In Section \ref{app_B_con_reg}, \ref{app_B_con_cla} and \ref{sec_proof_con_comp}, we prove Theorem \ref{thm_regression}, \ref{thm_classification} and \ref{thm_con_comp}, respectively.

\subsection{Proofs for examples in Section \ref{sec4_1}} \label{app_proof_prior_ex}
\begin{proof}[Example \ref{example_indep}]
    For any $\bm{\theta} \in [-\kappa,\kappa]^{T}$, we have
    \begin{align*}
        \pi(\bm{\theta}) = \prod_{t=1}^T \pi^{(t)}(\theta^{(t)})
        \geq \delta_\kappa^T
    \end{align*}
    and hence Assumption \ref{assumption_theta_prior} holds.
\end{proof}

\begin{proof}[Example \ref{example_hie}]
    We denote $\pi(\cdot|\bm{\psi})$ as the probability density function of $\Pi_{\bm{\theta}|\bm{\psi}}$. 
    Then, for every $\bm{\theta} \in [-\kappa, \kappa]^{T}$, we have
    \begin{align*}
        \pi(\bm{\theta}) =& \int_{\bm{\psi} \in \mathbb{R}^S} \pi(\bm{\theta}|\bm{\psi})  d\Pi_{\bm{\psi}}\\
        \geq&  \int_{\bm{\psi} \in \Psi} \pi(\bm{\theta}|\bm{\psi})  d\Pi_{\bm{\psi}} \geq (\delta_1 \delta_\kappa)^T.
    \end{align*}
    Hence, Assumption \ref{assumption_theta_prior} holds by  $\delta_1 \delta_\kappa$.
\end{proof}

\begin{proof}[Example \ref{example_mul}]
We denote $\bm{\mu} \in [-B,B]^T$ and $\Sigma \in \mathbb{R}^{T \times T}$ as the mean vector and covariance matrix of $\Pi_{\bm{\theta}}$, respectively.
Also, we denote $\lambda_1 , \dots, \lambda_T$ as the eigenvalues of $\Sigma$.
For any $\bm{\theta} \in [-\kappa,\kappa]^{T}$, we have 
$$\left|\bm{\theta}-\bm{\mu}\right|_2 \leq (B+\kappa) \sqrt{T}$$ 
and hence
\begin{align*}
    (\bm{\theta}-\bm{\mu})^{\top} \Sigma^{-1} (\bm{\theta} - \bm{\mu}) 
& = (B+\kappa)^2 T \cdot  \left(\frac{\bm{\theta}-\bm{\mu}}{(B+\kappa)\sqrt{T}}\right)^{\top} \Sigma^{-1} \left(\frac{\bm{\theta}-\bm{\mu}}{(B+\kappa)\sqrt{T}}\right) \\
& \leq \frac{(B+\kappa)^2 T}{\lambda_{\min}}.  
\end{align*}
Also, we have
\begin{align*}
    \operatorname{det}(\Sigma) =  \prod_{i=1}^T \lambda_i 
    \leq \left( \lambda_{\max} \right)^{T}.
\end{align*}
To sum up, we obtain
\begin{align*}
    \pi(\bm{\theta}) =& (2 \pi)^{-\frac{T}{2}} \operatorname{det}(\Sigma)^{-\frac{1}{2}} \exp \left(-\frac{1}{2}\left(\bm{\theta} - \bm{\mu}\right)^{\top} \Sigma^{-1}\left(\bm{\theta} - \bm{\mu}\right)\right)\\
    \geq& (2 \pi \lambda_{\max})^{-\frac{T}{2}} \exp 
    \left( - \frac{(B+\kappa)^2 T}{2 \lambda_{\min}} \right).
\end{align*}
Hence, Assumption \ref{assumption_theta_prior} holds by
$$\delta_\kappa = \frac{1}{\sqrt{2 \pi \lambda_{\max}} }\exp\left( -\frac{(B+\kappa)^2}{2 \lambda_{\min}} \right).$$
\end{proof}

\subsection{Auxiliary lemmas for posterior concentration results}
\label{app_B_aux_lemma}

\begin{lemma} \label{lemma_similar}
	Consider two DNN models $f_{\bm{\theta}_1, \bm{\rho}_{\nu}}^{\operatorname{DNN}} : [-a,a]^d \to \mathbb{R} , f_{\bm{\theta}_2, \bm{\rho}_{\nu}}^{\operatorname{DNN}} : [-a,a]^d \to \mathbb{R}$ with the $(L , \bm{r})$ architecture, where $L\in \mathbb{N}$, $\bm{r} = (d, r, r, ... , r, 1)^{\top} \in \mathbb{N}^{L+2}$ for some $r \in \mathbb{N}$ and $\nu \in [0,1)$. 
    If $|\bm{\theta}_1|_{\infty} \leq B$, $|\bm{\theta}_2|_{\infty} \leq B$ and $|\bm{\theta}_1 - \bm{\theta}_2|_{\infty} \leq \delta$ holds for some $B>0$ and $\delta>0$, then 
	\begin{equation*}
	\|f_{\bm{\theta}_1, \bm{\rho}_{\nu}}^{\operatorname{DNN}} 
 - f_{\bm{\theta}_2, \bm{\rho}_{\nu}}^{\operatorname{DNN}}\|_{\infty, [-a,a]^d} \leq 
    a (d+1) (r+1)^{L} B^{L} (L+1)\delta
	\end{equation*}
	holds.
\end{lemma}	
\begin{proof}
	We denote
	\begin{align*}
	f_{\bm{\theta}_1, \bm{\rho}_{\nu}}^{\operatorname{DNN}}(\cdot)=  A_{L+1 , 1} \circ \bm{\rho}_{\nu} \circ A_{L,1} \dots \circ \bm{\rho}_{\nu} \circ A_{1,1} (\cdot),\\
	f_{\bm{\theta}_2, \bm{\rho}_{\nu}}^{\operatorname{DNN}}(\cdot)=  A_{L+1 , 2} \circ \bm{\rho}_{\nu} \circ A_{L,2} \dots \circ \bm{\rho}_{\nu} \circ A_{1,2} (\cdot).
	\end{align*}
	Also, for $l \in [L]$, we define $\bm{h}_{\bm{\theta}_1, l, \bm{\rho}_{\nu}} : [-a,a]^d \to \mathbb{R}^{r}$ and $\bm{h}_{\bm{\theta}_2, l, \bm{\rho}_{\nu}} : [-a,a]^d \to \mathbb{R}^{r}$ as the DNN models whose outputs are $l$-th hidden layer of $f_{\bm{\theta}_1, \bm{\rho}_{\nu}}^{\operatorname{DNN}}$ and $f_{\bm{\theta}_1, \bm{\rho}_{\nu}}^{\operatorname{DNN}}$, respectively. 
	That is,    
	\begin{align*}
	\bm{h}_{\bm{\theta}_1, l, \bm{\rho}_{\nu}}(\cdot)=  A_{l , 1} \circ \bm{\rho}_{\nu} \circ A_{l-1,1} \dots \circ \bm{\rho}_{\nu} \circ A_{1,1} (\cdot), \\
	\bm{h}_{\bm{\theta}_2, l, \bm{\rho}_{\nu}}(\cdot)=  A_{l , 2} \circ \bm{\rho}_{\nu} \circ A_{l-1,2} \dots \circ \bm{\rho}_{\nu} \circ A_{1,2} (\cdot).
	\end{align*}
	Also, we denote $\bm{h}_{\bm{\theta}_1, L+1, \bm{\rho}_{\nu}} (\cdot) = f_{\bm{\theta}_1, \bm{\rho}_{\nu}}^{\operatorname{DNN}} (\cdot)$ 
 and $\bm{h}_{\bm{\theta}_2, L+1, \bm{\rho}_{\nu}} (\cdot) = f_{\bm{\theta}_2, \bm{\rho}_{\nu}}^{\operatorname{DNN}} (\cdot)$.
 
	For $l \in [L]$, we have
	\begin{equation*}
	\left\| |\bm{h}_{\bm{\theta_1}, l, \bm{\rho}_{\nu}}|_{\infty}\right\|_{\infty, [-a,a]^d} \leq a (d+1) (r+1)^{l-1} B^{l}.
	\end{equation*}
    and hence 
	\begin{align*}
	 &   \left\| |\bm{h}_{\bm{\theta}_1, l+1, \bm{\rho}_{\nu}} - \bm{h}_{\bm{\theta}_2, l+1, \bm{\rho}_{\nu}} |_{\infty}\right\|_{\infty, [-a,a]^d} \\
     &= \left\| | A_{l+1 , 1} \circ \bm{\rho}_{\nu} \circ \bm{h}_{\bm{\theta}_1, l, \bm{\rho}_{\nu}} - 
        A_{l+1 , 2} \circ \bm{\rho}_{\nu} \circ \bm{h}_{\bm{\theta}_2, l, \bm{\rho}_{\nu}} |_{\infty} \right\|_{\infty, [-a,a]^d} \\
        & \leq \left\| | A_{l+1 , 1} \circ \bm{\rho}_{\nu} \circ \bm{h}_{\bm{\theta}_1, l, \bm{\rho}_{\nu}} - 
        A_{l+1 , 2} \circ \bm{\rho}_{\nu} \circ \bm{h}_{\bm{\theta}_1, l, \bm{\rho}_{\nu}} |_{\infty} \right\|_{\infty, [-a,a]^d} \\
        & \quad + \left\| | A_{l+1 , 2} \circ \bm{\rho}_{\nu} \circ \bm{h}_{\bm{\theta}_1, l, \bm{\rho}_{\nu}} - 
        A_{l+1 , 2} \circ \bm{\rho}_{\nu} \circ \bm{h}_{\bm{\theta}_2, l, \bm{\rho}_{\nu}} |_{\infty} \right\|_{\infty, [-a,a]^d} \\
        & \leq
	    (r+1) |\bm{\theta}_1 - \bm{\theta}_2|_{\infty} 
	    \left\| |\bm{h}_{\bm{\theta}_1, l, \bm{\rho}_{\nu}}|_{\infty}\right\|_{\infty, [-a,a]^d} \\
	    & \quad + r |\bm{\theta}_2|_{\infty}
	    \left\| |\bm{h}_{\bm{\theta}_1, l, \bm{\rho}_{\nu}} - \bm{h}_{\bm{\theta}_2, l, \bm{\rho}_{\nu}}|_{\infty}\right\|_{\infty, [-a,a]^d} \\
        & \leq
	    \delta  
	    a (d+1) (r+1)^{l} B^{l} \\
	    & \quad + r B
	    \left\| |\bm{h}_{\bm{\theta}_1, l, \bm{\rho}_{\nu}} - \bm{h}_{\bm{\theta}_2, l, \bm{\rho}_{\nu}}|_{\infty}\right\|_{\infty, [-a,a]^d}.     
	\end{align*}
    Since 
	\begin{equation*}
	\left\| |\bm{h}_{\bm{\theta}_1, 1, \bm{\rho}_{\nu}} - \bm{h}_{\bm{\theta}_2, 1, \bm{\rho}_{\nu}}|_{\infty}\right\|_{\infty, [-a,a]^d} \leq  
	a (d+1) \delta
	\end{equation*}
    holds, we get
	\begin{equation*}
	\left\| |\bm{h}_{\bm{\theta}_1, l, \bm{\rho}_{\nu}} - \bm{h}_{\bm{\theta}_2, l, \bm{\rho}_{\nu}}|_{\infty}\right\|_{\infty, [-a,a]^d} \leq  
	a (d+1) (r+1)^{l-1} B^{l-1} l\delta
	\end{equation*}
	for every $l \in [L+1]$ by induction.
\end{proof}

\begin{lemma}[Theorem 19.3 of \citet{gyorfi2002distribution}]\label{gyorfi2002distribution}
	Let $\bm{X}, \bm{X}_1, \dots, \bm{X}_n$ be independent and identically distributed random vectors with values in $\mathbb{R}^d$. Let $K_1, K_2 \geq 1$ be constants and let $\mathcal{G}$ be a class of functions $g : \mathbb{R}^d \to \mathbb{R}$ with the properties
	\begin{equation*}
	|g(\bm{x})| \leq K_1 \quad (\bm{x} \in \mathbb{R}^d)  \qquad \text{and} \qquad \mathbb{E}(g(\bm{X})^2) \leq K_2 \mathbb{E}(g(\bm{X})).
	\end{equation*}
	Let $0<\tau<1$ and $\alpha>0$. Assume that
	\begin{equation*}
	\sqrt{n} \tau \sqrt{1-\tau} \sqrt{\alpha} \geq 288 \max \left\{2 K_{1}, \sqrt{2 K_{2}}\right\}
	\end{equation*}
	and that, for all $\bm{x}_1 , \dots , \bm{x}_n \in \mathbb{R}^d$ and for all $t \geq \frac{\alpha}{8}$,
	\begin{align*}
	\frac{\sqrt{n} \tau(1-\tau) t}{96 \sqrt{2} \max \left\{K_{1}, 2 K_{2}\right\}} 
	\geq \int_{\frac{\tau(1-\tau)t}{16 \max \left\{K_{1}, 2 K_{2}\right\}}}^{\sqrt{t}}  
	\sqrt{\log \mathcal{N}\left(u,\left\{g \in \mathcal{G}: \frac{1}{n} \sum_{i=1}^{n} g\left(\bm{x}_{i}\right)^{2} \leq 16 t\right\}, ||\cdot||_{1,n}\right)} d u.
	\end{align*}
	Then,
	\begin{align*}
	\mathbf{P}\left\{\sup _{g \in \mathcal{G}} \frac{\left|\mathbb{E}\{g(\bm{X})\}-\frac{1}{n} \sum_{i=1}^{n} g\left(\bm{X}_{i}\right)\right|}{\alpha+\mathbb{E}\{g(\bm{X})\}}>\tau\right\} \nonumber 
	\leq 60 \exp \left(-\frac{n \alpha \tau^{2}(1-\tau)}{128 \cdot 2304 \max \left\{K_{1}^{2}, K_{2}\right\}}\right).
	\end{align*}
\end{lemma}
\subsection{Proof of Theorem \ref{thm_regression}} \label{app_B_con_reg}

Without loss of generality,
we consider $\gamma$ in $(2, \frac{5}{2})$. 
We define $\mathcal{F}_n$ as the set of pairs of truncated DNN with the $(L_n, \bm{r}_n)$ architecture and variance of the Gaussian noise, 
\begin{align*}
\mathcal{F}_n := \Big\{ \left(T_F \circ f_{\bm{\theta}, \bm{\rho}_{\nu}}^{\operatorname{DNN}} , \sigma^2 \right)^{\top} \ : \ f_{\bm{\theta}, \bm{\rho}_{\nu}}^{\operatorname{DNN}} \in \mathcal{F}^{\operatorname{DNN}}_{\bm{\rho}_{\nu}}(L_n, r_n),\ 0 < \sigma^2 \leq e^{4n \varepsilon_n^2} \Big\},
\end{align*}
where $L_n$ and $r_n$ are defined on (\ref{network_size}).
We denote $T_n$ as the number of parameters in the DNN model with the ($L_n, \bm{r}_n$) architecture. In other words,
$$T_n := \sum_{l=1}^{L_n+1} (r_n^{(l-1)}+1)r_n^{(l)}.$$
For given $\bm{x}^{(n)} = (\bm{x}_1 , \dots, \bm{x}_n)$, we denote 
$P_{(f, \sigma^2), i}$ and $p_{(f, \sigma^2), i}$ as the probability measure and density corresponding to the Gaussian distribution $N(f(\bm{x}_i), \sigma^2)$, respectively. 
For given $\bm{x}^{(n)} = (\bm{x}_1 , \dots, \bm{x}_n)$, we define semimetric $h_n$ on $\mathcal{F}_n$ as the average of the squares of the Hellinger distances between $P_{(f, \sigma^2), i}$. That is,
\begin{align*}
    h_{n}^{2}\left( ( f_1,\sigma_1^2 ), ( f_2 ,\sigma_2^2 ) \right) := \frac{1}{n} \sum_{i=1}^{n} h^2 \left( P_{(f_1, \sigma_1^2), i}, P_{(f_2, \sigma_2^2), i} \right).
\end{align*}  

\begin{lemma} \label{reg_cond1}
    For given $\bm{x}^{(n)} = (\bm{x}_1 , \dots, \bm{x}_n)$, we have
    \begin{align*} 
    \sup _{\varepsilon>\varepsilon_{n}} \log \mathcal{N} \left(\frac{1}{36} \varepsilon,\left\{(f, \sigma^2) \in \mathcal{F}_{n}: h_{n}\left((f, \sigma^2), (f_0, \sigma_0^2)\right)<\varepsilon\right\}, h_{n}\right)
    \lesssim  n \varepsilon_{n}^2
    \end{align*}    
    under the conditions of Theorem \ref{thm_regression}.
\end{lemma}
\begin{proof}
We define semimetric $d_n$ on $\mathcal{F}_n$ as 
\begin{align*}
d_n^2 \left((f_1 , \sigma_1^2 ), (f_2 , \sigma_2^2 )\right) := ||f_1 - f_2||_{1,n} + |\sigma_1^2 - \sigma_2^2|^2.
\end{align*}
Then, $h_n^2(\cdot) \lesssim d_n^2(\cdot)$ holds by by Lemma B.1 of \citet{xie2020adaptive}
and hence $\mathcal{N}\left(\varepsilon ,\mathcal{F}_{n}, h_{n}\right) 
\lesssim  \mathcal{N}\left(\varepsilon^2 ,\mathcal{F}_{n}, d_n^2 \right)$.
Also, by the fact that $||f_1 - f_2||_{1,n} \leq \frac{\varepsilon^2}{2}$ and
$|\sigma_1^2 - \sigma_2^2|^2 \leq \frac{\varepsilon^2}{2}$ implies
$||f_1 - f_2||_{1,n} + |\sigma_1^2 - \sigma_2^2|^2 \leq  \varepsilon^2$,
we get
\begin{align}
\mathcal{N}\left(\varepsilon ,\mathcal{F}_{n}, h_{n}\right) 
\lesssim & \mathcal{N}\left(\varepsilon^2 ,\mathcal{F}_{n}, d_n^2 \right) \nonumber\\
\leq & 
\frac{\sqrt{2}}{\varepsilon} \cdot \exp(4n \varepsilon_{n}^2) \cdot \mathcal{N}\left(\frac{\varepsilon^2}{2} , 
T_F \circ \mathcal{F}^{\operatorname{DNN}}_{\bm{\rho}}(L_n, r_n), ||\cdot||_{1,n}\right) \nonumber \\
\leq & \frac{\sqrt{2}}{\varepsilon} \cdot \exp(4n \varepsilon_{n}^2) \cdot \mathcal{M}\left(\frac{\varepsilon^2}{2} ,T_F \circ \mathcal{F}^{\operatorname{DNN}}_{\bm{\rho}}(L_n, r_n), ||\cdot||_{1,n}\right) \nonumber \\
\leq& \frac{\sqrt{2}}{\varepsilon} \cdot \exp(4n \varepsilon_{n}^2) \cdot
3\left(\frac{8 e F}{\epsilon^{2}} \log \frac{12 e F}{\epsilon^{2}}\right)^{V_{T_F \circ \mathcal{F}^{\operatorname{DNN}}_{\bm{\rho}}(L_n, r_n)}^{+}} \nonumber \\
\leq& \frac{\sqrt{2}}{\varepsilon} \cdot \exp(4n \varepsilon_{n}^2) \cdot
3\left(\frac{8 e F}{\epsilon^{2}} \log \frac{12 e F}{\epsilon^{2}}\right)^{V_{\mathcal{F}^{\operatorname{DNN}}_{\bm{\rho}}(L_n, r_n)}^{+}} \nonumber \\
\leq& \frac{\sqrt{2}}{\varepsilon} \cdot \exp(4n \varepsilon_{n}^2) \cdot
3\left(\frac{8 e F}{\epsilon^{2}} \log \frac{12 e F}{\epsilon^{2}}\right)
^{c_{33} L_n^2 r_n^2 \log (L_n r_n^2)} \label{enp_bound_tmp}	
\end{align}
holds for every $\varepsilon > 0$, 
where the fourth and last inequalities hold by 
Theorem 9.4 of \citet{gyorfi2002distribution} and Theorem 7 of \citet{bartlett2019nearly}, respectively.
Here, $c_{33}>0$ is a constant not depending on $n$. 
Hence, we obtain
\begin{align} 
& \sup _{\varepsilon>\varepsilon_{n}} \log \mathcal{N} \left(\frac{1}{36} \varepsilon,\left\{(f, \sigma^2) \in \mathcal{F}_{n}: h_{n}\left((f, \sigma^2), (f_0, \sigma_0^2)\right)<\varepsilon\right\}, h_{n}\right) \nonumber \\
& \leq 
\log \mathcal{N} \left(\frac{1}{36} \varepsilon_n, \mathcal{F}_{n}, h_{n}\right) \nonumber \\
& \lesssim n \varepsilon_{n}^2 + L_n^2 r_n^2 \log L_n r_n^2 \log n \nonumber\\
& \lesssim  n \varepsilon_{n}^2.  \nonumber
\end{align}  
\end{proof}

\begin{lemma} \label{reg_cond2}
    For given $\bm{x}^{(n)} = (\bm{x}_1 , \dots, \bm{x}_n)$, we define
	\begin{align*}
    \begin{split} 
	K_i((f_0, \sigma_0^2), (f, \sigma^2)) =& \int \log (p_{(f_0, \sigma_0^2), i} / p_{(f, \sigma^2), i}) dP_{(f_0, \sigma_0^2), i},\\
	V_{i} ((f_0, \sigma_0^2), (f, \sigma^2)) =& \int  \left( \log (p_{(f_0, \sigma_0^2), i} /p_{(f, \sigma^2), i}) - K_i ((f_0, \sigma_0^2), (f, \sigma^2))\right)^2 dP_{(f_0, \sigma_0^2), i}     
    \end{split}
	\end{align*}
	and
	\begin{align*} 
	B_{n}^{*}\left((f_0, \sigma_0^2), \varepsilon_n\right)
    =\Bigg\{(f, \sigma^2) \in \mathcal{F}_n:  \frac{1}{n}\sum_{i=1}^n K_i((f_0, \sigma_0^2), (f, \sigma^2)) \leq \varepsilon_n^{2},& \\
	\ \frac{1}{n}\sum_{i=1}^n V_{i} ((f_0, \sigma_0^2), (f, \sigma^2)) \leq \varepsilon_n^{2} & \Bigg\}.
	\end{align*}
    Then, we have
    \begin{align*}
    \Pi_{\bm{\theta}, \sigma^2}\left( B_{n}^{*}\left((f_0 , \sigma_0^2), \varepsilon_n \right) \right) 
    \gtrsim& e^{-n \varepsilon_{n}^2}
    \end{align*}
    under the conditions of Theorem \ref{thm_regression}.
\end{lemma}
\begin{proof}
For $\varepsilon>0$, define
\begin{align*}
A_{n}^{*}\left((f_0 , \sigma_0^2), \varepsilon \right) := \Big\{ (f , \sigma^2) \in \mathcal{F}_n : \underset{i}{\max} |f(\bm{x}_i) - f_0(\bm{x}_i)| \leq \frac{\sigma_0 \varepsilon}{2},
\sigma^2 \in [\sigma_0^2 , (1+\varepsilon^2)\sigma_0^2] \Big\}.
\end{align*}
Then for every $f  \in A_{n}^{*}\left((f_0, \sigma_0^2), \varepsilon \right)$ and $i \in [n],$
\begin{align*}
K_i((f_0, \sigma_0^2), (f, \sigma^2)) = \frac{1}{2} \log \frac{\sigma^2}{\sigma_0^2} + \frac{\sigma_0^2 + (f_0(\bm{x}_i) - f(\bm{x}_i))^2}{2\sigma^2} - \frac{1}{2}
\leq \varepsilon^2
\end{align*} 
and
\begin{align*}
V_{i} ((f_0, \sigma_0^2), (f, \sigma^2)) =& \operatorname{Var}_{f_0, \sigma_0^2}\left(-\frac{(Y_i - f_0(\bm{x}_i))^2}{2\sigma_0^2}
+ \frac{(Y_i - f(\bm{x}_i))^2}{2\sigma^2}\right)\\
=& \operatorname{Var}_{f_0, \sigma_0^2}\left(-\frac{(Y_i - f_0(\bm{x}_i))^2}{2\sigma_0^2}
+ \frac{(Y_i - f_0(\bm{x}_i) + f_0(\bm{x}_i) -f(\bm{x}_i))^2}{2\sigma^2}\right)\\
=& \operatorname{Var}_{f_0, \sigma_0^2}\left( -\frac{1}{2}(1-\frac{\sigma_0^2}{\sigma^2})Z_i^2 + \frac{\sigma_0 (f_0(\bm{x}_i) - f(\bm{x}_i)) Z_i}{\sigma^2}\right)
\leq \varepsilon^2
\end{align*}
where $Z_i := \frac{Y_i - f_0(\bm{x}_i)}{\sigma_0} \sim N(0,1)$.
Hence, we obtain
\begin{align}
    A_{n}^{*}\left((f_0, \sigma_0^2), \varepsilon_n \right) \subset B_{n}^{*}\left((f_0, \sigma_0^2), \varepsilon_n \right). \label{A_in_B}
\end{align}
Also, by Theorem \ref{thm_approx} with $M = n^{\frac{1}{2(2\beta+d)}}$, there exists $f_{\hat{\bm{\theta}}, \bm{\rho}_{\nu}}^{\operatorname{DNN}} \in \mathcal{F}_{\bm{\rho}_{\nu}}^{\operatorname{DNN}}(L_n, r_n, C_B)$ such that
\begin{align}
\left\|f_{\hat{\bm{\theta}}, \bm{\rho}_{\nu}}^{\operatorname{DNN}} - f_{0}\right\|_{\infty, [-a,a]^d} 
\leq & c_1 n^{-\frac{\beta}{(2\beta+d)}} \nonumber \\
< & \frac{\sigma_0 \varepsilon_n}{4} \label{thetahat_def}
\end{align}
satisfies for sufficiently large $n$.
Note that 
$\hat{\bm{\theta}} \in [-C_B, C_B]^{T_n}$ holds. 
With (\ref{A_in_B}), (\ref{thetahat_def}) and Lemma \ref{lemma_similar}, we obtain 
\begin{align*}
&\Pi_{\bm{\theta}, \sigma^2} \left( B_{n}^{*}\left((f_0 , \sigma_0^2), \varepsilon_n \right) \right) \\
&\geq \Pi_{\bm{\theta}, \sigma^2}  \left( A_{n}^{*}\left((f_0 , \sigma_0^2), \varepsilon_n \right) \right) \\
& = \Pi_{\bm{\theta}} \left( \left\{ \bm{\theta} : 
\underset{i}{\max}\ \left|T_F \circ f_{\bm{\theta}, \bm{\rho}_{\nu}}^{\operatorname{DNN}}(\bm{x}_i) - f_0(\bm{x}_i)\right| \leq \frac{\sigma_0 \varepsilon_n}{2} 
\right\} \right) 
\Pi_{\sigma^2} \left(\left\{ \sigma^2 :
\sigma^2 \in [\sigma_0^2 , (1+\varepsilon_n^2)\sigma_0^2] \right\} \right)\\
& \geq \Pi_{\bm{\theta}} \left( \left\{ \bm{\theta} : 
\underset{i}{\max}\ \left|f_{\bm{\theta}, \bm{\rho}_{\nu}}^{\operatorname{DNN}}(\bm{x}_i) - f_0(\bm{x}_i)\right| \leq \frac{\sigma_0 \varepsilon_n}{2} 
\right\} \right) 
\Pi_{\sigma^2} \left(\left\{ \sigma^2 :
\sigma^2 \in [\sigma_0^2 , (1+\varepsilon_n^2)\sigma_0^2] \right\} \right)\\
& \geq \Pi_{\bm{\theta}} \left( \left\{ \bm{\theta} : 
\underset{i}{\max}\ \left|f_{\bm{\theta}, \bm{\rho}_{\nu}}^{\operatorname{DNN}}(\bm{x}_i) -  f_{\hat{\bm{\theta}}, \bm{\rho}_{\nu}}^{\operatorname{DNN}}(\bm{x}_i) \right| \leq \frac{\sigma_0 \varepsilon_n}{4}
\right\} \right) 
\Pi_{\sigma^2} \left(\left\{ \sigma^2 :
\sigma^2 \in [\sigma_0^2 , (1+\varepsilon_n^2)\sigma_0^2] \right\} \right)\\
&\geq  \Pi_{\bm{\theta}} \left( \bm{\theta} : \bm{\theta} \in C^{\star}_n (\hat{\bm{\theta}}) \right)
\Pi_{\sigma^2} \left(\left\{ \sigma^2 :
\sigma^2 \in [\sigma_0^2 , (1+\varepsilon_n^2)\sigma_0^2]\right\} \right),
\end{align*}
where $C^{*}_n (\hat{\bm{\theta}})$ is defined by
$$C^{*}_n (\hat{\bm{\theta}}) := \left\{ \bm{\theta} \in \mathbb{R}^{T_n} : |\bm{\theta}- \hat{\bm{\theta}}|_{\infty} 
\leq \frac{\sigma_0 \varepsilon_n}{4 a (d+1) (r_n + 1)^{L_n} C_B^{L_n} (L_n + 1)} \right\}.$$
By Assumption \ref{assumption_theta_prior}, there exists a constant $\delta_1>0$ such that 
$$\Pi_{\bm{\theta}} \left( C^{*}_n (\hat{\bm{\theta}}) \right) 
\geq \delta_1^{T_n} \left( \frac{\sigma_0 \varepsilon_n}{2 a (d+1) (r_n + 1)^{L_n} C_B^{L_n} (L_n + 1)} \right)^{T_n}.$$
Also, since the density function of $\Pi_{\sigma^2}(\sigma^2)$ is continuous and positive, there exists a constant $\delta_2>0$ such that $\Pi_{\sigma^2} \left(\left\{ \sigma^2 :
\sigma^2 \in [\sigma_0^2 , (1+\varepsilon_n^2)\sigma_0^2]\right\} \right) \geq \delta_2 \varepsilon_n^2$ by the Extreme Value Theorem.
Hence, we obtain
\begin{align}
\Pi_{\bm{\theta}, \sigma^2} \left( B_{n}^{*}\left((f_0 , \sigma_0^2), \varepsilon_n \right) \right) 
 \geq& \Pi_{\bm{\theta}} \left( C^{*}_n (\hat{\bm{\theta}}) \right)
\Pi_{\sigma^2} \left(\left\{ \sigma^2 :
\sigma^2 \in [\sigma_0^2 , (1+\varepsilon_n^2)\sigma_0^2]\right\} \right) \nonumber \\
 \geq & \delta_1^{T_n} \left( \frac{\sigma_0 \varepsilon_n}{2 a (d+1) (r_n + 1)^{L_n} C_B^{L_n} (L_n + 1)} \right)^{T_n} \delta_2 \varepsilon_n^2 \label{tmp_for_adapt_low} \\
\gtrsim& \exp \left(- C_r^2 C_L (\log n) n^{\frac{d}{2\beta+d}  } (\log n)^{2}\right) n^{-1} \nonumber \\
\gtrsim& e^{-n \varepsilon_{n}^2} \label{lower}
\end{align}
for all but finite many $n$.
\end{proof}

\begin{lemma} \label{reg_cond3}
For given $\bm{x}^{(n)} = (\bm{x}_1 , \dots, \bm{x}_n)$, we have
\begin{align*}
\frac{\Pi_{\bm{\theta}, \sigma^2}\left( \mathcal{F}_{n}^{c} \right)}{\Pi_{\bm{\theta}, \sigma^2}\left(B_{n}^{*}\left( (f_0 , \sigma_0^2), \varepsilon_{n} \right)\right)} 
= o(e^{-2n \varepsilon_{n}^2})
\end{align*}
under the conditions of Theorem \ref{thm_regression}.
\end{lemma}

\begin{proof}
    From (\ref{lower}) and
    \begin{align*}
        \Pi_{\bm{\theta}, \sigma^2}\left( \mathcal{F}_{n}^{c} \right) 
        =& \Pi_{ \sigma^2}\left( \sigma^2 > e^{4 n \varepsilon_n^2} \right)\\
        \lesssim & e^{-4 n \varepsilon_n^2},
    \end{align*}
    we obtain the assertion.
\end{proof}

\begin{proof}[Theorem \ref{thm_regression}]
From Lemma \ref{reg_cond1}, Lemma \ref{reg_cond2}, Lemma \ref{reg_cond3} and Theorem 4 of \citet{ghosal2007convergence}, we have 
\begin{equation*}
\mathbb{E}_0 \left[ \Pi_n \left( (f, \sigma^2) : h_n \left((f , \sigma^2 ), (f_0 , \sigma_0^2 )\right) > M_n \varepsilon_{n}  \middle\vert \mathcal{D}^{(n)}\right) \middle\vert \bm{X}^{(n)} = \bm{x}^{(n)}\right] \rightarrow 0
\end{equation*}
for every sequence $\{ \bm{x}^{(n)} \}_{n=1}^{\infty}$, where the expectation is with respect to $\{Y_i\}_{i=1}^n$.
Since
\begin{align*} 
    (||f_1 - f_2||_{2,n} + |\sigma_1^2 - \sigma_2^2|)^2
    &\leq 2 \left( || f_1 - f_2 ||_{2, n}^2 + |\sigma_1^2 - \sigma_2^2 |^2 \right)  \\
     &\lesssim h_n^2 \left((f_1 , \sigma_1^2 ), (f_2 , \sigma_2^2 )\right)
\end{align*}
holds by Lemma B.1 of \citet{xie2020adaptive}, we obtain
\begin{equation*}
\mathbb{E}_0 \left[ \Pi_n \left( (f, \sigma^2) : || f - f_0 ||_{2, n} +
|\sigma^2 - \sigma_0^2 |> M_n \varepsilon_{n}  \middle\vert \mathcal{D}^{(n)}\right) \middle\vert \bm{X}^{(n)} = \bm{x}^{(n)}\right] \rightarrow 0
\end{equation*}
for every sequence $\{ \bm{x}^{(n)} \}_{n=1}^{\infty}$, where the expectation is with respect to $\{Y_i\}_{i=1}^n$.
Note that we can consider
\begin{align}
    \mathbb{E}_0 \left[ \Pi_n \left( (f, \sigma^2) : || f - f_0 ||_{2, n} +
|\sigma^2 - \sigma_0^2 |> M_n \varepsilon_{n}  \middle\vert \mathcal{D}^{(n)}\right) \middle\vert \bm{X}^{(n)}\right] \label{tmp_concen}
\end{align}
as the sequence of bounded random variable.
Since (\ref{tmp_concen}) is uniformly integrable, we have
\begin{equation} \label{empirical_concen2}
\mathbb{E}_0 \left[ \Pi_n \left( (f, \sigma^2) : || f - f_0 ||_{2, n} +
|\sigma^2 - \sigma_0^2 |> M_n \varepsilon_{n}  \middle\vert \mathcal{D}^{(n)}\right) \right] \rightarrow 0,
\end{equation}
where the expectation is with respect to $\{(\bm{X}_i, Y_i)\}_{i=1}^n$.

Next, we will check the conditions in Lemma \ref{gyorfi2002distribution} for
\begin{align*}
\mathcal{G} :=& \left\{ g \ : \ g=(T_F \circ f_{\bm{\theta}, \bm{\rho}_{\nu}}^{\operatorname{DNN}} - f_0)^2 , f_{\bm{\theta}, \bm{\rho}_{\nu}}^{\operatorname{DNN}} \in \mathcal{F}^{\operatorname{DNN}}_{\bm{\rho}_{\nu}}(L_n, r_n) \right\},\\
\tau :=& \frac{1}{2} ,\ \alpha := \varepsilon_{n}^2 ,\ K_1 = K_2 = 4F^2 .
\end{align*}
First, it is easy to check $||g(\bm{x})||_{\infty} \leq 4F^2$ and $\mathbb{E}(g(\bm{X})^2) \leq 4F^2 \mathbb{E}(g(\bm{X}))$ for $g \in \mathcal{G}$.
Also, since 
\begin{align*}
& \left\| (T_F \circ f_{\bm{\theta}_1, \bm{\rho}_{\nu}}^{\operatorname{DNN}} - f_0)^2 
- (T_F \circ f_{\bm{\theta}_2, \bm{\rho}_{\nu}}^{\operatorname{DNN}} - f_0)^2 \right\|_{1, n} \\ 
& =
\left\| (T_F \circ f_{\bm{\theta}_1, \bm{\rho}_{\nu}}^{\operatorname{DNN}} - f_0 + 
T_F \circ f_{\bm{\theta}_2, \bm{\rho}_{\nu}}^{\operatorname{DNN}} - f_0)(T_F \circ f_{\bm{\theta}_1, \bm{\rho}_{\nu}}^{\operatorname{DNN}} - T_F \circ f_{\bm{\theta}_2, \bm{\rho}_{\nu}}^{\operatorname{DNN}} ) \right\|_{1, n}  \\
& \leq 4F \left\| T_F \circ f_{\bm{\theta_1}, \bm{\rho}_{\nu}}^{\operatorname{DNN}}  - T_F \circ f_{\bm{\theta}_2, \bm{\rho}_{\nu}}^{\operatorname{DNN}} \right\|_{1, n}
\end{align*}
holds for any $f_{\bm{\theta}_1, \bm{\rho}_{\nu}}^{\operatorname{DNN}}, f_{\bm{\theta}_2, \bm{\rho}_{\nu}}^{\operatorname{DNN}} \in \mathcal{F}^{\operatorname{DNN}}_{\bm{\rho}_{\nu}}(L_n, r_n)$, there exists $c_{34}>0$ such that
\begin{align*}
\mathcal{N} \left( u, \mathcal{G}, ||\cdot||_{1, n} \right)
\leq& \mathcal{N}\left( \frac{u}{4F}, T_F \circ \mathcal{F}^{\operatorname{DNN}}_{\bm{\rho}_{\nu}}(L_n, r_n), ||\cdot||_{1, n} \right) \\
\leq& \mathcal{M}\left( \frac{u}{4F}, T_F \circ \mathcal{F}^{\operatorname{DNN}}_{\bm{\rho}_{\nu}}(L_n, r_n), ||\cdot||_{1, n} \right) \\
\leq& 3\left(\frac{16 e F^2}{u} \log \frac{24 e F^2}{u}\right)^{\mathcal{F}^{\operatorname{DNN}}_{\bm{\rho}_{\nu}}(L_n, r_n)^{+}}\\
\lesssim& n^{c_{34} L_n^2 r_n^2 \log (L_n r_n^2)}
\end{align*}
for $u \geq n^{-1}$ by Theorem 9.4 of \citet{gyorfi2002distribution} and Theorem 7 of \citet{bartlett2019nearly}.
Hence for all $t \geq \frac{\varepsilon_n^2}{8}$,
\begin{align*}
\int_{\frac{\tau(1-\tau)t}{16 \max \left\{K_{1}, 2 K_{2}\right\}}}^{\sqrt{t}}  
\sqrt{\log \mathcal{N} \left( u, \mathcal{G}, ||\cdot||_{1, n} \right)} d u 
\lesssim& \sqrt{t} \left( n^{\frac{d}{2\beta + d}} (\log n)^{4} \right)^{\frac{1}{2}}\\
=& o\left( \frac{\sqrt{n} \tau (1-\tau) t}{96 \sqrt{2} \max \left\{K_1, 2K_2 \right\}} \right)
\end{align*}
holds.
To sum up, we conclude that
\begin{align}
\mathbf{P}\left\{\sup _{f \in \mathcal{F}^{\operatorname{DNN}}(L_n, r_n)} \frac{\left| ||f-f_0||_{2, \mathrm{P}_{X}}^2  - ||f-f_0||_{2, n}^2 \right|}{\varepsilon_{n}^2+||f-f_0||_{2, \mathrm{P}_{X}}^2}>\frac{1}{2}\right\}
\leq 60 \exp \left(-\frac{n \varepsilon_{n}^2 / 8}{128 \cdot 2304 \cdot 16F^4}\right)
\label{last}
\end{align}
holds for all but finite many $n$ by Lemma \ref{gyorfi2002distribution}. 
By (\ref{empirical_concen2}) and (\ref{last}), we obtain the assertion.
\end{proof}

\subsection{Proof of Theorem \ref{thm_classification}}
\label{app_B_con_cla}

Without loss of generality,
we consider $\gamma$ in $(2, \frac{5}{2})$. 
We define $\mathcal{F}_n$ as the set of truncated DNN with the $(L_n, \bm{r}_n)$ architecture, 
\begin{align*}
\mathcal{F}_n := \Big\{ T_F \circ f_{\bm{\theta}, \bm{\rho}_{\nu}}^{\operatorname{DNN}} \ : \ f_{\bm{\theta}, \bm{\rho}_{\nu}}^{\operatorname{DNN}} \in \mathcal{F}^{\operatorname{DNN}}_{\bm{\rho}_{\nu}}(L_n, r_n) \Big\}, 
\end{align*}
where $L_n$ and $r_n$ are defined on (\ref{network_size}).
We denote $T_n$ as the number of parameters in the DNN model with the ($L_n, \bm{r}_n$) architecture. In other words,
$$T_n := \sum_{l=1}^{L_n+1} (r_n^{(l-1)}+1)r_n^{(l)}.$$
For given $\bm{x}^{(n)} = (\bm{x}_1 , \dots, \bm{x}_n)$, we denote 
$P_{f, i}$ and $p_{f, i}$ as the probability measure and density corresponding to
the Bernoulli distribution $\operatorname{Bernoulli}\left( \phi \circ f (\bm{x}_i)\right)$, respectively.
Also, for given $\bm{x}^{(n)} = (\bm{x}_1 , \dots, \bm{x}_n)$, we define semimetric $h_n$ on $\mathcal{F}_n$ as the average of the squares of the Hellinger distances between $P_{f, i}$. That is,
\begin{align*}
    &h_{n}^{2}\left( f_1, f_2 \right) \\
    &:= \frac{1}{n} \sum_{i=1}^{n} h^2 \left( P_{f_1, i}, P_{f_2, i} \right)\\
    &= \frac{1}{2n} \sum_{i=1}^{n} \left[ \left(\sqrt{\phi \circ f_1(\bm{x}_i) }-\sqrt{\phi \circ f_1(\bm{x}_i)}\right)^2 
    + \left(\sqrt{1-\phi \circ f_1(\bm{x}_i) }-\sqrt{1-\phi \circ f_1(\bm{x}_i)}\right)^2\right]     
\end{align*}

\begin{lemma} \label{cla_cond1}
    For given $\bm{x}^{(n)} = (\bm{x}_1 , \dots, \bm{x}_n)$, we have
    \begin{align*} 
    \sup _{\varepsilon>\varepsilon_{n}} \log \mathcal{N} \left(\frac{1}{36} \varepsilon,\left\{f \in \mathcal{F}_{n}: h_{n}\left(f, f_0\right)<\varepsilon\right\}, h_{n}\right)
    \lesssim  n \varepsilon_{n}^2. 
    \end{align*}    
\end{lemma}
\begin{proof}
We define semimetric $d_n$ on $\mathcal{F}_n$ as 
\begin{align*}
d_n^2 \left(f_1, f_2\right) := ||\phi \circ f_1 - \phi \circ f_2||_{2,n}.
\end{align*}
Since the infinite norms of $f_1 , f_2 \in \mathcal{F}_n$ are bounded by $F$,
\begin{align*}
    & d_{n}^{2}\left(f_1, f_2\right) \\
    & = \frac{1}{2n} \sum_{i=1}^{n} \Bigg\{ (\phi \circ f_1(\bm{x}_i)-\phi \circ f_2(\bm{x}_i))^2
    + ((1 - \phi \circ f_1(\bm{x}_i))- (1 - \phi \circ f_2(\bm{x}_i)))^2 \Bigg\} \\
    &=  \frac{1}{2n} \sum_{i=1}^{n} \Bigg\{ \left( \sqrt{\phi \circ f_1(\bm{x}_i)}-\sqrt{\phi \circ f_2(\bm{x}_i)}\right)^2 
    \left(\sqrt{\phi \circ f_1(\bm{x}_i)}+\sqrt{\phi \circ f_2(\bm{x}_i)}\right)^2  \\
    & \quad  + \left( \sqrt{1-\phi \circ f_1(\bm{x}_i)}-\sqrt{1-\phi \circ f_2(\bm{x}_i)}\right)^2 
    \left(\sqrt{1-\phi \circ f_1(\bm{x}_i)}+\sqrt{1-\phi \circ f_2(\bm{x}_i)}\right)^2\Bigg\}\\
    & \gtrsim h_{n}^{2}\left(f_1, f_2\right)
\end{align*}
holds.
Hence, we obtain
\begin{align*}
\mathcal{N}\left(\varepsilon ,\mathcal{F}_{n}, h_{n}\right) 
\lesssim & \mathcal{N}\left(\varepsilon ,\mathcal{F}_{n}, d_n \right)\\
\leq & \mathcal{N}\left(\varepsilon ,\mathcal{F}_{n}, ||\cdot||_{2,n} \right)\\
\leq & \mathcal{M}\left(\varepsilon ,T_F \circ \mathcal{F}^{\operatorname{DNN}}_{\bm{\rho}}(L_n, r_n), ||\cdot||_{2,n} \right)\\
\leq& 3\left(\frac{8 e F^2}{\epsilon^{2}} \log \frac{12 e F^2}{\epsilon^{2}}\right)^{V_{T_F \circ \mathcal{F}^{\operatorname{DNN}}_{\bm{\rho}}(L_n, r_n)}^{+}}\\
\leq& 3\left(\frac{8 e F^2}{\epsilon^{2}} \log \frac{12 e F^2}{\epsilon^{2}}\right)^{V_{\mathcal{F}^{\operatorname{DNN}}_{\bm{\rho}}(L_n, r_n)}^{+}}\\
\leq& 3\left(\frac{8 e F^2}{\epsilon^{2}} \log \frac{12 e F^2}{\epsilon^{2}}\right)
^{c_{35} L_n^2 r_n^2 \log (L_n r_n^2)}	
\end{align*}
holds for every $\varepsilon > 0$, 
where the second, fourth and last inequalities hold by $1$-Lipschitz continuity of $\phi$,
Theorem 9.4 of \citet{gyorfi2002distribution} and Theorem 7 of \citet{bartlett2019nearly}, respectively.
Here, $c_{35}>0$ is a constant not depending on $n$. 
Hence, we obtain
\begin{align} 
\sup _{\varepsilon>\varepsilon_{n}} \log \mathcal{N} \left(\frac{1}{36} \varepsilon,\left\{f \in \mathcal{F}_{n}: h_{n}\left(f, f_0\right)<\varepsilon\right\}, h_{n}\right) 
\leq &
\log \mathcal{N} \left(\frac{1}{36} \varepsilon_n, \mathcal{F}_{n}, h_{n}\right) \nonumber \\
\lesssim& L_n^2 r_n^2 \log L_n r_n^2 \log n \nonumber\\
\lesssim & n \varepsilon_{n}^2.  \nonumber
\end{align}  
\end{proof}

\begin{lemma} \label{cla_cond2}
    For given $\bm{x}^{(n)} = (\bm{x}_1 , \dots, \bm{x}_n)$, we define
	\begin{align*}
	K_i(f_0, f) =& \int \log (p_{f_0, i} / p_{f, i}) dP_{f_0, i},\\
	V_{i} (f_0, f) =& \int  \left( \log (p_{f_0, i} /p_{f, i}) - K_i (f_0, f)\right)^2 dP_{f_0, i}, \\     
	B_{n}^{*}\left(f_0, \varepsilon_n \right)
    =& \Bigg\{f \in \mathcal{F}_n:  \frac{1}{n}\sum_{i=1}^n K_i(f_0, f) \leq \varepsilon_n^{2},
	\ \frac{1}{n}\sum_{i=1}^n V_{i} (f_0, f) \leq \varepsilon_n^{2}  \Bigg\}.
	\end{align*}
    Then, we have
    \begin{align*}
    \Pi_{\bm{\theta}} \left( B_{n}^{*}\left(f_0, \varepsilon_n \right) \right) 
    \gtrsim& e^{-n \varepsilon_{n}^2}. 
    \end{align*}
\end{lemma}

\begin{proof}
    For $\varepsilon>0$, define
    \begin{align*}
    A_{n}^{*}\left(f_0 , \varepsilon \right) := \Big\{ f \in \mathcal{F}_n : \underset{i}{\max} |f(\bm{x}_i) - f_0(\bm{x}_i)| \leq  \varepsilon \Big\}.
    \end{align*}
    Then by Lemma 3.2 of \citet{van2008rates}, we have
    \begin{align}
        A_{n}^{*}\left(f_{0}, \varepsilon_n \right) \subset B_{n}^{*}\left(f_{0}, \varepsilon_n \right). \label{A_in_B_cla}
    \end{align}
Also, by Theorem \ref{thm_approx} with $M = n^{\frac{1}{2(2\beta+d)}}$, there exists $f_{\hat{\bm{\theta}}, \bm{\rho}_{\nu}}^{\operatorname{DNN}} \in \mathcal{F}_{\bm{\rho}_{\nu}}^{\operatorname{DNN}}(L_n, r_n, C_B)$ such that
\begin{align}
\left\|f_{\hat{\bm{\theta}}, \bm{\rho}_{\nu}}^{\operatorname{DNN}} - f_{0}\right\|_{\infty} 
\leq & c_1 n^{-\frac{\beta}{(2\beta+d)}} \nonumber \\
< & \frac{\varepsilon_n}{2} \label{thetahat_def_cla}
\end{align}
satisfies for sufficiently large $n$.
Note that 
$\hat{\bm{\theta}} \in [-C_B, C_B]^{T_n}$ holds. 
With (\ref{A_in_B_cla}), (\ref{thetahat_def_cla}) and Lemma \ref{lemma_similar}, we obtain 
\begin{align*}
\Pi_{\bm{\theta}} \left( B_{n}^{*}\left(f_0 , \varepsilon_n \right) \right) 
&\geq \Pi_{\bm{\theta}} \left( A_{n}^{*}\left(f_0, \varepsilon_n \right) \right) \\
& = \Pi_{\bm{\theta}} \left( \left\{ \bm{\theta} : 
\underset{i}{\max}\ \left|T_F \circ f_{\bm{\theta}, \bm{\rho}_{\nu}}^{\operatorname{DNN}}(\bm{x}_i) - f_0(\bm{x}_i)\right| \leq \varepsilon_n 
\right\} \right) \\
& \geq \Pi_{\bm{\theta}} \left( \left\{ \bm{\theta} : 
\underset{i}{\max}\ \left|f_{\bm{\theta}, \bm{\rho}_{\nu}}^{\operatorname{DNN}}(\bm{x}_i) - f_0(\bm{x}_i)\right| \leq \varepsilon_n 
\right\} \right) \\
& \geq \Pi_{\bm{\theta}} \left( \left\{ \bm{\theta} : 
\underset{i}{\max}\ \left|f_{\bm{\theta}, \bm{\rho}_{\nu}}^{\operatorname{DNN}}(\bm{x}_i) -  f_{\hat{\bm{\theta}}, \bm{\rho}_{\nu}}^{\operatorname{DNN}}(\bm{x}_i) \right| \leq \frac{\varepsilon_n}{2}
\right\} \right) \\
&\geq  \Pi_{\bm{\theta}} \left( \bm{\theta} : \bm{\theta} \in C^{*}_n (\hat{\bm{\theta}}) \right),
\end{align*}
where $C^{*}_n (\hat{\bm{\theta}})$ is defined by
$$C^{*}_n (\hat{\bm{\theta}}) := \left\{ \bm{\theta} \in \mathbb{R}^{T_n} : |\bm{\theta}- \hat{\bm{\theta}}|_{\infty} 
\leq \frac{\varepsilon_n}{2 a (d+1) (r_n + 1)^{L_n} C_B^{L_n} (L_n + 1)} \right\}.$$
By Assumption \ref{assumption_theta_prior}, there exists a constant $\delta_3>0$ such that 
$$\Pi_{\bm{\theta}} \left( C^{*}_n (\hat{\bm{\theta}}) \right) 
\geq \delta_3^{T_n} \left( \frac{\varepsilon_n}{a (d+1) (r_n + 1)^{L_n} C_B^{L_n} (L_n + 1)} \right)^{T_n}.$$
Hence, we obtain
\begin{align}
\Pi_{\bm{\theta}} \left( B_{n}^{*}\left(f_0, \varepsilon_n \right) \right) 
 \geq& \Pi_{\bm{\theta}} \left( C^{*}_n (\hat{\bm{\theta}}) \right)
\nonumber \\
 \geq & \delta_3^{T_n} \left( \frac{\varepsilon_n}{ a (d+1) (r_n + 1)^{L_n} C_B^{L_n} (L_n + 1)} \right)^{T_n} \nonumber \\
\gtrsim& \exp \left(- C_r^2 C_L (\log n) n^{\frac{d}{2\beta+d}  } (\log n)^{2}\right) \nonumber \\
\gtrsim& e^{-n \varepsilon_{n}^2} \nonumber
\end{align}
for all but finite many $n$.    
\end{proof}

\begin{proof}[Theorem \ref{thm_classification}]
    From Lemma \ref{cla_cond1}, Lemma \ref{cla_cond2}, $\Pi_{n}( \mathcal{F}_{n}^{c}) = 0$ and Theorem 4 of \citet{ghosal2007convergence}, we have 
\begin{equation*}
\mathbb{E}_0 \left[ \Pi_n \left( f : h_n \left(f, f_0 \right) > M_n \varepsilon_{n}  \middle\vert \mathcal{D}^{(n)}\right) \middle\vert \bm{X}^{(n)} = \bm{x}^{(n)}\right] \rightarrow 0
\end{equation*}
for every sequence $\{ \bm{x}^{(n)} \}_{n=1}^{\infty}$, where the expectation is with respect to $\{Y_i\}_{i=1}^n$.
Since
\begin{align*}
    ||\phi \circ f_1 - \phi \circ f_2||_{2,n} 
    =& \frac{1}{n} \sum_{i=1}^{n} (\phi \circ f_1(\bm{x}_i)-\phi \circ f_2(\bm{x}_i))^2 \\
    = &  \frac{1}{n} \sum_{i=1}^{n} \left( \sqrt{\phi \circ f_1(\bm{x}_i)}-\sqrt{\phi \circ f_2(\bm{x}_i)}\right)^2 
    \left(\sqrt{\phi \circ f_1(\bm{x}_i)}+\sqrt{\phi \circ f_2(\bm{x}_i)}\right)^2  \\
    \lesssim & h_{n}^{2}\left(f_1, f_2\right),  
\end{align*}
we obtain
\begin{equation*}
\mathbb{E}_0 \left[ \Pi_n \left( f : || \phi \circ f - \phi \circ f_0 ||_{2, n} > M_n \varepsilon_{n}  \middle\vert \mathcal{D}^{(n)}\right) \middle\vert \bm{X}^{(n)} = \bm{x}^{(n)}\right] \rightarrow 0
\end{equation*}
for every sequence $\{ \bm{x}^{(n)} \}_{n=1}^{\infty}$, where the expectation is with respect to $\{Y_i\}_{i=1}^n$.
Since
\begin{align*}
    \mathbb{E}_0 \left[ \Pi_n \left( f : || \phi \circ f - \phi \circ f_0 ||_{2, n} > M_n \varepsilon_{n}  \middle\vert \mathcal{D}^{(n)}\right) \middle\vert \bm{X}^{(n)} = \bm{x}^{(n)}\right] \label{tmp_concen_cla}
\end{align*}
is uniformly integrable, we have
\begin{equation} \label{empirical_concen2_cla}
\mathbb{E}_0 \left[ \Pi_n \left( f : || \phi \circ f - \phi \circ f_0 ||_{2, n} > M_n \varepsilon_{n}  \middle\vert \mathcal{D}^{(n)}\right) \right] \rightarrow 0,
\end{equation}
where the expectation is with respect to $\{(\bm{X}_i, Y_i)\}_{i=1}^n$.

Next, we will check the conditions in Lemma \ref{gyorfi2002distribution} for
\begin{align*}
\mathcal{G} :=& \left\{ g \ : \ g=( \phi \circ T_F \circ f_{\bm{\theta}, \bm{\rho}_{\nu}}^{\operatorname{DNN}} - \phi \circ f_0)^2 , f_{\bm{\theta}, \bm{\rho}_{\nu}}^{\operatorname{DNN}} \in \mathcal{F}^{\operatorname{DNN}}_{\bm{\rho}_{\nu}}(L_n, r_n) \right\},\\
\tau :=& \frac{1}{2} ,\ \alpha := \varepsilon_{n}^2 ,\ K_1 = K_2 = 1.
\end{align*}
First, it is easy to check $||g(\bm{x})||_{\infty} \leq 1$ and $\mathbb{E}(g(\bm{X})^2) \leq \mathbb{E}(g(\bm{X}))$ for $g \in \mathcal{G}$.
Also, since 
\begin{align*}
& \left\| (\phi \circ T_F \circ f_{\bm{\theta}_1, \bm{\rho}_{\nu}}^{\operatorname{DNN}} - \phi \circ f_0)^2 
- (\phi \circ T_F \circ f_{\bm{\theta}_2, \bm{\rho}_{\nu}}^{\operatorname{DNN}} - \phi \circ f_0)^2 \right\|_{1, n} \\ 
& =
\left\| ( \phi \circ T_F \circ f_{\bm{\theta}_1, \bm{\rho}_{\nu}}^{\operatorname{DNN}} - \phi \circ f_0 + 
\phi \circ T_F \circ f_{\bm{\theta}_2, \bm{\rho}_{\nu}}^{\operatorname{DNN}} - \phi \circ f_0)
(\phi \circ T_F \circ f_{\bm{\theta}_2, \bm{\rho}_{\nu}}^{\operatorname{DNN}} - \phi \circ T_F \circ f_{\bm{\theta}_2, \bm{\rho}_{\nu}}^{\operatorname{DNN}} ) \right\|_{1, n}  \\
& \leq 4 \left\| \phi \circ T_F \circ f_{\bm{\theta_1}, \bm{\rho}_{\nu}}^{\operatorname{DNN}}  - \phi \circ T_F \circ f_{\bm{\theta}_2, \bm{\rho}_{\nu}}^{\operatorname{DNN}} \right\|_{1, n} \\
& \leq 4 \left\|T_F \circ f_{\bm{\theta}_1, \bm{\rho}_{\nu}}^{\operatorname{DNN}}  -  T_F \circ f_{\bm{\theta}_2, \bm{\rho}_{\nu}}^{\operatorname{DNN}} \right\|_{1, n}
\end{align*}
holds for any $f_{\bm{\theta}_1, \bm{\rho}_{\nu}}^{\operatorname{DNN}}, f_{\bm{\theta}_2, \bm{\rho}_{\nu}}^{\operatorname{DNN}} \in \mathcal{F}^{\operatorname{DNN}}_{\bm{\rho}_{\nu}}(L_n, r_n)$, there exists $c_{36}>0$ such that
\begin{align*}
\mathcal{N} \left( u, \mathcal{G}, ||\cdot||_{1, n} \right)
\leq& \mathcal{N}\left( \frac{u}{4}, T_F \circ \mathcal{F}^{\operatorname{DNN}}_{\bm{\rho}_{\nu}}(L_n, r_n), ||\cdot||_{1, n} \right) \\
\leq& \mathcal{M}\left( \frac{u}{4}, T_F \circ \mathcal{F}^{\operatorname{DNN}}_{\bm{\rho}_{\nu}}(L_n, r_n), ||\cdot||_{1, n} \right) \\
\leq& 3\left(\frac{16 e F}{u} \log \frac{24 e F}{u}\right)^{\mathcal{F}^{\operatorname{DNN}}_{\bm{\rho}_{\nu}}(L_n, r_n)^{+}}\\
\lesssim& n^{c_{36} L_n^2 r_n^2 \log (L_n r_n^2)}
\end{align*}
for $u \geq n^{-1}$ by Theorem 9.4 of \citet{gyorfi2002distribution} and Theorem 7 of \citet{bartlett2019nearly}.
Hence for all $t \geq \frac{\varepsilon_n^2}{8}$,
\begin{align*}
\int_{\frac{\tau(1-\tau)t}{16 \max \left\{K_{1}, 2 K_{2}\right\}}}^{\sqrt{t}}  
\sqrt{\log \mathcal{N} \left( u, \mathcal{G}, ||\cdot||_{1, n} \right)} d u 
\lesssim& \sqrt{t} \left( n^{\frac{d}{2\beta + d}} (\log n)^{4} \right)^{\frac{1}{2}}\\
=& o\left( \frac{\sqrt{n} \tau (1-\tau) t}{96 \sqrt{2} \max \left\{K_1, 2K_2 \right\}} \right)
\end{align*}
holds.
To sum up, we conclude that
\begin{align}
\mathbf{P}\left\{\sup _{f \in \mathcal{F}^{\operatorname{DNN}}(L_n, r_n)} \frac{\left| ||\phi \circ f-\phi \circ f_0||_{2, \mathrm{P}_{X}}^2  - ||\phi \circ f- \phi \circ f_0||_{2, n}^2 \right|}{\varepsilon_{n}^2+||\phi \circ f - \phi \circ f_0||_{2, \mathrm{P}_{X}}^2}>\frac{1}{2}\right\}
\leq 60 \exp \left(-\frac{n \varepsilon_{n}^2 / 8}{128 \cdot 2304}\right)
\label{last_cla}
\end{align}
holds for all but finite many $n$ by Lemma \ref{gyorfi2002distribution}. 
Hence by (\ref{empirical_concen2_cla}) and (\ref{last_cla}), we obtain the assertion.

\end{proof}

\subsection{Proof for hierarchical compositional structure Theorem \ref{thm_con_comp}}
\label{sec_proof_con_comp}
The primary advantage of assuming such a hierarchical composition structure is that the dimensions of each $g$ are significantly smaller compared to the overall dimensions. 
This fact allows the function to be approximated with a smaller DNN model, as demonstrated in the following lemma.

\begin{lemma} \label{thm_approx_comp}	
For $\nu \in [0,1)$, $\bm{N} \in \mathbb{N}^q$ and $\mathcal{P} \subset [\beta_{min}, \beta_{max}] \times \{ 1, \dots, d_{max} \}$,  	
there exist positive constants $\tilde{C}_L$, $\tilde{C}_r$, $\tilde{C}_B$ and $c_2$ 
such that for every $f_0$ that follows the hierarchical composition structure $\mathcal{H}(\bm{N},\mathcal{P})$,
	there exists $f_{\hat{\bm{\theta}}, \bm{\rho}_{\nu}}^{\operatorname{DNN}} \in \mathcal{F}_{\bm{\rho}_{\nu}}^{\operatorname{DNN}}(L_n, r_n, \tilde{C}_B)$,
where $L_n$ and $r_n$ are given by
\begin{align*}
    L_n :=& \lceil \tilde{C}_L \log_2 n \rceil, \\
    r_n :=& \lceil \tilde{C}_r \max_{(\beta', d') \in \mathcal{P}} n^{\frac{d'}{2(2\beta' + d')}} \rceil,
\end{align*} 
such that
	\begin{align*}
	\left\|f_{\hat{\bm{\theta}}, \bm{\rho}_{\nu}}^{\operatorname{DNN}}-f_{0}\right\|_{\infty, [-a,a]^d} \leq c_2 \max_{(\beta',d') \in \mathcal{P}} 
    n^{-\frac{\beta'}{2\beta'+d'}}
	\end{align*}
holds.
\end{lemma}



\begin{proof}
    For each $l \in [q]$ and $i \in [N_{l}]$,
    we have $\beta_{min} \leq \beta_{l,i} \leq \beta_{max}$ and
    $1 \leq d_{l,i} \leq d_{max}$.
    Hence, there exist constants $c_{37}>0$, $c_{38}>0$, $c_{39} \geq 1$ and $c_{40} > 0$ such that for any sufficiently large $M_{l,i}$ there exists
    $\hat{g}_{l,i}^{\operatorname{DNN}} \in \mathcal{F}_{\bm{\rho}_{\nu}}^{\operatorname{DNN}}(\lceil c_{37} \log_2 n \rceil, \lceil c_{38} n^{\frac{d_{l,i}}{2(2\beta_{l,i} + d_{l,i})}} \rceil, c_{39})$
	with
	\begin{align}
	\left\|\hat{g}_{l,i}^{\operatorname{DNN}}-g_{l,i}\right\|_{\infty, [-a',a']^d} \leq c_{40} n^{-\frac{\beta_{l,i}}{2\beta_{l,i} + d_{l,i}}} \label{gli_approx}
	\end{align}
    by putting $M = n^{\frac{1}{2(2\beta_{l,i} + d_{l,i})}} $ in Theorem \ref{thm_approx}, where $a' := \max(a, 2F)$.
    From the bottom layer, we sequentially construct a network by 
    $$\hat{f}_{1,i}^{\operatorname{DNN}}(\bm{x}) = \hat{g}_{1,i}^{\operatorname{DNN}} \left( f_{0, \sum_{i'=1}^{i-1} d_{1,i'} + 1}(\bm{x}), \dots, f_{0, \sum_{i'=1}^{i-1} d_{1,i'} + d_{1,i}}(\bm{x})\right)$$
    for $i \in [N_1]$ and
    $$\hat{f}_{l,i}^{\operatorname{DNN}}(\bm{x})=\hat{g}_{l,i}^{\operatorname{DNN}} \left(\hat{f}^{\operatorname{DNN}}_{l-1, \sum_{i'=1}^{i-1} d_{l,i'} + 1}(\bm{x}), \dots, \hat{f}^{\operatorname{DNN}}_{l-1, \sum_{i'=1}^{i-1} d_{l,i'} + d_{l,i}}(\bm{x})\right)$$
    for $l \in \{2,\dots,q\}$ and $i \in [N_l]$.    
    Note that each 
    $$ \left( f_{0, \sum_{i'=1}^{i-1} d_{1,i'} + 1}(\bm{x}), \dots, f_{0, \sum_{i'=1}^{i-1} d_{1,i'} + d_{1,i}}(\bm{x})\right)$$
    for $i \in [N_1]$ is a permutation of $\bm{x}$ (with length $d_{l,i}$).
    Hence, by defining $\tilde{C}_L := 2q \cdot c_{37}$, $\tilde{C}_r := \max_{l \in [q]} N_l \cdot c_{38}$ and $\tilde{C}_B := c_{39}^2$, we obtain
    \begin{align}
        \hat{f}_{q,1}^{\operatorname{DNN}} \in \mathcal{F}^{\operatorname{DNN}}_{\bm{\rho}_{\nu}}\left(\lceil \tilde{C}_L \log_2 n \rceil, \lceil \tilde{C}_r \max_{(\beta, d) \in \mathcal{P}} n^{\frac{d}{2(2\beta + d)}} \rceil, \tilde{C}_B\right). \label{comp_approx_re_1}
    \end{align}
    
    Now, for $l \in [q]$ and $i \in [N_l]$, we will show 
    \begin{align}
        \left\| \hat{f}_{l,i}^{\operatorname{DNN}} -  f_{l,i}\right\|_{\infty, [-a,a]^d} \leq c_{40} l \left( C_{Lip} \sqrt{d_{\max}} \right)^{l-1} 
        \max_{(\beta',d') \in \mathcal{P}} 
    n^{-\frac{\beta'}{(2\beta'+d')}}   \label{thm3_induction}
    \end{align}
    holds by induction.
    First, for $l=1$, we obtain (\ref{thm3_induction}) directly from (\ref{gli_approx}). 
    Assume that (\ref{thm3_induction}) holds for some $l \in [q-1]$ and every $i \in [N_l]$.
    Then, for any $j \in [N_{l+1}]$, we have
    \begin{align*}
        \left| \hat{f}_{l+1,j}^{\operatorname{DNN}} (\bm{x})
        - f_{l+1,j} (\bm{x}) \right| 
        =&  \Bigg| \hat{g}_{l+1,j}^{\operatorname{DNN}} \left( \hat{f}^{\operatorname{DNN}}_{l, \sum_{i'=1}^{j-1} d_{l+1,i'} + 1}(\bm{x}), \dots, \hat{f}^{\operatorname{DNN}}_{l, \sum_{i'=1}^{j-1} d_{l+1,i'} + d_{l+1,j}}(\bm{x}) \right) \\
        & \qquad - g_{l+1,j} \left( f_{l, \sum_{i'=1}^{j-1} d_{l+1,i'} + 1}(\bm{x}), \dots, f_{l, \sum_{i'=1}^{j-1} d_{l+1,i'} + d_{l+1,j}} (\bm{x}) \right) \Bigg| \\
         \leq & \Bigg| \hat{g}_{l+1,j}^{\operatorname{DNN}} \left( \hat{f}^{\operatorname{DNN}}_{l, \sum_{i'=1}^{j-1} d_{l+1,i'} + 1}(\bm{x}), \dots, \hat{f}^{\operatorname{DNN}}_{l, \sum_{i'=1}^{j-1} d_{l+1,i'} + d_{l+1,j}}(\bm{x}) \right) \\
        & \qquad - g_{l+1,j} \left( \hat{f}^{\operatorname{DNN}}_{l, \sum_{i'=1}^{j-1} d_{l+1,i'} + 1}(\bm{x}), \dots, \hat{f}^{\operatorname{DNN}}_{l, \sum_{i'=1}^{j-1} d_{l+1,i'} + d_{l+1,j}}(\bm{x}) \right) \Bigg| \\
        & + \Bigg| g_{l+1,j} \left( \hat{f}^{\operatorname{DNN}}_{l, \sum_{i'=1}^{j-1} d_{l+1,i'} + 1}(\bm{x}), \dots, \hat{f}^{\operatorname{DNN}}_{l, \sum_{i'=1}^{j-1} d_{l+1,i'} + d_{l+1,j}}(\bm{x}) \right) \\
        & \qquad - g_{l+1,j} \left( f_{l, \sum_{i'=1}^{j-1} d_{l+1,i'} + 1}(\bm{x}), \dots, f_{l, \sum_{i'=1}^{j-1} d_{l+1,i'} + d_{l+1,j}} (\bm{x}) \right) \Bigg| \\
        \leq & c_{40} n^{-\frac{\beta_{l+1,j}}{2\beta_{l+1,j} + d_{l+1,j}}} \\
        & + C_{Lip} \sqrt{d_{l+1, j}}  c_{40} l \left( C_{Lip} \sqrt{d_{\max}} \right)^{l-1} \max_{(\beta',d') \in \mathcal{P}} 
    n^{-\frac{\beta'}{(2\beta'+d')}} \\
        \leq & c_{40} (l+1) \left( C_{Lip} \sqrt{d_{\max}} \right)^{l} \max_{(\beta',d') \in \mathcal{P}} 
    n^{-\frac{\beta'}{(2\beta'+d')}}
    \end{align*}
    for any $\bm{x} \in [-a,a]^d$, where the second inequality holds by $||\hat{f}^{\operatorname{DNN}}_{l, i}||_{\infty, [-a',a']^d} \leq 2F \leq a'$ for $i \in [N_l]$ and the Lipschitz condition of $g_{l+1,j}$.
    By defining $c_2 := c_{40} q \left( C_{Lip} \sqrt{d_{\max}} \right)^{q-1} $, we obtain
	\begin{align}
	\left\|\hat{f}_{q,1}^{\operatorname{DNN}} - f_{0} \right\|_{\infty, [-a,a]^d} \leq c_2 \max_{(\beta',d') \in \mathcal{P}} 
    n^{-\frac{\beta'}{(2\beta'+d')}}. \label{comp_approx_re_2}
	\end{align}
    By (\ref{comp_approx_re_1}) and (\ref{comp_approx_re_2}), we obtain the assertion.
\end{proof}

\begin{proof}[Theorem \ref{thm_con_comp}]
We only present results for nonparametric Gaussian regression.
Extending to nonparametric logistic regression can be done similarly to those in Appendix \ref{app_B_con_cla}.
Without loss of generality,
we consider $\gamma$ in $(2, \frac{5}{2})$.  
We define $\mathcal{F}_n$ as the set of pairs of truncated DNN with the $(L_n, \bm{r}_n)$ architecture and variance of the Gaussian noise, 
\begin{align}
\mathcal{F}_n := \Big\{ \left(T_F \circ f_{\bm{\theta}, \bm{\rho}_{\nu}}^{\operatorname{DNN}} , \sigma^2 \right)^{\top} \ : \ f_{\bm{\theta}, \bm{\rho}_{\nu}}^{\operatorname{DNN}} \in \mathcal{F}^{\operatorname{DNN}}_{\bm{\rho}_{\nu}}(L_n, r_n),\ 0 < \sigma^2 \leq e^{4n \varepsilon_n^2} \Big\}, \label{F_n_comp_reg}
\end{align}
where $L_n$ and $r_n$ are defined on (\ref{network_size_comp}).
We denote $T_n$ as the number of parameters in the DNN model with the ($L_n, \bm{r}_n$) architecture. 

First, for given $\bm{x}^{(n)} = (\bm{x}_1 , \dots, \bm{x}_n)$, we obtain
\begin{align*} 
\sup _{\varepsilon>\varepsilon_{n}} \log \mathcal{N} \left(\frac{1}{36} \varepsilon,\left\{(f, \sigma^2) \in \mathcal{F}_{n}: h_{n}\left((f, \sigma^2), (f_0, \sigma_0^2)\right)<\varepsilon\right\}, h_{n}\right)
\lesssim  n \varepsilon_{n}^2 
\end{align*}    
under the conditions of Theorem \ref{thm_con_comp}, by following the proof of Lemma \ref{reg_cond1}.
Also, we define
$K_i((f_0, \sigma_0^2), (f, \sigma^2))$, $V_{i} ((f_0, \sigma_0^2), (f, \sigma^2))$ and $B_{n}^{*}\left((f_0, \sigma_0^2), \varepsilon_n\right)$ in the same way as in Lemma \ref{reg_cond2}, with the only change in the definition of $\mathcal{F}_{n}$ by (\ref{F_n_comp_reg}).
By Lemma \ref{thm_approx_comp}, there exists $f_{\hat{\bm{\theta}}, \bm{\rho}_{\nu}}^{\operatorname{DNN}} \in \mathcal{F}_{\bm{\rho}_{\nu}}^{\operatorname{DNN}}(L_n, r_n, C_B)$ such that
\begin{align}
\left\|f_{\hat{\bm{\theta}}, \bm{\rho}_{\nu}}^{\operatorname{DNN}} - f_{0}\right\|_{\infty, [-a,a]^d} 
\leq & c_2 \max_{(\beta',d') \in \mathcal{P}} 
    n^{-\frac{\beta'}{2\beta'+d'}} \nonumber \\
< & \frac{\sigma_0 \varepsilon_n}{4} \nonumber
\end{align}
satisfies for sufficiently large $n$. 
Hence, we obtain
\begin{align*}
\Pi_{\bm{\theta}, \sigma^2}\left( B_{n}^{*}\left((f_0 , \sigma_0^2), \varepsilon_n \right) \right) 
\gtrsim& e^{-n \varepsilon_{n}^2} 
\end{align*}
and
\begin{align*}
\frac{\Pi_{\bm{\theta}, \sigma^2}\left( \mathcal{F}_{n}^{c} \right)}{\Pi_{\bm{\theta}, \sigma^2}\left(B_{n}^{*}\left( (f_0 , \sigma_0^2), \varepsilon_{n} \right)\right)} 
= o(e^{-2n \varepsilon_{n}^2}) 
\end{align*}
under the conditions of Theorem \ref{thm_con_comp}, by following the proofs of Lemma \ref{reg_cond2} and \ref{reg_cond3}.
The rest of the proof can be completed along the lines of that of Theorem \ref{thm_regression}.
\end{proof}

\newpage

\section{Proof for Theorem \ref{thm_adaptive}} \label{sec_proof_adapt}
\renewcommand{\theequation}{C.\arabic{equation}}

Extending to nonparametric logistic regression is straightforward by following Appendix \ref{app_B_con_cla}, therefore, we only present results for nonparametric Gaussian regression.
Without loss of generality,
we consider $\gamma$ in $(\frac{5}{2}, 3)$.  
We define
$$ \xi_n := \lceil \tilde{C}_r \max_{(\beta', d') \in \mathcal{P}} n^{\frac{d'}{2(2\beta' + d')}} \rceil$$
and
$$ \bm{\xi}_n := (d, \xi_n , \dots, \xi_n, 1)^{\top} \in \mathbb{N}^{L_n + 2},$$
where $\tilde{C}_r$ is a constant (depending on $\beta_{min}$, $\beta_{max}$ and $d_{max}$) defined in Lemma \ref{thm_approx_comp}.
We denote $S_n$ as the number of parameters in the DNN model with the ($L_n, \bm{\xi}_n$) architecture. 
In other words,
$$S_n := \sum_{l=1}^{L_n+1} (\xi_n^{(l-1)}+1)\xi_n^{(l)}.$$
We define the sieve $\mathcal{F}_n$ as
\begin{align}
\mathcal{F}_n := \Big\{ \left(T_F \circ f_{\bm{\theta}, \bm{\rho}_{\nu}}^{\operatorname{DNN}} , \sigma^2 \right)^{\top} \ : \ f_{\bm{\theta}, \bm{\rho}_{\nu}}^{\operatorname{DNN}} \in \bigcup_{r=1}^{\xi_n} \mathcal{F}^{\operatorname{DNN}}_{\bm{\rho}_{\nu}}(L_n, r),\ 0 < \sigma^2 \leq e^{4n \varepsilon_n^2} \Big\}, \label{F_n_adaptive_reg}
\end{align}
where $L_n$ is defined on (\ref{network_size_adapt}).

\begin{lemma} \label{reg_cond1_adapt}
    For given $\bm{x}^{(n)} = (\bm{x}_1 , \dots, \bm{x}_n)$, we have
    \begin{align*} 
    \sup _{\varepsilon>\varepsilon_{n}} \log \mathcal{N} \left(\frac{1}{36} \varepsilon,\left\{(f, \sigma^2) \in \mathcal{F}_{n}: h_{n}\left((f, \sigma^2), (f_0, \sigma_0^2)\right)<\varepsilon\right\}, h_{n}\right)
    \lesssim  n \varepsilon_{n}^2
    \end{align*}    
    under the conditions of Theorem \ref{thm_adaptive}, where $\mathcal{F}_{n}$  is defined on (\ref{F_n_adaptive_reg}).
\end{lemma}

\begin{proof}
We have
\begin{align} 
& \sup _{\varepsilon>\varepsilon_{n}} \log \mathcal{N} \left(\frac{1}{36} \varepsilon,\left\{(f, \sigma^2) \in \mathcal{F}_{n}: h_{n}\left((f, \sigma^2), (f_0, \sigma_0^2)\right)<\varepsilon\right\}, h_{n}\right) \nonumber \\
& \leq 
\log \mathcal{N} \left(\frac{1}{36} \varepsilon_n, \mathcal{F}_{n}, h_{n}\right) \nonumber \\
& \leq 
\log \left( \sum_{r=1}^{\xi_n}  \mathcal{N} \left(\frac{1}{36} \varepsilon_n, \Big\{ \left(T_F \circ f_{\bm{\theta}, \bm{\rho}_{\nu}}^{\operatorname{DNN}} , \sigma^2 \right)^{\top} :  f_{\bm{\theta}, \bm{\rho}_{\nu}}^{\operatorname{DNN}} \in  \mathcal{F}^{\operatorname{DNN}}_{\bm{\rho}_{\nu}}(L_n, r),\ 0 < \sigma^2 \leq e^{4n \varepsilon_n^2} \Big\}, h_{n}\right) \right) \nonumber \\
& \leq 
\log \left( \xi_n   \mathcal{N} \left(\frac{1}{36} \varepsilon_n, \Big\{ \left(T_F \circ f_{\bm{\theta}, \bm{\rho}_{\nu}}^{\operatorname{DNN}} , \sigma^2 \right)^{\top} :  f_{\bm{\theta}, \bm{\rho}_{\nu}}^{\operatorname{DNN}} \in  \mathcal{F}^{\operatorname{DNN}}_{\bm{\rho}_{\nu}}(L_n, \xi_n),\ 0 < \sigma^2 \leq e^{4n \varepsilon_n^2} \Big\}, h_{n}\right) \right) \nonumber \\
& \lesssim \log \xi_n + n \varepsilon_{n}^2 + L_n^2 \xi_n^2 \log L_n \xi_n^2 \log n \nonumber\\
& \lesssim  n \varepsilon_{n}^2,  \nonumber
\end{align}   
where the fourth inequality holds by (\ref{enp_bound_tmp}).
\end{proof}

\begin{lemma} \label{reg_cond2_adapt}
    For given $\bm{x}^{(n)} = (\bm{x}_1 , \dots, \bm{x}_n)$, we define
    $K_i((f_0, \sigma_0^2), (f, \sigma^2))$, $V_{i} ((f_0, \sigma_0^2), (f, \sigma^2))$ and $B_{n}^{*}\left((f_0, \sigma_0^2), \varepsilon_n\right)$ in the same way as in Lemma \ref{reg_cond2}, with the only change in the definition of $\mathcal{F}_{n}$ by (\ref{F_n_adaptive_reg}).
    Then, we have
    \begin{align*}
    \Pi_{r, \bm{\theta}, \sigma^2}\left( B_{n}^{*}\left((f_0 , \sigma_0^2), \varepsilon_n \right) \right) 
    \gtrsim& e^{-n \varepsilon_{n}^2}
    \end{align*}
    under the conditions of Theorem \ref{thm_adaptive}.
\end{lemma}
\begin{proof}
We define
$$ \xi_n^{\prime} := \lceil \tilde{C}_r \max_{(\beta', d') \in \mathcal{P}} n^{\frac{d'}{2(2\beta' + d')}} \rceil$$
and
\begin{align*}
    \bm{\xi}_n^{\prime} &:= (d, \xi_n^{\prime} , \dots, \xi_n^{\prime}, 1)^{\top} \in \mathbb{N}^{L_n + 2}, \\
    S_n^{\prime} &:= \sum_{l=1}^{L_n+1} ({\xi_n^{\prime}}^{(l-1)}+1){\xi_n^{\prime}}^{(l)}.
\end{align*}
Then, by Lemma \ref{thm_approx_comp}, 
there exists $f_{\hat{\bm{\theta}}, \bm{\rho}_{\nu}}^{\operatorname{DNN}} \in \mathcal{F}_{\bm{\rho}_{\nu}}^{\operatorname{DNN}}(L_n, \xi_n^{\prime}, \tilde{C}_B)$ such that
\begin{align*}
\left\|f_{\hat{\bm{\theta}}, \bm{\rho}_{\nu}}^{\operatorname{DNN}} - f_{0}\right\|_{\infty, [-a,a]^d} 
\leq & c_2 \max_{(\beta',d') \in \mathcal{P}} 
    n^{-\frac{\beta'}{2\beta'+d'}} \\
< & \frac{\sigma_0 \varepsilon_n}{4} 
\end{align*}
holds for sufficiently large $n$. 
Hence, we obtain 
\begin{align}
& \Pi_{r, \bm{\theta}, \sigma^2} \left( B_{n}^{*}\left((f_0 , \sigma_0^2), \varepsilon_n \right) \right) \nonumber \\
&\geq \Pi_{r, \bm{\theta}, \sigma^2}  \left( B_{n}^{*}\left((f_0 , \sigma_0^2), \varepsilon_n \right) \right) \nonumber \\
&\geq \Pi_{r}(\xi_n^{\prime}) \Pi_{(\bm{\theta}, \sigma^2) |r=\xi_n^{\prime} }  \left( B_{n}^{*}\left((f_0 , \sigma_0^2), \varepsilon_n \right) \right) \nonumber \\
& \gtrsim \frac{1}{(\log n)^5} e^{- (\log n)^5 (\xi_n^{\prime})^2} \delta_1^{S_n^{\prime}} \left( \frac{\sigma_0 \varepsilon_n}{2 a (d+1) (\xi_n^{\prime} + 1)^{L_n} C_B^{L_n} (L_n + 1)} \right)^{S_n^{\prime}} \delta_2 \varepsilon_n^2 \nonumber \\
& \gtrsim \exp \left( - (\log n)^5 C_r^2 \max_{(\beta',d') \in \mathcal{P}} 
    n^{\frac{d'}{2\beta'+d'}} \right) \exp \left(- C_r^2 C_L (\log n) \max_{(\beta',d') \in \mathcal{P}} 
    n^{\frac{d'}{2\beta'+d'}} (\log n)^{2}\right) n^{-1} \nonumber \\
& \gtrsim e^{-n \varepsilon_{n}^2} \label{lower_adapt}
\end{align}
for all but finite many $n$, where the third inequality holds by the proof of (\ref{tmp_for_adapt_low}).
\end{proof}

\begin{lemma} \label{reg_cond3_adapt}
For given $\bm{x}^{(n)} = (\bm{x}_1 , \dots, \bm{x}_n)$, we have
\begin{align*}
\frac{\Pi_{r, \bm{\theta}, \sigma^2}\left( \mathcal{F}_{n}^{c} \right)}{\Pi_{r, \bm{\theta}, \sigma^2}\left(B_{n}^{*}\left( (f_0 , \sigma_0^2), \varepsilon_{n} \right)\right)} 
= o(e^{-2n \varepsilon_{n}^2})
\end{align*}
under the conditions of Theorem \ref{thm_adaptive}, where $\mathcal{F}_{n}$  is defined on (\ref{F_n_adaptive_reg}).
\end{lemma}
\begin{proof}
Since
\begin{align*}
    \left( \frac{1}{2kr} - \frac{1}{4 k^2 r^3} \right)e^{-kr^2} 
    \leq \int_{r}^{\infty} e^{-kt^2} dt
    \leq \frac{1}{2kr} e^{-kr^2}
\end{align*}
for any $k>0$ and $s>0$,
\begin{align*}
\Pi_r (r > \xi_n) 
\leq & \frac{\sum_{r = \xi_n + 1}^{\infty} e^{- (\log n)^5 {r}^2}}
{\sum_{r = 1}^{\infty} e^{- (\log n)^5 {r}^2}}\\
\lesssim & \frac{e^{- (\log n)^5 \xi_n^2} }{\xi_n e^{- (\log n)^5} } \\
\lesssim& e^{- (\log n)^5 {\xi_n}^2} e^{(\log n)^5} 
\end{align*}
holds.
    From (\ref{lower_adapt}) and
    \begin{align*}
        \Pi_{r, \bm{\theta}, \sigma^2}\left( \mathcal{F}_{n}^{c} \right) 
        \leq & \Pi_r \left( r > \xi_n \right) + \Pi_{ \sigma^2}\left( \sigma^2 > e^{4 n \varepsilon_n^2} \right)\\
        \lesssim &  e^{- (\lambda \log n)^5 {\xi_n}^2} e^{(\lambda \log n)^5} + e^{-4 n \varepsilon_n^2}\\
        \lesssim & e^{-4 n \varepsilon_n^2},
    \end{align*}
    we obtain the assertion.
\end{proof}

\begin{proof}[Theorem \ref{thm_adaptive} with nonparametric Gaussian regression problem]

From Lemma \ref{reg_cond1_adapt}, Lemma \ref{reg_cond2_adapt}
and Theorem 4 of \citet{ghosal2007convergence}, we have 
\begin{equation*}
\mathbb{E}_0 \left[ \Pi_n \left( (f, \sigma^2) \in \mathcal{F}_n : h_n \left((f , \sigma^2 ), (f_0 , \sigma_0^2 )\right) > M_n \varepsilon_{n}  \middle\vert \mathcal{D}^{(n)}\right) \middle\vert \bm{X}^{(n)} = \bm{x}^{(n)}\right] \rightarrow 0
\end{equation*}
for every sequence $\{ \bm{x}^{(n)} \}_{n=1}^{\infty}$, where the expectation is with respect to $\{Y_i\}_{i=1}^n$.
Similar with the proof of (\ref{empirical_concen2}), we have
\begin{equation} \label{empirical_concen2_adapt}
\mathbb{E}_0 \left[ \Pi_n \left( (f, \sigma^2) \in \mathcal{F}_n : || f - f_0 ||_{2, n} +
|\sigma^2 - \sigma_0^2 |> M_n \varepsilon_{n}  \middle\vert \mathcal{D}^{(n)}\right) \right] \rightarrow 0,
\end{equation}
where the expectation is with respect to $\{(\bm{X}_i, Y_i)\}_{i=1}^n$.

Next, we will check the conditions in Lemma \ref{gyorfi2002distribution} for
\begin{align*}
\mathcal{G} :=& \left\{ g \ : \ g=(T_F \circ f_{\bm{\theta}, \bm{\rho}_{\nu}}^{\operatorname{DNN}} - f_0)^2 , f_{\bm{\theta}, \bm{\rho}_{\nu}}^{\operatorname{DNN}} \in \bigcup_{r=1}^{\xi_n} \mathcal{F}^{\operatorname{DNN}}_{\bm{\rho}_{\nu}}(L_n, r) \right\},\\
\tau :=& \frac{1}{2} ,\ \alpha := \varepsilon_{n}^2 ,\ K_1 = K_2 = 4F^2 .
\end{align*}
First, it is easy to check $||g(\bm{x})||_{\infty} \leq 4F^2$ and $\mathbb{E}(g(\bm{X})^2) \leq 4F^2 \mathbb{E}(g(\bm{X}))$ for $g \in \mathcal{G}$.
Also, since 
\begin{align*}
& \left\| (T_F \circ f_{\bm{\theta}_1, \bm{\rho}_{\nu}}^{\operatorname{DNN}} - f_0)^2 
- (T_F \circ f_{\bm{\theta}_2, \bm{\rho}_{\nu}}^{\operatorname{DNN}} - f_0)^2 \right\|_{1, n} \\ 
& =
\left\| (T_F \circ f_{\bm{\theta}_1, \bm{\rho}_{\nu}}^{\operatorname{DNN}} - f_0 + 
T_F \circ f_{\bm{\theta}_2, \bm{\rho}_{\nu}}^{\operatorname{DNN}} - f_0)(T_F \circ f_{\bm{\theta}_1, \bm{\rho}_{\nu}}^{\operatorname{DNN}} - T_F \circ f_{\bm{\theta}_2, \bm{\rho}_{\nu}}^{\operatorname{DNN}} ) \right\|_{1, n}  \\
& \leq 4F \left\| T_F \circ f_{\bm{\theta_1}, \bm{\rho}_{\nu}}^{\operatorname{DNN}}  - T_F \circ f_{\bm{\theta}_2, \bm{\rho}_{\nu}}^{\operatorname{DNN}} \right\|_{1, n}
\end{align*}
holds for any $f_{\bm{\theta}_1, \bm{\rho}_{\nu}}^{\operatorname{DNN}}, f_{\bm{\theta}_2, \bm{\rho}_{\nu}}^{\operatorname{DNN}} \in \bigcup_{r=1}^{\xi_n} \mathcal{F}^{\operatorname{DNN}}_{\bm{\rho}_{\nu}}(L_n, r)$, there exists $c_{34}>0$ such that
\begin{align*}
\mathcal{N} \left( u, \mathcal{G}, ||\cdot||_{1, n} \right)
\leq& \mathcal{N}\left( \frac{u}{4F}, \bigcup_{r=1}^{\xi_n}  T_F \circ \mathcal{F}^{\operatorname{DNN}}_{\bm{\rho}_{\nu}}(L_n, r), ||\cdot||_{1, n} \right) \\
\leq& \xi_n \mathcal{N}\left( \frac{u}{4F}, T_F \circ \mathcal{F}^{\operatorname{DNN}}_{\bm{\rho}_{\nu}}(L_n, \xi_n), ||\cdot||_{1, n} \right) \\
\leq& \xi_n \mathcal{M}\left( \frac{u}{4F}, T_F \circ \mathcal{F}^{\operatorname{DNN}}_{\bm{\rho}_{\nu}}(L_n, \xi_n), ||\cdot||_{1, n} \right) \\
\leq& 3 \xi_n \left(\frac{16 e F^2}{u} \log \frac{24 e F^2}{u}\right)^{\mathcal{F}^{\operatorname{DNN}}_{\bm{\rho}_{\nu}}(L_n, \xi_n)^{+}}\\
\lesssim& \xi_n n^{c_{34} L_n^2 \xi_n^2 \log (L_n \xi_n^2)}
\end{align*}
for $u \geq n^{-1}$ by Theorem 9.4 of \citet{gyorfi2002distribution} and Theorem 7 of \citet{bartlett2019nearly}.
Hence for all $t \geq \frac{\varepsilon_n^2}{8}$,
\begin{align*}
\int_{\frac{\tau(1-\tau)t}{16 \max \left\{K_{1}, 2 K_{2}\right\}}}^{\sqrt{t}}  
\sqrt{\log \mathcal{N} \left( u, \mathcal{G}, ||\cdot||_{1, n} \right)} d u 
\lesssim& \sqrt{t} \left( n^{\frac{d}{2\beta + d}} (\log n)^{6} \right)^{\frac{1}{2}}\\
=& o\left( \frac{\sqrt{n} \tau (1-\tau) t}{96 \sqrt{2} \max \left\{K_1, 2K_2 \right\}} \right)
\end{align*}
holds.
Hence, we have
\begin{align}
\mathbf{P}\left\{\sup _{f \in \mathcal{F}^{\operatorname{DNN}}(L_n, r_n)} \frac{\left| ||f-f_0||_{2, \mathrm{P}_{X}}^2  - ||f-f_0||_{2, n}^2 \right|}{\varepsilon_{n}^2+||f-f_0||_{2, \mathrm{P}_{X}}^2}>\frac{1}{2}\right\}
\leq 60 \exp \left(-\frac{n \varepsilon_{n}^2 / 8}{128 \cdot 2304 \cdot 16F^4}\right)
\label{last_adapt}
\end{align}
holds for all but finite many $n$ by Lemma \ref{gyorfi2002distribution}. 
By (\ref{empirical_concen2_adapt}) and (\ref{last_adapt}), we obtain
\begin{equation*} \label{empirical_concen3_adapt}
\mathbb{E}_0 \left[ \Pi_n \left( (f, \sigma^2) \in \mathcal{F}_n : || f - f_0 ||_{2, \mathrm{P}_{X}} +
|\sigma^2 - \sigma_0^2 |> M_n \varepsilon_{n}  \middle\vert \mathcal{D}^{(n)}\right) \right] \rightarrow 0,
\end{equation*}
where the expectation is with respect to $\{(\bm{X}_i, Y_i)\}_{i=1}^n$.

Finally, Lemma \ref{reg_cond3_adapt} and Lemma 1 of \citet{ghosal2007convergence} imply that 
\begin{equation*}
\mathbb{E}_0 \left[ \Pi_n \left( (f, \sigma^2)^{\top} \in \mathcal{F}_n^{c}  \middle\vert \mathcal{D}^{(n)}\right)\right] \rightarrow 0.
\end{equation*}
Hence, we obtain the assertion. 
\end{proof}

\bibliography{JMLR}

\begin{thebibliography}{81}
\providecommand{\natexlab}[1]{#1}
\providecommand{\url}[1]{\texttt{#1}}
\expandafter\ifx\csname urlstyle\endcsname\relax
  \providecommand{\doi}[1]{doi: #1}\else
  \providecommand{\doi}{doi: \begingroup \urlstyle{rm}\Url}\fi

\bibitem[Bai et~al.(2020)Bai, Song, and Cheng]{bai2020efficient}
Jincheng Bai, Qifan Song, and Guang Cheng.
\newblock Efficient variational inference for sparse deep learning with theoretical guarantee.
\newblock \emph{Advances in Neural Information Processing Systems}, 33:\penalty0 466--476, 2020.

\bibitem[Barron(1993)]{barron1993universal}
Andrew~R Barron.
\newblock Universal approximation bounds for superpositions of a sigmoidal function.
\newblock \emph{IEEE Transactions on Information theory}, 39\penalty0 (3):\penalty0 930--945, 1993.

\bibitem[Bartlett et~al.(2019)Bartlett, Harvey, Liaw, and Mehrabian]{bartlett2019nearly}
Peter~L Bartlett, Nick Harvey, Christopher Liaw, and Abbas Mehrabian.
\newblock Nearly-tight vc-dimension and pseudodimension bounds for piecewise linear neural networks.
\newblock \emph{The Journal of Machine Learning Research}, 20\penalty0 (1):\penalty0 2285--2301, 2019.

\bibitem[Bauer and Kohler(2019)]{bauer2019deep}
Benedikt Bauer and Michael Kohler.
\newblock On deep learning as a remedy for the curse of dimensionality in nonparametric regression.
\newblock \emph{The Annals of Statistics}, 47\penalty0 (4):\penalty0 2261, 2019.

\bibitem[Beker et~al.(2020)Beker, Wo{\l}os, Szymku{\'c}, and Grzybowski]{beker2020minimal}
Wiktor Beker, Agnieszka Wo{\l}os, Sara Szymku{\'c}, and Bartosz~A Grzybowski.
\newblock Minimal-uncertainty prediction of general drug-likeness based on bayesian neural networks.
\newblock \emph{Nature Machine Intelligence}, 2\penalty0 (8):\penalty0 457--465, 2020.

\bibitem[Blundell et~al.(2015)Blundell, Cornebise, Kavukcuoglu, and Wierstra]{blundell2015weight}
Charles Blundell, Julien Cornebise, Koray Kavukcuoglu, and Daan Wierstra.
\newblock Weight uncertainty in neural network.
\newblock In \emph{International Conference on Machine Learning}, pages 1613--1622. PMLR, 2015.

\bibitem[Chen et~al.(2022)Chen, Jiang, Liao, and Zhao]{chen2022nonparametric}
Minshuo Chen, Haoming Jiang, Wenjing Liao, and Tuo Zhao.
\newblock Nonparametric regression on low-dimensional manifolds using deep relu networks: Function approximation and statistical recovery.
\newblock \emph{Information and Inference: A Journal of the IMA}, 11\penalty0 (4):\penalty0 1203--1253, 2022.

\bibitem[Chen et~al.(2014)Chen, Fox, and Guestrin]{chen2014stochastic}
Tianqi Chen, Emily Fox, and Carlos Guestrin.
\newblock Stochastic gradient hamiltonian monte carlo.
\newblock In \emph{International conference on machine learning}, pages 1683--1691. PMLR, 2014.

\bibitem[Ch{\'e}rief-Abdellatif(2020)]{cherief2020convergence}
Badr-Eddine Ch{\'e}rief-Abdellatif.
\newblock Convergence rates of variational inference in sparse deep learning.
\newblock In \emph{International Conference on Machine Learning}, pages 1831--1842. PMLR, 2020.

\bibitem[Cranmer et~al.(2021)Cranmer, Tamayo, Rein, Battaglia, Hadden, Armitage, Ho, and Spergel]{cranmer2021bayesian}
Miles Cranmer, Daniel Tamayo, Hanno Rein, Peter Battaglia, Samuel Hadden, Philip~J Armitage, Shirley Ho, and David~N Spergel.
\newblock A bayesian neural network predicts the dissolution of compact planetary systems.
\newblock \emph{Proceedings of the National Academy of Sciences}, 118\penalty0 (40), 2021.

\bibitem[Douglas and Yu(2018)]{douglas2018relu}
Scott~C Douglas and Jiutian Yu.
\newblock Why relu units sometimes die: analysis of single-unit error backpropagation in neural networks.
\newblock In \emph{2018 52nd Asilomar Conference on Signals, Systems, and Computers}, pages 864--868. IEEE, 2018.

\bibitem[Dusenberry et~al.(2020)Dusenberry, Jerfel, Wen, Ma, Snoek, Heller, Lakshminarayanan, and Tran]{5_dusenberry2020efficient}
Michael Dusenberry, Ghassen Jerfel, Yeming Wen, Yian Ma, Jasper Snoek, Katherine Heller, Balaji Lakshminarayanan, and Dustin Tran.
\newblock Efficient and scalable bayesian neural nets with rank-1 factors.
\newblock In \emph{International conference on machine learning}, pages 2782--2792. PMLR, 2020.

\bibitem[Eldan and Shamir(2016)]{eldan2016power}
Ronen Eldan and Ohad Shamir.
\newblock The power of depth for feedforward neural networks.
\newblock In \emph{Conference on learning theory}, pages 907--940. PMLR, 2016.

\bibitem[Farquhar et~al.(2020)Farquhar, Osborne, and Gal]{15_farquhar2020radial}
Sebastian Farquhar, Michael~A Osborne, and Yarin Gal.
\newblock Radial bayesian neural networks: Beyond discrete support in large-scale bayesian deep learning.
\newblock In \emph{International Conference on Artificial Intelligence and Statistics}, pages 1352--1362. PMLR, 2020.

\bibitem[Fortuin(2022)]{fortuin2022priors}
Vincent Fortuin.
\newblock Priors in bayesian deep learning: A review.
\newblock \emph{International Statistical Review}, 90\penalty0 (3):\penalty0 563--591, 2022.

\bibitem[Fortuin et~al.(2022)Fortuin, Garriga-Alonso, Ober, Wenzel, Ratsch, Turner, van~der Wilk, and Aitchison]{18_fortuin2022bayesian}
Vincent Fortuin, Adri{\`a} Garriga-Alonso, Sebastian~W. Ober, Florian Wenzel, Gunnar Ratsch, Richard~E Turner, Mark van~der Wilk, and Laurence Aitchison.
\newblock Bayesian neural network priors revisited.
\newblock In \emph{International Conference on Learning Representations}, 2022.
\newblock URL \url{https://openreview.net/forum?id=xkjqJYqRJy}.

\bibitem[Ghosal and van~der Vaart(2007)]{ghosal2007convergence}
Subhashis Ghosal and Aad van~der Vaart.
\newblock Convergence rates of posterior distributions for noniid observations.
\newblock \emph{The Annals of Statistics}, 35\penalty0 (1):\penalty0 192--223, 2007.

\bibitem[Ghosal and van~der Vaart(2017)]{ghosal2017fundamentals}
Subhashis Ghosal and Aad van~der Vaart.
\newblock \emph{Fundamentals of nonparametric Bayesian inference}, volume~44.
\newblock Cambridge University Press, 2017.

\bibitem[Ghosal et~al.(2000)Ghosal, Ghosh, and van~der Vaart]{ghosal2000convergence}
Subhashis Ghosal, Jayanta~K Ghosh, and Aad~W van~der Vaart.
\newblock Convergence rates of posterior distributions.
\newblock \emph{Annals of Statistics}, pages 500--531, 2000.

\bibitem[Ghosh et~al.(2019)Ghosh, Yao, and Doshi-Velez]{12_ghosh2019model}
Soumya Ghosh, Jiayu Yao, and Finale Doshi-Velez.
\newblock Model selection in bayesian neural networks via horseshoe priors.
\newblock \emph{J. Mach. Learn. Res.}, 20\penalty0 (182):\penalty0 1--46, 2019.

\bibitem[Goodfellow et~al.(2016)Goodfellow, Bengio, and Courville]{goodfellow2016deep}
Ian Goodfellow, Yoshua Bengio, and Aaron Courville.
\newblock \emph{Deep learning}.
\newblock MIT press, 2016.

\bibitem[Graves(2011)]{graves2011practical}
Alex Graves.
\newblock Practical variational inference for neural networks.
\newblock \emph{Advances in neural information processing systems}, 24, 2011.

\bibitem[Green(1995)]{green1995reversible}
Peter~J Green.
\newblock Reversible jump markov chain monte carlo computation and bayesian model determination.
\newblock \emph{Biometrika}, 82\penalty0 (4):\penalty0 711--732, 1995.

\bibitem[Gy{\"o}rfi et~al.(2002)Gy{\"o}rfi, Kohler, Krzy{\.z}ak, and Walk]{gyorfi2002distribution}
L{\'a}szl{\'o} Gy{\"o}rfi, Michael Kohler, Adam Krzy{\.z}ak, and Harro Walk.
\newblock \emph{A distribution-free theory of nonparametric regression}, volume~1.
\newblock Springer, 2002.

\bibitem[Heek and Kalchbrenner(2019)]{heek2019bayesian}
Jonathan Heek and Nal Kalchbrenner.
\newblock Bayesian inference for large scale image classification.
\newblock \emph{arXiv preprint arXiv:1908.03491}, 2019.

\bibitem[Hern{\'a}ndez-Lobato and Adams(2015)]{3_hernandez2015probabilistic}
Jos{\'e}~Miguel Hern{\'a}ndez-Lobato and Ryan Adams.
\newblock Probabilistic backpropagation for scalable learning of bayesian neural networks.
\newblock In \emph{International conference on machine learning}, pages 1861--1869. PMLR, 2015.

\bibitem[Imaizumi and Fukumizu(2019)]{imaizumi2019deep}
Masaaki Imaizumi and Kenji Fukumizu.
\newblock Deep neural networks learn non-smooth functions effectively.
\newblock In \emph{The 22nd international conference on artificial intelligence and statistics}, pages 869--878. PMLR, 2019.

\bibitem[Izmailov et~al.(2021)Izmailov, Vikram, Hoffman, and Wilson]{16_izmailov2021bayesian}
Pavel Izmailov, Sharad Vikram, Matthew~D Hoffman, and Andrew Gordon~Gordon Wilson.
\newblock What are bayesian neural network posteriors really like?
\newblock In \emph{International conference on machine learning}, pages 4629--4640. PMLR, 2021.

\bibitem[Jantre et~al.(2023)Jantre, Bhattacharya, and Maiti]{jantre2023layer}
Sanket Jantre, Shrijita Bhattacharya, and Tapabrata Maiti.
\newblock Layer adaptive node selection in bayesian neural networks: Statistical guarantees and implementation details.
\newblock \emph{Neural Networks}, 167:\penalty0 309--330, 2023.

\bibitem[Jiao et~al.(2023)Jiao, Shen, Lin, and Huang]{jiao2023deep}
Yuling Jiao, Guohao Shen, Yuanyuan Lin, and Jian Huang.
\newblock Deep nonparametric regression on approximate manifolds: Nonasymptotic error bounds with polynomial prefactors.
\newblock \emph{The Annals of Statistics}, 51\penalty0 (2):\penalty0 691--716, 2023.

\bibitem[Jospin et~al.(2022)Jospin, Laga, Boussaid, Buntine, and Bennamoun]{19_jospin2022hands}
Laurent~Valentin Jospin, Hamid Laga, Farid Boussaid, Wray Buntine, and Mohammed Bennamoun.
\newblock Hands-on bayesian neural networks—a tutorial for deep learning users.
\newblock \emph{IEEE Computational Intelligence Magazine}, 17\penalty0 (2):\penalty0 29--48, 2022.

\bibitem[Kendall and Gal(2017)]{kendall2017uncertainties}
Alex Kendall and Yarin Gal.
\newblock What uncertainties do we need in bayesian deep learning for computer vision?
\newblock \emph{Advances in neural information processing systems}, 30, 2017.

\bibitem[Kohler and Langer(2021{\natexlab{a}})]{kohler2021rate}
Michael Kohler and Sophie Langer.
\newblock On the rate of convergence of fully connected deep neural network regression estimates.
\newblock \emph{The Annals of Statistics}, 49\penalty0 (4):\penalty0 2231--2249, 2021{\natexlab{a}}.

\bibitem[Kohler and Langer(2021{\natexlab{b}})]{kohler2021supplementA}
Michael Kohler and Sophie Langer.
\newblock Supplement {A} to ``{O}n the rate of convergence of fully connected deep neural network regression estimates.”.
\newblock \url{https://projecteuclid.org/journals/supplementalcontent/10.1214/20-AOS2034/aos2034suppa.pdf}, 2021{\natexlab{b}}.

\bibitem[Kohler and Langer(2021{\natexlab{c}})]{kohler2021supplementB}
Michael Kohler and Sophie Langer.
\newblock Supplement {B} to ``{O}n the rate of convergence of fully connected deep neural network regression estimates.”.
\newblock \url{https://projecteuclid.org/journals/supplementalcontent/10.1214/20-AOS2034/aos2034suppb.pdf}, 2021{\natexlab{c}}.

\bibitem[Kohler et~al.(2022)Kohler, Krzy{\.z}ak, and Langer]{kohler2022estimation}
Michael Kohler, Adam Krzy{\.z}ak, and Sophie Langer.
\newblock Estimation of a function of low local dimensionality by deep neural networks.
\newblock \emph{IEEE transactions on information theory}, 68\penalty0 (6):\penalty0 4032--4042, 2022.

\bibitem[Kong et~al.(2023)Kong, Yang, Lee, Ohn, Baek, and Kim]{pmlr-v202-kong23e}
Insung Kong, Dongyoon Yang, Jongjin Lee, Ilsang Ohn, Gyuseung Baek, and Yongdai Kim.
\newblock Masked {B}ayesian neural networks : Theoretical guarantee and its posterior inference.
\newblock In \emph{Proceedings of the 40th International Conference on Machine Learning}, volume 202 of \emph{Proceedings of Machine Learning Research}, pages 17462--17491, 2023.

\bibitem[Lee and Lee(2022)]{lee2022asymptotic}
Kyeongwon Lee and Jaeyong Lee.
\newblock Asymptotic properties for bayesian neural network in besov space.
\newblock \emph{Advances in Neural Information Processing Systems}, 35:\penalty0 5641--5653, 2022.

\bibitem[Leshno et~al.(1993)Leshno, Lin, Pinkus, and Schocken]{leshno1993multilayer}
Moshe Leshno, Vladimir~Ya Lin, Allan Pinkus, and Shimon Schocken.
\newblock Multilayer feedforward networks with a nonpolynomial activation function can approximate any function.
\newblock \emph{Neural networks}, 6\penalty0 (6):\penalty0 861--867, 1993.

\bibitem[Li et~al.(2016)Li, Chen, Carlson, and Carin]{li2016preconditioned}
Chunyuan Li, Changyou Chen, David Carlson, and Lawrence Carin.
\newblock Preconditioned stochastic gradient langevin dynamics for deep neural networks.
\newblock In \emph{Thirtieth AAAI Conference on Artificial Intelligence}, 2016.

\bibitem[Liu(2021)]{liu2021variable}
Jeremiah Liu.
\newblock Variable selection with rigorous uncertainty quantification using deep bayesian neural networks: Posterior concentration and bernstein-von mises phenomenon.
\newblock In \emph{International Conference on Artificial Intelligence and Statistics}, pages 3124--3132. PMLR, 2021.

\bibitem[Louizos and Welling(2017)]{louizos2017multiplicative}
Christos Louizos and Max Welling.
\newblock Multiplicative normalizing flows for variational bayesian neural networks.
\newblock In \emph{International Conference on Machine Learning}, pages 2218--2227. PMLR, 2017.

\bibitem[Louizos et~al.(2017)Louizos, Ullrich, and Welling]{13_louizos2017bayesian}
Christos Louizos, Karen Ullrich, and Max Welling.
\newblock Bayesian compression for deep learning.
\newblock \emph{Advances in neural information processing systems}, 30, 2017.

\bibitem[Lu et~al.(2021)Lu, Shen, Yang, and Zhang]{lu2021deep}
Jianfeng Lu, Zuowei Shen, Haizhao Yang, and Shijun Zhang.
\newblock Deep network approximation for smooth functions.
\newblock \emph{SIAM Journal on Mathematical Analysis}, 53\penalty0 (5):\penalty0 5465--5506, 2021.

\bibitem[Lu et~al.(2020)Lu, Shin, Su, and Karniadakis]{lu2020dying}
Lu~Lu, Yeonjong Shin, Yanhui Su, and George~Em Karniadakis.
\newblock Dying relu and initialization: Theory and numerical examples.
\newblock \emph{COMMUNICATIONS IN COMPUTATIONAL PHYSICS}, 28\penalty0 (5):\penalty0 1671--1706, 2020.

\bibitem[Maas et~al.(2013)Maas, Hannun, Ng, et~al.]{maas2013rectifier}
Andrew~L Maas, Awni~Y Hannun, Andrew~Y Ng, et~al.
\newblock Rectifier nonlinearities improve neural network acoustic models.
\newblock In \emph{Proc. icml}, volume~30, page~3. Atlanta, GA, 2013.

\bibitem[MacKay(1992)]{mackay1992practical}
David~JC MacKay.
\newblock A practical bayesian framework for backpropagation networks.
\newblock \emph{Neural computation}, 4\penalty0 (3):\penalty0 448--472, 1992.

\bibitem[Montufar et~al.(2014)Montufar, Pascanu, Cho, and Bengio]{montufar2014number}
Guido~F Montufar, Razvan Pascanu, Kyunghyun Cho, and Yoshua Bengio.
\newblock On the number of linear regions of deep neural networks.
\newblock \emph{Advances in Neural Information Processing Systems}, 27:\penalty0 2924--2932, 2014.

\bibitem[Nair and Hinton(2010)]{nair2010rectified}
Vinod Nair and Geoffrey~E Hinton.
\newblock Rectified linear units improve restricted boltzmann machines.
\newblock In \emph{Proceedings of the 27th international conference on machine learning (ICML-10)}, pages 807--814, 2010.

\bibitem[Nakada and Imaizumi(2020)]{nakada2020adaptive}
Ryumei Nakada and Masaaki Imaizumi.
\newblock Adaptive approximation and generalization of deep neural network with intrinsic dimensionality.
\newblock \emph{Journal of Machine Learning Research}, 21\penalty0 (174):\penalty0 1--38, 2020.

\bibitem[Neal(2012)]{neal2012bayesian}
Radford~M Neal.
\newblock \emph{Bayesian learning for neural networks}, volume 118.
\newblock Springer Science \& Business Media, 2012.

\bibitem[Nguyen et~al.(2018)Nguyen, Li, Bui, and Turner]{nguyen2018variational}
Cuong~V Nguyen, Yingzhen Li, Thang~D Bui, and Richard~E Turner.
\newblock Variational continual learning.
\newblock In \emph{International Conference on Learning Representations}, 2018.

\bibitem[Ober and Aitchison(2021)]{9_ober2021global}
Sebastian~W Ober and Laurence Aitchison.
\newblock Global inducing point variational posteriors for bayesian neural networks and deep gaussian processes.
\newblock In \emph{International Conference on Machine Learning}, pages 8248--8259. PMLR, 2021.

\bibitem[Oh et~al.(2020)Oh, Adamczewski, and Park]{14_oh2020radial}
Changyong Oh, Kamil Adamczewski, and Mijung Park.
\newblock Radial and directional posteriors for bayesian neural networks.
\newblock \emph{AAAI Conference on Artificial Intelligence}, 2020.

\bibitem[Ohn and Kim(2019)]{ohn2019smooth}
Ilsang Ohn and Yongdai Kim.
\newblock Smooth function approximation by deep neural networks with general activation functions.
\newblock \emph{Entropy}, 21\penalty0 (7):\penalty0 627, 2019.

\bibitem[Ohn and Kim(2022)]{ohn2022nonconvex}
Ilsang Ohn and Yongdai Kim.
\newblock Nonconvex sparse regularization for deep neural networks and its optimality.
\newblock \emph{Neural computation}, 34\penalty0 (2):\penalty0 476--517, 2022.

\bibitem[Ohn and Lin(2024)]{ohn2024adaptive}
Ilsang Ohn and Lizhen Lin.
\newblock Adaptive variational bayes: Optimality, computation and applications.
\newblock \emph{The Annals of Statistics}, 52\penalty0 (1):\penalty0 335--363, 2024.

\bibitem[Petersen and Voigtlaender(2018)]{petersen2018optimal}
Philipp Petersen and Felix Voigtlaender.
\newblock Optimal approximation of piecewise smooth functions using deep relu neural networks.
\newblock \emph{Neural Networks}, 108:\penalty0 296--330, 2018.

\bibitem[Polson and Ro{\v{c}}kov{\'a}(2018)]{polson2018posterior}
Nicholas~G Polson and Veronika Ro{\v{c}}kov{\'a}.
\newblock Posterior concentration for sparse deep learning.
\newblock In \emph{Proceedings of the 32nd International Conference on Neural Information Processing Systems}, pages 938--949, 2018.

\bibitem[Rousseau and Szabo(2020)]{rousseau2020asymptotic}
Judith Rousseau and Botond Szabo.
\newblock Asymptotic frequentist coverage properties of bayesian credible sets for sieve priors.
\newblock \emph{The Annals of Statistics}, 48\penalty0 (4):\penalty0 2155--2179, 2020.

\bibitem[Schmidt-Hieber(2019)]{schmidt2019deep}
Johannes Schmidt-Hieber.
\newblock Deep relu network approximation of functions on a manifold.
\newblock \emph{arXiv preprint arXiv:1908.00695}, 2019.

\bibitem[Schmidt-Hieber(2020)]{schmidt2020nonparametric}
Johannes Schmidt-Hieber.
\newblock Nonparametric regression using deep neural networks with relu activation function.
\newblock \emph{The Annals of Statistics}, 48\penalty0 (4):\penalty0 1875--1897, 2020.

\bibitem[Seto et~al.(2021)Seto, Wells, and Zhang]{1_seto2021halo}
Skyler Seto, Martin~T Wells, and Wenyu Zhang.
\newblock Halo: Learning to prune neural networks with shrinkage.
\newblock In \emph{Proceedings of the 2021 SIAM International Conference on Data Mining (SDM)}, pages 558--566. SIAM, 2021.

\bibitem[{\v{S}}pendl and Pirc(2022)]{17_spendl2023easy}
Martin {\v{S}}pendl and Klementina Pirc.
\newblock Easy bayesian transfer learning with informative priors.
\newblock In \emph{Neural Information Processing Systems}, 2022.

\bibitem[Sun et~al.(2022)Sun, Song, and Liang]{sun2022consistent}
Yan Sun, Qifan Song, and Faming Liang.
\newblock Consistent sparse deep learning: Theory and computation.
\newblock \emph{Journal of the American Statistical Association}, 117\penalty0 (540):\penalty0 1981--1995, 2022.

\bibitem[Suzuki(2018)]{suzuki2018adaptivity}
Taiji Suzuki.
\newblock Adaptivity of deep relu network for learning in besov and mixed smooth besov spaces: optimal rate and curse of dimensionality.
\newblock In \emph{International Conference on Learning Representations}, 2018.

\bibitem[Swiatkowski et~al.(2020)Swiatkowski, Roth, Veeling, Tran, Dillon, Snoek, Mandt, Salimans, Jenatton, and Nowozin]{swiatkowski2020k}
Jakub Swiatkowski, Kevin Roth, Bastiaan Veeling, Linh Tran, Joshua Dillon, Jasper Snoek, Stephan Mandt, Tim Salimans, Rodolphe Jenatton, and Sebastian Nowozin.
\newblock The k-tied normal distribution: A compact parameterization of gaussian mean field posteriors in bayesian neural networks.
\newblock In \emph{International Conference on Machine Learning}, pages 9289--9299. PMLR, 2020.

\bibitem[Szab{\'o} et~al.(2015)Szab{\'o}, van~der Vaart, and van Zanten]{szabo2015frequentist}
BT~Szab{\'o}, AW~van~der Vaart, and JH~van Zanten.
\newblock Frequentist coverage of adaptive nonparametric bayesian credible sets.
\newblock \emph{The Annals of Statistics}, 43\penalty0 (4):\penalty0 1391--1428, 2015.

\bibitem[Tran et~al.(2019)Tran, Do, Reid, and Carneiro]{tran2019bayesian}
Toan Tran, Thanh-Toan Do, Ian Reid, and Gustavo Carneiro.
\newblock Bayesian generative active deep learning.
\newblock In \emph{International Conference on Machine Learning}, pages 6295--6304. PMLR, 2019.

\bibitem[Tsybakov(2009)]{tsybakov2009introduction}
Alexandre~B Tsybakov.
\newblock Introduction to nonparametric estimation, 2009.

\bibitem[van~der Vaart and van Zanten(2008)]{van2008rates}
Aad~W van~der Vaart and J~Harry van Zanten.
\newblock Rates of contraction of posterior distributions based on gaussian process priors.
\newblock \emph{The Annals of Statistics}, 36\penalty0 (3):\penalty0 1435--1463, 2008.

\bibitem[Wang and Yeung(2020)]{wang2020survey}
Hao Wang and Dit-Yan Yeung.
\newblock A survey on bayesian deep learning.
\newblock \emph{ACM computing surveys (csur)}, 53\penalty0 (5):\penalty0 1--37, 2020.

\bibitem[Wang et~al.(2015)Wang, Wang, and Yeung]{wang2015collaborative}
Hao Wang, Naiyan Wang, and Dit-Yan Yeung.
\newblock Collaborative deep learning for recommender systems.
\newblock In \emph{Proceedings of the 21th ACM SIGKDD international conference on knowledge discovery and data mining}, pages 1235--1244, 2015.

\bibitem[Welling and Teh(2011)]{welling2011bayesian}
Max Welling and Yee~W Teh.
\newblock Bayesian learning via stochastic gradient langevin dynamics.
\newblock In \emph{Proceedings of the 28th international conference on machine learning (ICML-11)}, pages 681--688. Citeseer, 2011.

\bibitem[Williams(1995)]{10_williams1995bayesian}
Peter~M Williams.
\newblock Bayesian regularization and pruning using a laplace prior.
\newblock \emph{Neural computation}, 7\penalty0 (1):\penalty0 117--143, 1995.

\bibitem[Wilson and Izmailov(2020)]{wilson2020bayesian}
Andrew~G Wilson and Pavel Izmailov.
\newblock Bayesian deep learning and a probabilistic perspective of generalization.
\newblock \emph{Advances in neural information processing systems}, 33:\penalty0 4697--4708, 2020.

\bibitem[Wu et~al.(2019)Wu, Nowozin, Meeds, Turner, Hernandez-Lobato, and Gaunt]{8_wu2018deterministic}
Anqi Wu, Sebastian Nowozin, Edward Meeds, Richard~E. Turner, Jose~Miguel Hernandez-Lobato, and Alexander~L. Gaunt.
\newblock Deterministic variational inference for robust bayesian neural networks.
\newblock In \emph{International Conference on Learning Representations}, 2019.
\newblock URL \url{https://openreview.net/forum?id=B1l08oAct7}.

\bibitem[Xie and Xu(2020)]{xie2020adaptive}
Fangzheng Xie and Yanxun Xu.
\newblock Adaptive bayesian nonparametric regression using a kernel mixture of polynomials with application to partial linear models.
\newblock \emph{Bayesian Analysis}, 15\penalty0 (1):\penalty0 159--186, 2020.

\bibitem[Xu et~al.(2015)Xu, Wang, Chen, and Li]{xu2015empirical}
Bing Xu, Naiyan Wang, Tianqi Chen, and Mu~Li.
\newblock Empirical evaluation of rectified activations in convolutional network.
\newblock \emph{arXiv preprint arXiv:1505.00853}, 2015.

\bibitem[Yarotsky(2017)]{yarotsky2017error}
Dmitry Yarotsky.
\newblock Error bounds for approximations with deep relu networks.
\newblock \emph{Neural Networks}, 94:\penalty0 103--114, 2017.

\bibitem[Zhang et~al.(2020)Zhang, Li, Zhang, Chen, and Wilson]{zhang2019cyclical}
Ruqi Zhang, Chunyuan Li, Jianyi Zhang, Changyou Chen, and Andrew~Gordon Wilson.
\newblock Cyclical stochastic gradient mcmc for bayesian deep learning.
\newblock In \emph{International Conference on Learning Representations}, 2020.

\end{thebibliography}

\end{document}